\documentclass[10pt]{article} 
\usepackage[preprint]{tmlr}

\usepackage{url}
\usepackage{booktabs}
\usepackage{amsmath,amssymb,amsfonts,amsthm}
\usepackage{mathtools}
\usepackage{bm}
\usepackage{nicefrac,xfrac}
\usepackage{microtype}
\usepackage{xcolor}
\usepackage{float}
\usepackage{graphicx}
\usepackage{wrapfig}
\usepackage{enumerate}
\usepackage{multirow}
\usepackage{mathrsfs}
\usepackage{algorithm}
\usepackage[noend]{algpseudocode}
\usepackage{caption}
\usepackage{subcaption}
\usepackage{tikz}
\usepackage{hyperref}       

\newtheorem{mytheo}{Theorem}
\newtheorem{mydef}{Definition}

\newtheorem{myrema}{Remark}
\newtheorem{mylemma}{Lemma}

\newtheorem{myassump}{Assumption}
\newtheorem{myprop}{Proposition}

\title{Interaction-Limited Safe Continuous-Time RL \\
for Dynamical Medical Treatment}

\author{%
\textbf{%
Yuepeng Wang$^{1}$\thanks{Y. Wang and K. Kawano contributed equally to this work. $^{\dagger}$Corresponding author.} \quad
Ken Kawano$^{2*}$ \quad
Akifumi Wachi$^{3}$ \quad
Zhou Yongqi$^{4}$ \quad
}\\
\textbf{
Richard Weiss$^{4}$ \quad
Yoshihiko Fujisawa$^{2}$ \quad
Mehrshad Sadria$^{5}$ \quad
Ying Chen$^{4}$ \quad
}\\
\textbf{
Xin Liu$^{6}$ \quad
Sebastien Gros$^{7}$ \quad
Xiao Hu$^{8}$ \quad
Kyoung\mbox{-}Sook Kim$^{6}$ \quad
}\\
\textbf{
Mengmou Li$^{9}$ \quad
Katsuki Fujisawa$^{2}$ \quad
Kenji Wakabayashi$^{2}$ \quad
Xun Shen$^{1\dagger}$ 
}\\
{\normalfont\small
$^{1}$Tokyo University of Agriculture and Technology, 
$^{2}$Institute of Science Tokyo,
$^{3}$LY Corporation, \\
$^{4}$National University of Singapore, 
$^{5}$Altos Labs, Inc., \\
$^{6}$National Institute of Advanced Industrial Science and Technology (AIST), \\
$^{7}$Norwegian University of Science and Technology,
$^{8}$Emory University, 
$^{9}$Hiroshima University\\
\texttt{shen@go.tuat.ac.jp}}
}
\def\month{MM}  
\def\year{YYYY} 
\def\openreview{\url{https://openreview.net/forum?id=XXXX}} 

\begin{document}

\maketitle

\begin{abstract}
Dynamic medical treatment requires deciding treatment intensity and intervention timing, while patient states evolve continuously and adverse events may occur between clinical interactions. 
Most existing treatment learning methods assume fixed schedules or enforce safety only at discrete decision points. 
We propose \textit{Interaction-Limited Safe Continuous-Time Reinforcement Learning}, a framework that jointly optimizes treatment administration and clinical interaction timing under trajectory-level safety constraints. 
Our key idea is to reformulate the continuous-time treatment problem as an option-based semi-Markov decision process, where each option specifies a continuous-time treatment policy and its duration. 
We develop a safety-tightening mechanism showing that suitably constructed constraints at interaction times guarantee safety over the full continuous-time trajectory with high probability. 
We further establish finite-sample guarantees for policy learning from logged treatment trajectories and introduce a practical data-driven conservative surrogate. 
Experiments show that the proposed adaptive interaction-timing mechanism improves both safety and treatment effectiveness over equidistant interaction schemes across different safe policy optimization methods.
\end{abstract}

\section{Introduction}
\label{sec:introduction}

Clinical treatment planning is inherently a continuous-time decision-making problem \citep{Carbajo2026}.\footnote{Large language model tools were used to assist with formatting and language editing of this manuscript; all ideas, claims, and results are the authors' own.}
In many high-stakes settings such as medical treatment of sepsis \citep{Ghossein2024} or cancer \citep{Zhang2023CTDT}, clinicians must decide not only how much treatment to administer but also when to intervene, while balancing treatment efficacy against safety constraints such as toxicity limits, physiological risk thresholds, and maximum allowable delays between clinical assessments.
Recent reinforcement learning approaches to dynamic treatment typically formulate the problem as a discrete-time Markov decision process with fixed or pre-specified decision intervals \citep{Komorowski, Shen_NeurIPS2025}. 
As a result, treatment decisions are optimized only at discrete intervention times, while patient states continue to evolve between them. 
This mismatch is especially problematic in safety-critical settings, because adverse events may occur between clinical interactions rather than exactly at them. 
Even recent continuous-time learning frameworks that optimize both treatment timing and dosage from logged data do not explicitly guarantee safety over the full continuous-time trajectory between interventions \citep{Zhang2023CTDT, Schneider2026}. 
These limitations motivate a new formulation for dynamic treatment that jointly optimizes treatment administration and interaction timing while enforcing trajectory-level safety in continuous time. 
Related work is discussed in Section~\ref{sec:related_work}.

\paragraph{Contributions.}
We propose \textit{Interaction-Limited Safe Continuous-Time Reinforcement Learning}, a theoretically grounded framework for learning safe and effective treatment policies for dynamical medical treatment. 
Our key contributions are as follows. 
(1) We formulate dynamic treatment as a continuous-time safety-constrained decision-making problem that jointly optimizes treatment administration and clinical interaction timing. 
(2) By introducing an interaction-limited option-based reformulation, our framework captures both continuous-time treatment evolution and the clinically important decision of when to measure, reassess, or intervene. 
(3) We develop a safety-tightening mechanism showing that constraints enforced only at interaction times suffice to ensure continuous-time trajectory safety with high probability. We also provide finite-sample guarantees for learning from logged treatment trajectories and a practical data-driven conservative surrogate for implementation.
(4) We show that the resulting practical approximation can be constructed directly from logged patient trajectories and integrated with existing safe reinforcement learning algorithms. 
This provides a principled and deployable approach for learning continuous-time treatment policies that retain high-probability feasibility for the original safety-constrained treatment problem.

\section{Related Work}
\label{sec:related_work}

\paragraph{Reinforcement learning for dynamic treatment.} 
Dynamic treatment planning is often formulated using reinforcement learning (RL), especially in healthcare applications such as sepsis management, where the goal is to learn treatment strategies from longitudinal ICU data \citep{Job2024}. 
A dominant line of work models the problem as a discrete-time Markov decision process (MDP) \citep{Huang2022SepsisRL}, in which the patient state is sampled at fixed intervals and treatment actions are optimized only on this discrete decision grid \citep{Tang2018, Hargrave2024}. 
Representative examples include the AI Clinician for sepsis treatment and subsequent offline RL approaches based on ICU data \citep{Komorowski, gottesman2019guidelines}. This discrete-time formulation has been widely adopted because it is compatible with standard RL algorithms and can be implemented directly from temporally aggregated electronic health records \citep{Nambiar2023, tumay2025, Shen_NeurIPS2025}. 
However, it also introduces an important abstraction gap: the underlying physiological process is continuous, whereas the learned policy only reasons over discretized snapshots. As a result, such methods may fail to capture clinically important changes that occur between decision times, and the timing itself is usually treated as fixed rather than optimized.

\paragraph{Continuous-time treatment learning.} 
To address the mismatch between continuous physiological evolution and discrete-time decision making, recent work has begun to study continuous-time treatment learning. In these approaches, patient evolution is modeled more explicitly over time \citep{Kuang2024, Lingsch2024}. 
In particular, recent continuous-time learning frameworks optimize not only treatment dosage but also the timing of observations or interventions from logged clinical data \citep{Zhang2023CTDT, Schneider2026}. This direction is particularly relevant to healthcare, since measurements and treatment updates are rarely performed on a fixed clock and instead depend on patient condition and clinical judgment. Nevertheless, existing continuous-time treatment learning methods still do not explicitly guarantee safety over the full continuous-time trajectory between two interventions. Their focus is primarily on representation, prediction, or policy optimization in continuous time, rather than on enforcing trajectory-level safety throughout the inter-interaction interval. This gap is especially important in high-stakes treatment settings, where adverse events may arise between two clinically observed time points.

\paragraph{Continuous-time and time-adaptive RL.} 
Beyond healthcare-specific treatment learning, a broader RL literature studies continuous-time control and time-adaptive decision making \citep{mannb14, Yu2021TAAC}. 
Continuous-time RL methods typically model state evolution through differential equations and derive policies that interact with the environment over continuous time \citep{Yildiz2021, Treven2023, Holt2023}. Within this line, time-adaptive methods further optimize when to observe or intervene, rather than assuming a fixed decision schedule \citep{Treven2024}. 
These formulations are conceptually close to real clinical decision making, since both the treatment action and the next interaction time may need to be chosen adaptively. However, most existing continuous-time RL work is developed for general control problems rather than medical treatment, and does not directly address trajectory-level safety guarantees under interaction limitations.

\paragraph{Safe reinforcement learning.} 
A separate literature studies safe RL through constrained MDPs, chance constraints, risk-sensitive objectives, shielding, and Lyapunov-based approaches \citep{altman2021constrained, Achiam, chow2018lyapunov}. These methods provide important tools for handling safety during policy optimization, and constrained policy methods such as CPO are particularly influential in practice. However, most of this literature is developed in discrete time and typically controls expected cumulative costs, constraint violations at decision times, or surrogate notions of stability. Even when safety constraints are imposed, they usually apply to sampled states on the decision grid rather than to the full continuous-time trajectory between two decisions. Consequently, these methods do not directly address the setting considered in our paper, where safety must hold throughout the interval between clinical interactions.

\paragraph{Temporal abstraction, options, and SMDPs.} 
Temporal abstraction provides another relevant line of research. Options and semi-Markov decision processes (SMDPs) extend standard MDPs by allowing high-level decisions to trigger temporally extended behaviors \citep{sutton1999between,bacon2017option}. This framework is attractive in healthcare, where reassessment and treatment changes are costly and cannot occur continuously, while low-level physiological evolution continues between two clinical interactions. Options therefore provide a natural modeling tool for jointly representing treatment administration and interaction timing. However, prior work on options and SMDPs has mainly focused on representation learning, exploration efficiency, or hierarchical control, rather than on safety-critical continuous-time treatment problems with trajectory-level guarantees.

\paragraph{Position of our work.} 
Our work lies at the intersection of these directions. Similar to prior medical RL work, we study dynamic treatment planning from data; similar to continuous-time RL and time-adaptive methods, we explicitly model treatment timing as a decision variable; similar to safe RL, we impose explicit safety constraints; and similar to options/SMDPs, we use temporal abstraction to represent interaction-limited treatment control. The key difference is that our framework combines these ingredients in a way tailored to safety-critical continuous-time treatment: we jointly optimize treatment dosage and interaction timing, impose an explicit limit on the number of clinical interactions, and develop a safety-tightening mechanism that guarantees safety over the full continuous-time trajectory between decisions. To the best of our knowledge, this combination has not been addressed in prior medical RL work.

\section{Problem Formulation}
\label{sec:problem}

\paragraph{Targeted system.} 
Let $\bm{\mathrm{x}}\in\mathcal{X}\subseteq\mathbb{R}^{d_{\mathsf{s}}}$ and $\bm{\mathrm{u}}\in\mathcal{U}\subseteq\mathbb{R}^{d_{\mathsf{a}}}$ denote the patient state and the permissible treatment action. 
Both $\mathcal{X}$ and $\mathcal{U}$ are continuous spaces. 
Moreover, $\mathcal{U}$ is assumed to be compact, since treatment actions, such as dosage levels, are typically restricted to a bounded and closed set. 
The patient state evolves in continuous time as follows:
\begin{equation}
    \label{eq:system_dynamics}
    \bm{\mathrm{x}}_{t}=\bm{\mathrm{x}}_0 + \int_{0}^{t} f(\bm{\mathrm{x}}_{\tau},\bm{\mathrm{u}}_{\tau},\bm{\mathrm{w}}_{\tau})\mathsf{d}\tau. 
\end{equation}
Here, $\bm{\mathrm{w}}_{t}\in\mathcal{W}$ represents a stochastic disturbance or model uncertainty, which captures factors affecting patient evolution that are not fully predictable from the current state and treatment action. In medical treatment scenarios, such disturbances may reflect unobserved physiological variability, heterogeneous treatment response, measurement noise, latent disease progression, or unexpected external influences such as concurrent medications and clinical complications. 
The disturbance process $\{\bm{\mathrm{w}}_{t}\}_{t\ge 0}$ is assumed to be exogenous, with sample paths that are measurable and locally integrable in time. 
The measurability and local integrability conditions are natural from a modeling perspective, since patient conditions and treatment responses typically vary continuously over time without requiring idealized smooth trajectories, while still allowing accumulated uncertainty over any finite clinical interval to be well defined. 
Here, $f(\bm{\mathrm{x}},\bm{\mathrm{u}},\bm{\mathrm{w}})$ is assumed to be measurable in $\bm{\mathrm{w}}$ and locally Lipschitz continuous in $\bm{\mathrm{x}}$, uniformly over $\bm{\mathrm{u}}\in\mathcal{U}$ and $\bm{\mathrm{w}}\in\mathcal{W}$. 
The safe region is given by
\begin{equation}
    \mathcal{X}_{\mathsf{safe}}:=\left\{\bm{\mathrm{x}}\in\mathcal{X}:g(\bm{\mathrm{x}})\leq 0\right\}.
\end{equation}

\paragraph{Underlying interaction-limited continuous-time treatment problem.}
In practice, clinical measurements and treatment updates are costly and occur only at irregular interaction times.
We therefore first formulate the underlying control problem for \emph{interaction-limited continuous-time treatment}.
\begin{mydef}
\label{def:segment_controller}
An interaction-limited segment controller is defined as a tuple $\gamma=(\phi,\eta),$
where
(i) $\phi:\mathcal{X}\to\mathcal{U}$ is a measurable intra-segment feedback law, and 
(ii) $\eta:\mathcal{X}\times[0,T)\to[0,1]$ is a termination function that induces a (possibly stochastic) starting time for the next segment.
\end{mydef}

Let $\Gamma$ denote the class of admissible segment controllers that induce inter-interaction durations in $[t_{\mathsf{min}}, t_{\mathsf{max}}]$.
At interaction time \(t_k\), a high-level policy \(\nu\) selects a controller $\gamma_k=(\phi_k,\eta_k)\sim \nu(\cdot\,|\,\bm{\mathrm{x}}_{t_k},t_k,k).$ 
Once selected, $\gamma_k$ governs the treatment continuously until the next interaction time: $\bm{\mathrm{u}}_t=\phi_k(\bm{\mathrm{x}}_t)$ for all $t\in[t_k,t_{k+1}).$
The controller terminates at $t_{k+1}=\min\{t_k+\tau_k,T\}$,
where \(\tau_k\) denotes the (possibly random) duration induced by \(\eta_k\).
We define $\delta t_k:=t_{k+1}-t_k$ and impose $t_{\mathsf{min}}\le \delta t_k\le t_{\mathsf{max}}$ almost surely
for all $k$ to enforce clinically meaningful timing constraints.
Let \(N:=\min\{k\ge 0:\ t_k=T\}\) denote the number of executed treatment segments up to horizon \(T\).
We impose a hard budget $K$ on the number of high-level clinical interactions, i.e., $N\le K$ almost surely.
The resulting interaction-limited treatment problem is
\begin{align}
\max_{\nu}\ &\ \mathbb{E}\left[\sum_{k=0}^{N-1}\int_{t_k}^{t_{k+1}} b\!\left(\bm{\mathrm{x}}_t,\phi_k(\bm{\mathrm{x}}_t)\right)\mathsf{d}t\right]
\tag{$\mathsf{CTTP\text{-}IL}$}\label{eq:problem_cttp_il}\\
\mathsf{s.t.}\ &\ \mathsf{Pr}\left\{\bm{\mathrm{x}}_{t}\in\mathcal{X}_{\mathsf{safe}},\ \forall t\in[0,T]\right\}\ge 1-\alpha,
\label{eq:problem_cttp_il_chance}\\
&\ N\le K,\quad t_0=0,\quad t_N=T,\quad \delta t_k\in[t_{\mathsf{min}},t_{\mathsf{max}}]\ \ \forall k<N.
\label{eq:problem_cttp_il_budget}
\end{align}
Problem \ref{eq:problem_cttp_il} captures the clinically relevant restriction that treatment is updated at a limited interaction number while the patient state and the treatment evolve continuously between interactions.

\paragraph{Interaction-limited high-level control via options.}
We now introduce options as a structured representation of the segment controllers in \ref{eq:problem_cttp_il}.
Thus, the options framework is used here not to define a different control objective, but to represent the same interaction-limited treatment problem in a hierarchical and learning-friendly form \citep{sutton1999between}.

\begin{mydef}[Continuous-time option]
\label{def:ct_option}
A (continuous-time) option \(o\in\mathcal{O}\) is a tuple $o=(\pi_o,\beta_o),$
where
(i) \(\pi_o:\mathcal{X}\to\mathcal{U}\) is an intra-option feedback policy, and
(ii) \(\beta_o:\mathcal{X}\times[0,T)\to[0,1]\) is a termination function that induces a (possibly stochastic) starting time for the next clinical interaction.
\end{mydef}

At interaction time \(t_k\), the high-level option policy \(\mu\) selects an option $o_k\sim \mu(\cdot\,|\,\bm{\mathrm{x}}_{t_k},t_k,k).$
Once initiated, the option generates the continuous-time treatment trajectory
\begin{equation}
\label{eq:intra_option_control}
\bm{\mathrm{u}}_t=\pi_{o_k}(\bm{\mathrm{x}}_t),\qquad t\in[t_k,t_{k+1}),
\end{equation}
and terminates at
\begin{equation}
\label{eq:option_termination_time}
t_{k+1}=\min\{t_k+\tau_k,T\},
\qquad
\delta t_k:=t_{k+1}-t_k,
\end{equation}
where \(\tau_k\) denotes the duration induced by \(\beta_{o_k}\).
As above, we impose $t_{\mathsf{min}}\le \delta t_k\le t_{\mathsf{max}}
\ \text{a.s.}$

\paragraph{Option-based Interaction-Limited Safety Constrained Problem (O-ILSCP).}
Using the option representation, we write the interaction-limited treatment problem as
\begin{align}
\max_{\mu,\{\pi_o,\beta_o\}_{o\in\mathcal{O}}}\ &\ \mathbb{E}\left[\sum_{k=0}^{N-1}\int_{t_k}^{t_{k+1}} b\!\left(\bm{\mathrm{x}}_{t},\pi_{o_k}(\bm{\mathrm{x}}_t)\right)\mathsf{d}t\right]
\tag{$\mathsf{O\text{-}ILSCP}$}\label{eq:problem_option_ils}\\
\mathsf{s.t.}\ &\ \mathsf{Pr}\left\{\bm{\mathrm{x}}_{t}\in\mathcal{X}_{\mathsf{safe}},\ \forall t\in[0,T]\right\}\ge 1-\alpha,
\label{eq:problem_option_ils_chance}\\
&\ N\le K,\quad t_0=0,\quad t_N=T,\quad \delta t_k\in[t_{\mathsf{min}},t_{\mathsf{max}}]\ \ \forall k<N.
\label{eq:problem_option_ils_budget}
\end{align}
Problem \ref{eq:problem_option_ils} therefore optimizes treatment administration and clinical interaction timing under the same trajectory-level safety requirement as \ref{eq:problem_cttp_il}, but uses options as the representation of the inter-interaction treatment law.
The following result clarifies the relation between the underlying interaction-limited treatment problem and its option-based formulation.

\begin{myprop}[Every option-based policy induces an interaction-limited treatment law]
\label{prop:oilscp_to_cttpil}
Assume that every admissible option \(o\in\mathcal O\) induces an admissible segment controller $(\pi_o,\beta_o)\in\Gamma.$
Then every feasible solution of \ref{eq:problem_option_ils} induces a feasible solution of \ref{eq:problem_cttp_il} with the same closed-loop state-control trajectory, the same accumulated reward, and the same trajectory-level safety probability.
Then, $V^\star_{\mathsf{O\text{-}ILSCP}}
\le
V^\star_{\mathsf{CTTP\text{-}IL}},$
where \(V^\star_{\mathsf{O\text{-}ILSCP}}\) and \(V^\star_{\mathsf{CTTP\text{-}IL}}\) denote the corresponding optimal values.
\end{myprop}

\begin{mytheo}[Exact reformulation under an expressive option class]
\label{thm:oilscp_exact}
Suppose, in addition, that the admissible option representation is rich enough so that for every admissible segment controller $\gamma=(\phi,\eta)\in\Gamma,$
there exists an option \(o^\gamma\in\mathcal O\) such that
$\pi_{o^\gamma}=\phi,
\qquad
\beta_{o^\gamma}=\eta.$
Then \ref{eq:problem_option_ils} is an exact reformulation of \ref{eq:problem_cttp_il}.
In particular, the two problems induce the same feasible closed-loop laws on \([0,T]\), and $V^\star_{\mathsf{O\text{-}ILSCP}}
=
V^\star_{\mathsf{CTTP\text{-}IL}}.$
\end{mytheo}

\begin{myrema}[Restricted option classes]
\label{rem:restricted_option_classes}
The exact equivalence in Theorem~\ref{thm:oilscp_exact} requires the option class to be sufficiently expressive.
If, in implementation, \(\mathcal O\) is restricted to a finite dictionary or a parametric family, then \ref{eq:problem_option_ils} becomes a structured conservative approximation of \ref{eq:problem_cttp_il}.
\end{myrema}

\section{Method}
\label{sec:method}

\begin{figure}[t]
    \centering
    \includegraphics[width=\linewidth]{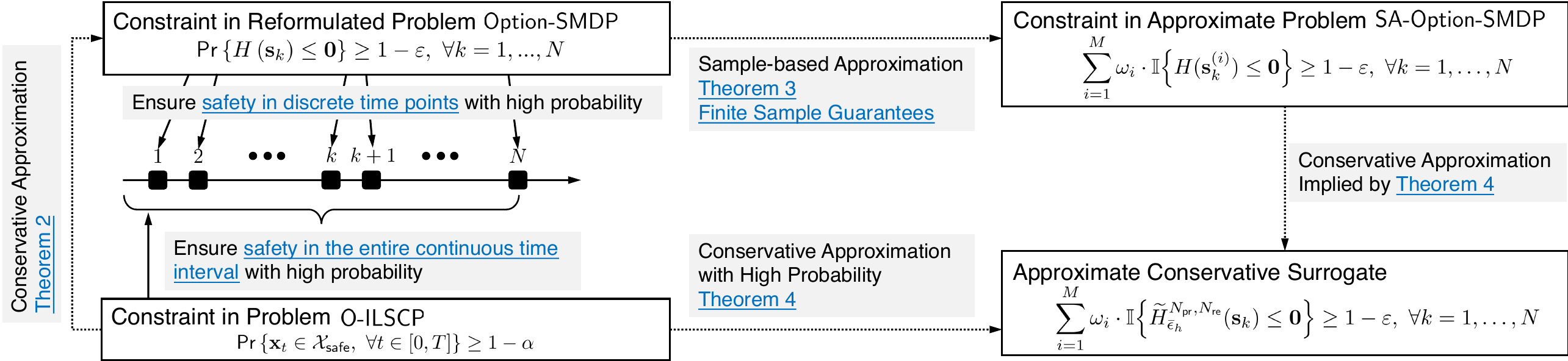}
    \caption{Summary of the relations among the main theorems and constraints in Section \ref{sec:method}.}
   \label{fig:contribution_summary}
\end{figure}

This section presents a practical approximation framework for solving \ref{eq:problem_option_ils}. 
First, \ref{eq:problem_option_ils} is reformulated as an option-based \emph{semi-Markov decision process} (SMDP), given by \ref{eq:problem_option_smdp} in Section \ref{subsec:option_smdp}. 
Theorem \ref{theo:selection_safety} in Section \ref{subsec:selection_safety_function} shows that, by choosing a sufficiently conservative safety function, enforcing safety only at clinical interaction times is sufficient to guarantee safety over the entire continuous-time treatment trajectory with high probability. 
Next, Theorem \ref{theo:finite_sample_main} in Section \ref{subsection:finite_samples} provides finite-sample guarantees for the sample-based approximation of \ref{eq:problem_option_smdp}, which is essential because practical deployment relies on finite logged treatment data. 
Section \ref{subsec:practical_implementation} further introduces a data-driven conservative surrogate of the safety constraint in \ref{eq:problem_option_smdp}, constructed from samples of the current state, successor state, and inter-interaction duration. 
Taken together, these results yield a practical approximation of \ref{eq:problem_option_ils} that can be formulated directly from data and solved using existing safe RL algorithms. 
Moreover, the resulting approximate solution retains a high-probability safety guarantee for the original continuous-time treatment problem. 
The relations among the main theorems and constraints in Section \ref{sec:method} are summarized in Figure \ref{fig:contribution_summary}.

\subsection{Reformulation as an Option Semi-Markov Decision Process}
\label{subsec:option_smdp}

We reformulate \ref{eq:problem_option_ils} as a SMDP defined over options. 
At interaction time $t_k$, the policy observes $(\bm{\mathrm{x}}_{t_k},t_k,k)$ and selects an option $o_k$, after which the patient state evolves under the corresponding continuous-time treatment rule until the next clinical interaction.

\paragraph{Option-induced state and reward flows.}
Conditioned on initiating option $o_k$ at time $t_k$ from $\bm{\mathrm{x}}_{t_k}=\bm{\mathrm{x}}_k$, the resulting patient trajectory satisfies
\begin{equation}
\label{eq:option_state_flow}
\bm{\mathrm{x}}_{t_k+\tau}
=
\bm{\mathrm{x}}_{k}
+
\int_{t_k}^{t_k+\tau}
f\!\left(\bm{\mathrm{x}}_{t},\pi_{o_k}(\bm{\mathrm{x}}_{t}),\bm{\mathrm{w}}_{t}\right)\mathsf{d}t,\qquad \tau\in[0,\delta t_k].
\end{equation}
The next interaction time is given by $t_{k+1}=t_k+\delta t_k$ (clipped at $T$ as in \eqref{eq:option_termination_time}), where $\delta t_k$ is determined by the termination function $\beta_{o_k}$.
Cumulative treatment reward of executing option $o_k$ is defined as
\begin{equation}
\label{eq:option_reward_flow}
R\!\left(\bm{\mathrm{x}}_{k},t_k,k,o_k\right)
:=
\int_{t_k}^{t_{k+1}}
b\!\left(\bm{\mathrm{x}}_{t},\pi_{o_k}(\bm{\mathrm{x}}_{t})\right)\mathsf{d}t.
\end{equation}

\paragraph{SMDP state and transition.}
We define the SMDP state at interaction times as
\begin{equation}
\label{eq:option_aug_state}
\bm{\mathrm{s}}_k := (\bm{\mathrm{x}}_{k},t_k,k),
\qquad \bm{\mathrm{x}}_{k}:=\bm{\mathrm{x}}_{t_k}.
\end{equation}
Given $\bm{\mathrm{s}}_k$ and option $o_k$, the next state is $\bm{\mathrm{s}}_{k+1}=(\bm{\mathrm{x}}_{k+1},t_{k+1},k+1)$ with $\bm{\mathrm{x}}_{k+1}:=\bm{\mathrm{x}}_{t_{k+1}}$ generated according to \eqref{eq:option_state_flow}.
This induces an SMDP transition kernel $\mathcal{P}(\bm{\mathrm{s}}_{k+1}\,|\,\bm{\mathrm{s}}_k,o_k),$
which captures both the exogenous disturbance process and the randomness arising from option termination.

\paragraph{Option-SMDP objective.}
With the above notation, \ref{eq:problem_option_ils} can be equivalently written as
\begin{align}
\max_{\mu,\{\pi_o,\beta_o\}_{o\in\mathcal{O}}}\ &\
\mathbb{E}\left[\sum_{k=0}^{N-1} r(\bm{\mathrm{s}}_k,o_k)\right]
\tag{$\mathsf{Option\text{-}SMDP}$}\label{eq:problem_option_smdp}\\
\mathsf{s.t.}\ &\ 
\mathsf{Pr}\left\{ H\left(\bm{\mathrm{s}}_{k},o_k\right) \leq \bm{0} \right\}\geq 1-\varepsilon,\ \forall k=0,\ldots,N-1, \label{eq:problem_option_smdp_safe}\\
&\ N\le K,\quad \delta t_k\in[t_{\mathsf{min}},t_{\mathsf{max}}]\ \ \forall k<N,
\end{align}
where $o_k\sim\mu(\cdot\,|\,\bm{\mathrm{s}}_k)$ and $r(\bm{\mathrm{s}}_k,o_k):=R(\bm{\mathrm{x}}_k,t_k,k,o_k)$ as defined in \eqref{eq:option_reward_flow}.
Problem \ref{eq:problem_option_smdp} enables the simultaneous optimization of
(i) intra-option treatment policies $\{\pi_o\}$, which govern continuous-time treatment administration between clinical interactions, and
(ii) termination rules $\{\beta_o\}$, which determine state-dependent interaction intervals $\delta t_k$,
while remaining compatible with standard SMDP and option-learning frameworks under continuous-time rollouts.
Problem \ref{eq:problem_option_smdp} differs from Problem \ref{eq:problem_option_ils} in that the continuous-time joint chance constraint \eqref{eq:problem_option_ils_chance} is replaced by chance constraints imposed only at discrete interaction times through the function $H(\cdot)$ in \eqref{eq:problem_option_smdp_safe}. 
It becomes necessary to characterize how the function $H(\cdot)$ should be constructed so that the originally required constraint \eqref{eq:problem_option_ils_chance} remains satisfied under this reformulation.

\subsection{Selection of Safety Function}
\label{subsec:selection_safety_function}

To characterize safety between two clinical interactions, we introduce the option-wise growth property.

\begin{mydef}
\label{def:cl_growth}
Fix an interaction time $t_k$ and the executed option $o_k\in\mathcal{O}$. 
Consider the closed-loop trajectory generated by $\bm{\mathrm{u}}_t=\pi_{o_k}(\bm{\mathrm{x}}_t)$ on $t\in[t_k,t_k+\delta t_k]$. 
If there exist constants $\lambda_{o_k}\in\mathbb{R}$ and $\gamma_{o_k}\ge 0$, and a nonnegative quantity $b_{k,o_k}$ measurable with respect to $(\bm{\mathrm{s}}_k,o_k)$, such that for all $\tau\in[0,\delta t_k]$,
\begin{equation}
\label{eq:cl_growth_bound}
\|\bm{\mathrm{x}}_{t_k+\tau}-\bm{\mathrm{x}}_{k}\|
\le
\frac{e^{\lambda_{o_k}\tau}-1}{\lambda_{o_k}}\,b_{k,o_k}
+
\gamma_{o_k} e^{\lambda_{o_k}\tau}\,\bm{\mathrm{W}}_{k,\tau},
\qquad
\bm{\mathrm{W}}_{k,\tau}:=\int_{t_k}^{t_k+\tau}\|\bm{\mathrm{w}}_{s}\|\mathsf{d}s,
\end{equation}
where the expression $\frac{e^{\lambda_{o_k}\tau}-1}{\lambda_{o_k}}$ is interpreted as $\tau$ when $\lambda_{o_k}=0$. 
Then, we say the option $o_k$ satisfies an \textit{option-wise closed-loop $(\lambda_{o_k},b_{k,o_k},\gamma_{o_k})$-growth property} on $[t_k,t_k+\delta t_k]$ for $\bm{\mathrm{W}}_{k,\tau}$. 
\end{mydef}


\begin{myassump}
\label{assump:system_property}
The following conditions hold throughout this paper.
(a) The safety function $g$ is Lipschitz continuous with Lipschitz constant $L_{g,\bm{\mathrm{x}}}$.
(b) Define $\bm{\mathrm{W}}_{k}:=\bm{\mathrm{W}}_{k,\delta t_k}$. 
There exists a constant $p\ge 1$ such that $\mathbb{E}\left[\left(\bm{\mathrm{W}}_{k}\right)^p\right]<\infty$ for all $k$.
(c) For every interaction interval $k$ and any $\varepsilon\in(0,1)$, define
$\overline{W}_k(\varepsilon):=\left(\mathbb{E}\left[\left(\bm{\mathrm{W}}_{k}\right)^p\right]/\varepsilon\right)^{1/p}.$
Then, on the event $\{\bm{\mathrm{W}}_{k}\le \overline{W}_k(\varepsilon)\}$, there exists $(\lambda_{o_k},b_{k,o_k},\gamma_{o_k})$ such that the executed option $o_k$ satisfies the option-wise closed-loop $(\lambda_{o_k},b_{k,o_k},\gamma_{o_k})$-growth property in Definition \ref{def:cl_growth} on $[t_k,t_k+\delta t_k]$.
\end{myassump}

The parameters $(\lambda_{o_k},b_{k,o_k},\gamma_{o_k})$ summarize how the patient state may deviate during the interval governed by option $o_k$. 
In particular, $\lambda_{o_k}$ controls the rate of deviation growth, $b_{k,o_k}$ captures the nominal drift at the beginning of the interval, and $\gamma_{o_k}$ measures sensitivity to accumulated uncertainty through $\bm{\mathrm{W}}_{k,\tau}$. 
These quantities need not be known exactly; any valid upper bounds satisfying \eqref{eq:cl_growth_bound} are sufficient, and conservative choices only increase safety conservatism.
This option-wise growth property is weaker than assuming global Lipschitz continuity of the treatment dynamics. 
As shown in Appendix \ref{appendix:cl_growth_from_f}, the bound \eqref{eq:cl_growth_bound} follows from a radial inequality on the closed-loop vector field $f(\bm{\mathrm{x}},\pi_{o_k}(\bm{\mathrm{x}}),\bm{\mathrm{w}})$, which can hold even when $f$ is non-smooth or not globally Lipschitz. 
In particular, Lipschitz continuity of $f$ on a compact set implies the existence of parameters $(\lambda_{o_k},b_{k,o_k},\gamma_{o_k})$, whereas the converse does not hold. 
Therefore, the proposed formulation avoids imposing a strong global smoothness condition and instead uses a trajectory-level bound that is directly aligned with safe treatment evolution between clinical interactions.
We present a choice of $H(\cdot)$ that enforces safety over the entire treatment trajectory by checking safety only at clinical interaction times.

\begin{mytheo}
\label{theo:selection_safety}
Suppose that Assumption \ref{assump:system_property} holds. 
Then, for $\varepsilon\leq\alpha/(2K)$, any feasible solution of Problem \ref{eq:problem_option_smdp} is also feasible for Problem \ref{eq:problem_option_ils} in the following three cases: 
(a) On the event $\{\bm{\mathrm{W}}_{k}\le \overline{W}_k(\varepsilon)\}$, let $(\lambda_{o_k},b_{k,o_k},\gamma_{o_k})$ be any triple whose existence is guaranteed by Assumption \ref{assump:system_property}(c).
The function $H(\cdot)$ is chosen as
\begin{equation}
\label{eq:H_selection_main}
H\left(\bm{\mathrm{s}}_{k},o_k\right):=
g(\bm{\mathrm{x}}_k)
+
L_{g,\bm{\mathrm{x}}}\cdot
\left(
\frac{e^{|\lambda_{o_k}|\delta t_k}-1}{|\lambda_{o_k}|}\,b_{k,o_k}
+
\gamma_{o_k}e^{|\lambda_{o_k}|\delta t_k}\,\overline{W}_k(\varepsilon)
\right).
\end{equation}
(b) $\bm{\mathrm{W}}_{k,\delta t}$ is almost surely bounded by $W_{\mathsf{max},k}$ for all $k$ and $\delta t$, and $H(\cdot)$ is selected as in \eqref{eq:H_selection_main} with $\overline{W}_k(\varepsilon)$ replaced by $W_{\mathsf{max},k}$. 
(c) $\bm{\mathrm{W}}_{k,\delta t}$ is sub-Gaussian for all $k$ and $\delta t$, and $H(\cdot)$ is selected as in \eqref{eq:H_selection_main} with $\overline{W}_k(\varepsilon)$ replaced by $\mathbb{E}[\bm{\mathrm{W}}_k]
+
\sigma_k\sqrt{2\log\left(1/\varepsilon\right)}$, where $\sigma_k:=\sup_{\delta t\in[0,\delta t_k]}\ \sigma_{k,\delta t}$ and $\sigma_{k,\delta t}$ denotes the variance of $\bm{\mathrm{W}}_{k,\delta t}$.
\end{mytheo}

Theorem \ref{theo:selection_safety} shows that appropriately tightening the safety constraint at each clinical interaction time is sufficient to guarantee safety throughout the interval until the next interaction.
Thus, although Problem \ref{eq:problem_option_smdp} imposes constraints only at discrete interaction times, it still guarantees safety over the full continuous-time treatment trajectory.
The proof is provided in Appendix \ref{appendix:theo_selection_safety}.
The quantities in Theorem \ref{theo:selection_safety} are introduced to obtain a tractable upper bound on hidden safety violations between clinical interactions, rather than to impose restrictive treatment-dynamics assumptions.
In particular, $\overline{W}_k(\varepsilon)$ is a probabilistic upper envelope on the accumulated uncertainty over one interaction interval and does not require exact knowledge of the disturbance distribution.
In practice, it can be conservatively estimated from prior knowledge or learned from data.
Similarly, valid upper bounds on the relevant Lipschitz constants are sufficient, and over-approximation only increases conservatism without invalidating the guarantee.
Among the three cases, case (a) is the most general and only requires a finite moment condition on the integrated disturbance.
Cases (b) and (c) provide tighter tightenings when stronger information, such as boundedness or sub-Gaussianity, is available.
The bound is valid for both stable and unstable closed-loop behaviors, since the proof uses $|\lambda_{o_k}|$ to obtain a sign-independent worst-case bound.
When $\lambda_{o_k}<0$, this treatment is conservative, but it yields a unified construction of $H(\cdot)$ that is directly applicable in both regimes.

\subsection{Finite Sample Guarantees}
\label{subsection:finite_samples}

To study learning performance from finite clinical data, we consider a logged dataset of $M$ option-execution trajectories $\mathcal{D}_M:=\left\{\{(\bm{\mathrm{s}}_k^{(i)},o_k^{(i)})\}_{k=0}^{N-1}\right\}_{i=1}^N$, where trajectory $i$ contains the realized interaction-state and option sequence. 
Using $\mathcal{D}_M$, we construct the following importance-weighted sample approximation of Problem~\ref{eq:problem_option_smdp}:
\begin{align}
\max_{\mu,\{\pi_o,\beta_o\}_{o\in\mathcal{O}}}\ &\
\widetilde{J}\!\left(\mu,\{\pi_o,\beta_o\}_{o\in\mathcal{O}};\mathcal{D}_M\right)
:=
\sum_{i=1}^M
\omega_i\cdot
\left(
\sum_{k=0}^{N-1}
r(\bm{\mathrm{s}}_k^{(i)},o_k^{(i)})
\right)
\tag{$\mathsf{SA\text{-}Option\text{-}SMDP}$}
\label{eq:problem_option_smdp_sample_main}
\\
\mathsf{s.t.}\ &\
\sum_{i=1}^M
\omega_i\cdot
\mathbb{I}\!\left\{
H(\bm{\mathrm{s}}_k^{(i)},o_k^{(i)})\le \bm{0}
\right\}
\ge
1-\varepsilon,
\qquad
\forall k=1,\ldots,N.
\label{eq:problem_option_smdp_sample_main_safe}
\end{align}

Here, $\omega_i$ is short notation of $\omega_i\!\left(\bm{\mathrm{s}}_0,\mu,\{\pi_o,\beta_o\}_{o\in\mathcal{O}};\mathcal{D}_M\right)$, which denotes a normalized importance weight satisfying
$\sum_{i=1}^M \omega_i = 1.$
It reweights trajectory $i$ to compensate for the mismatch between the trajectory distribution induced by the candidate option-SMDP policy $(\mu,\{\pi_o,\beta_o\}_{o\in\mathcal{O}})$ and that of the logged data.
In practice, after parameterizing the policy as in Appendix~\ref{appendix:chance_constrained_optimization}, these weights can be estimated through conditional density-ratio estimation, for example, by a kernel density estimation (KDE) method~\citep{Hyndman1996} or a generator-based conditional density model~\citep{Zhou2023}.
The detailed trajectory-level construction is provided in Appendix~\ref{appendix:proofs_safety}, especially Appendix~\ref{appendix:importance_weighting} for importance weight and Appendix~\ref{appendix:uniform_exp_dev_residuals} for uniform exponential deviation for residual processes.
The finite-sample guarantee of Problem \ref{eq:problem_option_smdp_sample_main} is as follows.

\begin{mytheo}
\label{theo:finite_sample_main}
Suppose that Assumption \ref{assump:system_property} and some mild regularity condition stated in Appendix \ref{appendix:uniform_exp_dev_residuals} hold. 
Let $\widehat{\vartheta}_M$ and $\vartheta^\star$ denote a solution of the sample-based Problem
\ref{eq:problem_option_smdp_sample_main} and an optimal solution of the original Problem \ref{eq:problem_option_smdp}, 
where $\vartheta:=(\mu,\{\pi_o,\beta_o\}_{o\in\mathcal{O}})$.
Then for any tolerances $\delta_{\mathsf{obj}}>0$ and $\delta_{\mathsf{con}}>0$, there exist failure probabilities
$\rho_M^{\mathsf{obj}}$ and $\rho_M^{\mathsf{con}}$ such that, with probability at least $1-\rho_M^{\mathsf{obj}}-\rho_M^{\mathsf{con}},$
the following statements hold: \textbf{(a) Near optimality:} $J(\widehat{\vartheta}_M,\bm{\mathrm{s}}_0) \ge J(\vartheta^\star,\bm{\mathrm{s}}_0)-2\delta_{\mathsf{obj}}$; \textbf{(b) Feasibility:} If the empirical safety constraints in \eqref{eq:problem_option_smdp_sample_main_safe} are satisfied with $1-\varepsilon+\delta_{\mathsf{con}}$ instead of $1-\varepsilon$, then the learned solution is feasible for Problem \ref{eq:problem_option_smdp}.
\end{mytheo}

Theorem \ref{theo:finite_sample_main} shows, in a general form, that from finite logged treatment trajectories one can obtain a policy that is both approximately optimal and safety-feasible with high probability. 
For specific conditional density-ratio estimation methods, the resulting bounds depend on the corresponding hyper-parameters; see Theorem \ref{theo:consistency_residuals_two_methods} in Appendix \ref{appendix:uniform_exp_dev_residuals} for the feasibility bounds under KDE-based and generator-based estimation, and Propositions \ref{prop:near_opt_kde} and \ref{prop:near_opt_gen} in Appendix \ref{appendix:near_optimality} for the corresponding near-optimality bounds.
The proof is based on the consistency of the importance-weighted empirical objective and constraint processes.
Appendix~\ref{appendix:uniform_exp_dev_residuals} establishes uniform exponential deviation bounds for the estimated importance weights, Appendix~\ref{appendix:consistency_estimation} derives the resulting uniform consistency, and Appendix~\ref{appendix:near_optimality}--\ref{appendix:feasibility_result} translate these results into the finite-sample guarantees above.

\subsection{Practical Implementation} 
\label{subsec:practical_implementation}

The analytical safety-tightening term in \eqref{eq:H_selection_main} is generally not directly computable in practice, mainly because it depends on $\lambda_{o_k}$, which characterizes the local closed-loop growth behavior induced by the treatment dynamics and the option feedback controller $\pi_{o_k}$, while its explicit dependence on the current patient state $\bm{\mathrm{x}}_k$ is unknown.
To make the practical construction explicit, we define
\begin{equation}
    \label{eq:def_h_k_practical}
    h(\bm{\mathrm{x}}_k,\delta t_k)
    :=
    \frac{e^{|\lambda_{o_k}|\delta t_k}-1}{|\lambda_{o_k}|}\,b_{k,o_k}
    +
    \gamma_{o_k}e^{|\lambda_{o_k}|\delta t_k}\,\overline{W}_k(\varepsilon).
\end{equation}
The function $h(\bm{\mathrm{x}}_k,\delta t_k)$ represents the analytical safety-tightening term over the interaction interval $[t_k,t_k+\delta t_k]$.
This definition is natural because \eqref{eq:def_h_k_practical} is exactly the quantity that upper-bounds the inter-interaction state deviation in Theorem \ref{theo:selection_safety}, and thus determines the amount of safety tightening required at time $t_k$. 
Our goal in practical implementation is therefore to construct a computable surrogate that upper-bounds $h(\bm{\mathrm{x}}_k,\delta t_k)$. 
We formalize this notion as follows:
\begin{mydef}
    \label{def:conservative_surrogate}
    A function $\bar{h}(\bm{\mathrm{x}},\delta t)$ is called a conservative surrogate of $h(\bm{\mathrm{x}},\delta t)$ if $\bar{h}(\bm{\mathrm{x}},\delta t)\ge h(\bm{\mathrm{x}},\delta t),\ \forall (\bm{\mathrm{x}},\delta t)$ in the operating region of interest.
\end{mydef}

In practice, we assume access to a dataset
$\mathcal{D}^{\mathsf{pr}}_{N_{\mathsf{pr}}}:=
\left\{
\left(\bm{\mathrm{x}}^{(i)},\delta t^{(i)},d^{(i)}\right)
\right\}_{i=1}^{N_{\mathsf{pr}}}$,
where
$d^{(i)}:=
\left\|
\bm{\mathrm{x}}^{(i)}_{+}-\bm{\mathrm{x}}^{(i)}
\right\|$,
and $\bm{\mathrm{x}}^{(i)}_{+}$ denotes the successor state of $\bm{\mathrm{x}}^{(i)}$ after an inter-interaction duration $\delta t^{(i)}$. 
Here, the endpoint deviation is used as a practical proxy for the interval-wise deviation; see Appendix \ref{appendix:endpoint_deviation} for discussion.
Thus, each sample records the current patient state, the applied interaction duration, and the resulting sampled finite-interval state deviation.
We assume that the deviation variable $d$ follows a conditional probability distribution given $(\bm{\mathrm{x}},\delta t)$, and denote its Conditional Probability Density (CPD) by
$d\sim p_{\mathsf{d}}(\cdot\mid\bm{\mathrm{x}},\delta t)$.
This conditional distribution captures the stochastic variation of the sampled finite-interval deviation under the current state and interaction duration.
Its uncertainty integrates the effect of both treatment-dynamics uncertainty and the input component induced by the option policy, which is not explicitly retained in the practical implementation. 

We next construct a sample-based approximate conservative surrogate of $h(\bm{\mathrm{x}},\delta t)$. 
A KDE-based method can be used to estimate the approximate CPD $\tilde{p}_{\mathsf{d}}(\cdot|\bm{\mathrm{x}},\delta t)$ of $p_{\mathsf{d}}(\cdot|\bm{\mathrm{x}},\delta t)$. 
We then resample $N_{\mathsf{re}}$ samples $\tilde{d}_{\mathsf{re}}^{(j)}\sim\tilde{p}_{\mathsf{d}}(\cdot|\bm{\mathrm{x}},\delta t),\ j=1,...,N_{\mathsf{re}}.$ 
Without loss of generality, assume that the samples $\left\{\tilde{d}_{\mathsf{re}}^{(j)}\right\}_{j=1}^{N_{\mathsf{re}}}$ are sorted in descending order, namely, $\tilde{d}_{\mathsf{re}}^{(1)}\ge \tilde{d}_{\mathsf{re}}^{(2)}\ge \cdots \ge \tilde{d}_{\mathsf{re}}^{(N_{\mathsf{re}})}.$
Define $m_{\mathsf{re}}:=\left\lceil N_{\mathsf{re}}\bar{\epsilon}_h' \right\rceil.$
The sample-based approximation $\tilde{v}_{\bar{\epsilon}_h}^{N_{\mathsf{pr}},N_{\mathsf{re}}}(\bm{\mathrm{x}},\delta t)$ is defined as the $(N_{\mathsf{re}}\bar{\epsilon}_h')$-th largest resampled deviation, namely, $\tilde{v}_{\bar{\epsilon}_h}^{N_{\mathsf{pr}},N_{\mathsf{re}}}(\bm{\mathrm{x}},\delta t)=\tilde{d}_{\mathsf{re}}^{(m_{\mathsf{re}})}.$
This quantity provides a data-driven approximation of the conservative surrogate needed to control hidden safety violations between two clinical interaction times.
Instead of using $H\left(\bm{\mathrm{s}}_{k},o_k\right)$ in \eqref{eq:H_selection_main}, we introduce the empirical safety function
\begin{equation}
\label{eq:H_selection_empirical_main}
\widetilde{H}_{\bar{\epsilon}_h}^{N_{\mathsf{pr}},N_{\mathsf{re}}}(\bm{\mathrm{s}}_k)
:=
g(\bm{\mathrm{x}}_k)
+
L_{g,\bm{\mathrm{x}}}\cdot
\tilde{v}_{\bar{\epsilon}_h}^{N_{\mathsf{pr}},N_{\mathsf{re}}}(\bm{\mathrm{x}}_k,\delta t_k).
\end{equation}
Using \eqref{eq:H_selection_empirical_main} preserves feasibility for Problem \ref{eq:problem_option_smdp} with high probability.

\begin{mytheo}
\label{theo:selection_safety_empirical_main}
Suppose that Assumption \ref{assump:system_property} and some mild regularity condition stated in Appendix \ref{appendix:uniform_exp_dev_residuals} hold.  
Let $\widetilde{\vartheta}_{M,\bar{\epsilon}_h,\delta_{\mathsf{con}}}^{N_{\mathsf{pr}},N_{\mathsf{re}}}$ and $\vartheta^\star$ denote a solution of the sample-based Problem
\ref{eq:problem_option_smdp_sample_main} after replacing $H(\cdot)$ by $\widetilde{H}_{\bar{\epsilon}_h}^{N_{\mathsf{pr}},N_{\mathsf{re}}}(\cdot)$ defined by \eqref{eq:H_selection_empirical_main} and setting the probability level in the right side of \eqref{eq:problem_option_smdp_sample_main_safe} as $1-\varepsilon+\delta_{\mathsf{con}}$.
Then, there exist failure probability
$\rho_M^{\mathsf{con}}$ and $\rho_h$ such that, with probability at least $1-\rho_M^{\mathsf{con}}-\rho_h,$ $\widetilde{\vartheta}_{M,\bar{\epsilon}_h,\delta_{\mathsf{con}}}^{N_{\mathsf{pr}},N_{\mathsf{re}}}$ is feasible for Problem \ref{eq:problem_option_smdp}.
\end{mytheo}
The proof of Theorem \ref{theo:selection_safety_empirical_main} is summarized in Appendix \ref{appendix:proof_theo_selection_safety_empirical_main}.
Theorem \ref{theo:selection_safety_empirical_main} shows that the learned treatment policy remains safety-feasible with high probability even when the analytical safety-tightening term is replaced by a sample-based approximation. 
This is practically important in clinical settings, where safety must be certified from finite retrospective data rather than exact patient dynamics. 
The result links feasibility to the amount of logged treatment data and the accuracy of the conservative surrogate approximation, making the framework suitable for reliable policy deployment. 
While we adopt this KDE-based method to estimate $\tilde{v}_{\bar{\epsilon}_h}^{N_{\mathsf{pr}},N_{\mathsf{re}}}(\bm{\mathrm{x}}_k,\delta t_k)$ for theoretical guarantees, Gaussian process can approximate the support more conveniently (Appendix \ref{appendix:acs_gp}). 
In the experiments, we adopt a constant safety margin as a numerically stable approximation of the GP-based conservative surrogate. 
This can be interpreted as a robustified version of the adaptive margin, where a single conservative upper bound is chosen to cover a relevant region of $(\bm{\mathrm{x}},\delta t)$ rather than being adjusted pointwise. 
Such a simplification avoids the optimization instability caused by directly coupling the state- and duration-dependent margin with $\delta t_k$, while preserving the intended role of the surrogate as a conservative buffer against inter-interaction risk, details are provided in Appendix \ref{appendix:detailed_validation_results}.
After reformulating the original continuous-time problem as the option-based discrete-time constrained problem \ref{eq:problem_option_smdp}, existing safe reinforcement learning algorithms can be used as approximate solvers in practical implementation. 
In particular, although methods such as Constrained Policy Optimization (CPO) \citep{Achiam}, are originally developed for infinite-horizon discrete-time constrained Markov decision processes, they can still be applied here in a finite-horizon receding manner. 
This is reasonable because the essential inputs required by these algorithms, the transition samples, reward, and safety constraint, are all available after the option-SMDP reformulation. 
Therefore, the practical role of the present framework is to construct a continuous-time safety-aware option-level decision model, upon which existing safe RL methods can be deployed for approximate policy improvement.


\section{Validations}
\label{sec:validations}

\begin{table}[t]
\centering
\caption{Comparison across methods showing appropriate intensification rate (AIR), SOFA, lactate, and safety probability (mean $\pm$ standard deviation (SD)) for $K = 8$. 
All indicators are reported as mean $\pm$ SD over all evaluation trajectories from five test seeds, with 100 trajectories per seed.
$\uparrow$ ($\downarrow$) indicates that higher (lower) values are better. 
Blue(Red) indicates the best(worst) value in each row. 
The proposed methods, $\mathsf{CPO}$-$\mathsf{O}$ and $\mathsf{PCPO}$-$\mathsf{O}$, are highlighted in purple.
}
\label{table:clinical_results}
\resizebox{\textwidth}{!}{%
\begin{tabular}{lcccccccc}
\toprule
\textbf{Method}&  {\color{violet}$\mathsf{CPO}$-$\mathsf{O}$}  & $\mathsf{SAC}$-$\mathsf{O}$
 & $\mathsf{CPO}$-$\mathsf{E}$ & $\mathsf{SAC}$-$\mathsf{E}$  &  {\color{violet}$\mathsf{PCPO}$-$\mathsf{O}$}  & $\mathsf{TRPO}$-$\mathsf{O}$ &  $\mathsf{PCPO}$-$\mathsf{E}$  & $\mathsf{TRPO}$-$\mathsf{E}$
 \\
\midrule
\textbf{AIR}($\uparrow$)  & \textcolor{blue}{\textbf{87.60 $\pm$ 0.26}} & \textcolor{red}{\textbf{1.17 $\pm$ 0.45}} & 53.40 $\pm$ 3.06 & 47.54 $\pm$ 0.67 & 15.99 $\pm$ 0.78 & 71.56 $\pm$ 1.36 & 7.44 $\pm$ 0.86 & 24.85 $\pm$ 0.56 \\
\textbf{SOFA}($\downarrow$)  & \textcolor{blue}{\textbf{7.85 $\pm$ 0.03}} & \textcolor{red}{\textbf{8.92 $\pm$ 0.04}} & 8.15 $\pm$ 0.11 & 8.25 $\pm$ 0.07 & 7.92 $\pm$ 0.06 & 8.59 $\pm$ 0.13 & 8.06 $\pm$ 0.06 & 8.08 $\pm$ 0.11 \\ 
\textbf{Lactate}($\downarrow$) & 6.23 $\pm$ 0.20 & 6.60 $\pm$ 0.13 & 8.53 $\pm$ 0.95 & 8.24 $\pm$ 0.58 & \textcolor{blue}{\textbf{5.83 $\pm$ 0.15}} & \textcolor{red}{\textbf{9.75 $\pm$ 0.98}} & 7.85 $\pm$ 0.45 & 8.54 $\pm$ 0.74 \\
\textbf{Safety}(\%, $\uparrow$) & \textcolor{blue}{\textbf{100.00 $\pm$ 0.00}} & 25.80 $\pm$ 2.59 & 0.60 $\pm$ 0.89 & \textcolor{red}{\textbf{0.00 $\pm$ 0.00}} & 99.60 $\pm$ 0.55 & 3.20 $\pm$ 2.59 & 4.60 $\pm$ 1.67 & 0.20 $\pm$ 0.45 \\

\bottomrule
\end{tabular}}
\end{table}


\paragraph{Experimental Setup.} 
We evaluate the proposed method in a patient-specific continuous-time treatment environment constructed from MIMIC-III sepsis data \citep{Alistair}. 
Rather than using a discrete-time transition model, we build the environment by approximating the underlying physiological dynamics with a Physics-Informed Neural Network (PINN) \citep{RAISSI2019686} and an additional stochastic residual term. 
This construction preserves the continuous-time treatment structure, explicitly captures the effect of elapsed interaction intervals, and provides a substantially more realistic and challenging testbed than fixed-step simulators commonly used in treatment recommendation. 
The environment is built for one representative patient and does not include an explicit personalization parameter, allowing us to focus on validating the proposed safe continuous-time control framework under a high-dimensional and stochastic dynamics model. 
The present experiments evaluate the proposed framework in a learned continuous-time simulator constructed from clinical data, and primarily validate the interaction-limited safe control mechanism rather than the more challenging fully offline learning setting under strong dataset-coverage and out-of-distribution limitations.  
Detailed definitions of the state, reward, safety metrics, environment construction, and evaluation protocol are provided in Appendix \ref{appendix:experimental_details}. 
To further assess robustness beyond the primary patient-specific environment, we also provide supplementary evaluations using cluster-level PINN-based treatment environments in Appendix \ref{appendix:cluster_level_results}. 
We compare the proposed framework, instantiated with different policy optimization algorithms against methods using equidistant interaction times. In particular, $\mathsf{CPO}$-$\mathsf{O}$ ($\mathsf{CPO}$ \citep{Achiam} with optimal time-adaptation) and $\mathsf{PCPO}$-$\mathsf{O}$ ($\mathsf{PCPO}$ \citep{yang2020_pcpo} with optimal time-adaptation) are implementations of the proposed method, while $\mathsf{SAC}$-$\mathsf{O}$ ($\mathsf{SAC}$ \citep{haarnoja18b_sac} with optimal time-adaptation) and $\mathsf{TRPO}$-$\mathsf{O}$ ($\mathsf{TRPO}$ \citep{schulman15_trpo} with optimal time-adaptation) evaluate the proposed time-adaptation mechanism with unconstrained optimizers. As references, we further consider $\mathsf{CPO}$-$\mathsf{E}$ ($\mathsf{CPO}$ with equidistant interaction times), $\mathsf{SAC}$-$\mathsf{E}$ ($\mathsf{SAC}$ with equidistant interaction times), $\mathsf{PCPO}$-$\mathsf{E}$ ($\mathsf{PCPO}$ with equidistant interaction times), and $\mathsf{TRPO}$-$\mathsf{E}$ ($\mathsf{TRPO}$ with equidistant interaction times). 
We evaluate these methods from: treatment effectiveness, measured by the mean $\mathsf{SOFA}$ score per step; safety, measured by the safety rate after the critical time threshold; and trajectory-level behavior, measured by the $\mathsf{SOFA}$ score's time evolution.

\begin{figure}[htbp]
\centering
\includegraphics[width=0.975\linewidth]{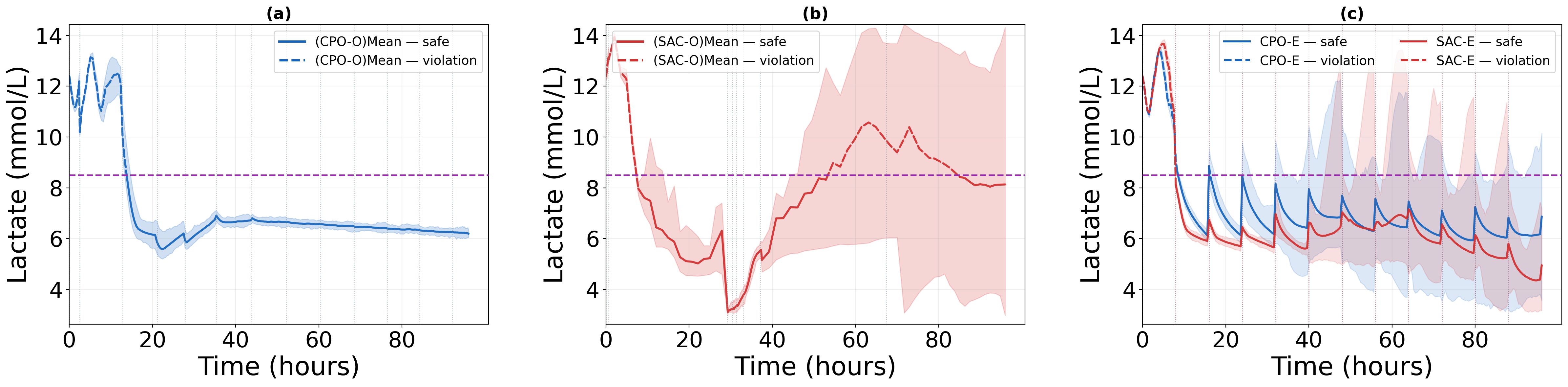}
\caption{Lactate evaluation results of different methods. The purple dashed horizontal line is the safety bound, set to $8.5$ mmol/L. In all panels, solid lines denote the mean lactate trajectories of safe rollouts, dashed lines denote the mean trajectories of violating rollouts, and shaded regions indicate variability bands. (a) $\mathsf{CPO}$-$\mathsf{O}$. (b) $\mathsf{SAC}$-$\mathsf{O}$. (c) Blue: $\mathsf{CPO}$-$\mathsf{E}$; Red: $\mathsf{SAC}$-$\mathsf{E}$, respectively.}
\label{fig:safety}
\end{figure}

\begin{figure}[htbp]
\centering
\includegraphics[width=0.975\linewidth]{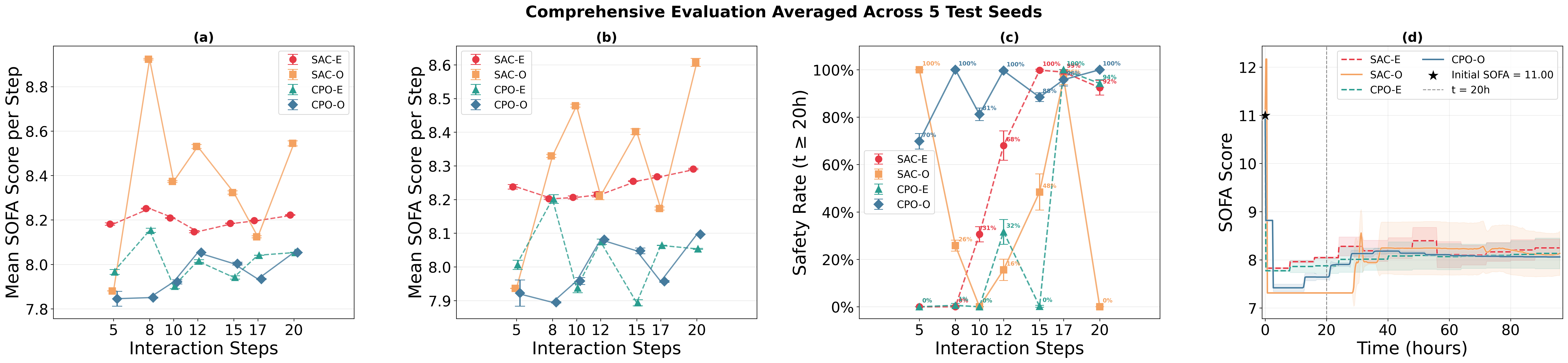}
\caption{Comprehensive evaluation of different methods, averaged across five test seeds. (a) Mean $\mathsf{SOFA}$ score per step under the first evaluation setting. (b) Mean $\mathsf{SOFA}$ score per step under the second evaluation setting. (c) Safety rate at $t\geq 20$h across different interaction steps. (d) Mean $\mathsf{SOFA}$ trajectories over time, where the vertical dashed line marks $t=20$h and the black star indicates the initial $\mathsf{SOFA}$ score. Error bars and shaded regions indicate variability across test seeds. The detailed definitions of the evaluation settings are provided in Appendix \ref{appendix:training}.}
\label{fig:comprehensive}
\end{figure}

\paragraph{Results and Discussion.} 
Figures \ref{fig:safety}--\ref{fig:comprehensive} and Table \ref{table:clinical_results} show that the proposed optimal time-adaptation mechanism consistently improves both treatment effectiveness and safety compared with equidistant interaction schemes. 
In particular, $\mathsf{CPO}$-$\mathsf{O}$ achieves lower lactate trajectories and lower mean $\mathsf{SOFA}$ scores than $\mathsf{CPO}$-$\mathsf{E}$, indicating that adaptive interaction timing is beneficial. 
Figure \ref{fig:safety} further shows that $\mathsf{CPO}$-$\mathsf{O}$ yields more stable lactate evolution and fewer safety-violating rollouts than $\mathsf{SAC}$-$\mathsf{O}$, suggesting that combining the proposed framework with constrained optimization provides additional safety benefits. 
The same trend is observed in Figure \ref{fig:comprehensive}(c)--(d), where constrained methods with optimal time-adaptation achieve higher safety rates after the critical time threshold and more favorable $\mathsf{SOFA}$ trajectories. 
These results are also consistent with the quantitative metrics in Table \ref{table:clinical_results}, where the proposed constrained optimal time-adaptation methods, $\mathsf{CPO}$-$\mathsf{O}$ and $\mathsf{PCPO}$-$\mathsf{O}$, achieve better AIR, lower $\mathsf{SOFA}$, and lower lactate than their equidistant counterparts. 
In contrast, the same trend is not consistently observed for unconstrained optimizers, which further highlights the importance of combining adaptive interaction timing with explicit safety-aware policy optimization.
Besides, the results of $\mathsf{PCPO}$-$\mathsf{O}$, $\mathsf{TRPO}$-$\mathsf{O}$, $\mathsf{PCPO}$-$\mathsf{E}$, and $\mathsf{TRPO}$-$\mathsf{E}$ follow the same pattern, showing that the benefit of the proposed framework is not tied to a specific optimizer but generalizes across both constrained and unconstrained policy optimization methods. 
Overall, the proposed resulting policy achieves a better balance between treatment effectiveness and safety, especially in terms of stabilizing clinically important physiological indicators and reducing unsafe trajectories. 
The constrained methods do not suffer a deterioration in treatment effectiveness and instead achieve better reward-related outcomes in this setting. From a clinical perspective, this is reasonable because constraining lactate prevents the policy from exploiting treatment strategies that may improve short-term reward while driving the patient toward physiologically unstable states. 
Since elevated lactate is a well-recognized marker of inadequate tissue perfusion and worsening sepsis severity, explicitly controlling lactate can steer the policy toward interventions that are both safer and more aligned with long-term recovery.
For all methods, each trained policy is evaluated $100$ times under each of five random seeds, and the table reports the resulting mean and standard deviation. 
Detailed training settings, evaluation protocols, training curves, state-action trajectories, safety evaluations, and the full results for $\mathsf{PCPO}$ and $\mathsf{TRPO}$ are deferred to Appendix \ref{appendix:training}--\ref{appendix:comprehensive_PCPO_TRPO}.

\section{Limitations}
\label{sec:limitations}

This work has several limitations that suggest important directions for future research. 
First, the current safety analysis focuses on the learned or converged policy, and does not explicitly address safety during policy exploration or training. Extending the framework to incorporate safe exploration mechanisms, such as shielding-based intervention, is an important next step. 
Second, the present formulation assumes full state observability at interaction times and synchronous treatment implementation, whereas real clinical decision-making is often partially observed and different interventions may be administered asynchronously. Incorporating continuous-time partial observability and more realistic treatment execution models would further improve clinical fidelity. 
Third, although this work focuses on policy learning with the setting of a full-coverage dataset, extending the framework to fully offline learning under partial coverage and stronger out-of-distribution shift in continuous-time settings remains an important direction. 
Fourth, the safety constraint in this work is based on blood lactate levels with a threshold of 8.5 mmol/L. While this value was chosen to accommodate the physiological complexity of sepsis, its clinical justification warrants further scrutiny. In standard ICU practice, lactate levels above 2 mmol/L are considered abnormal, and values exceeding 4 mmol/L are associated with significantly increased mortality. The relatively permissive threshold adopted here may therefore not fully reflect the clinical urgency with which elevated lactate is treated in practice. Future work should examine the sensitivity of policy behavior to this threshold and consider incorporating clinically validated cutoff values in consultation with critical care specialists. 
Fifth, the reward function in this work is based on the SOFA score, which has established correlations with patient mortality and is widely used in critical care research. However, the absolute differences in mean SOFA scores observed between methods in our experiments are modest, and whether such differences constitute clinically meaningful improvements in patient outcomes remains to be established. Future work should consider examining the relationship between observed SOFA improvements and longer-term outcomes such as 28-day mortality or ICU length of stay.

\section{Conclusions}
\label{sec:conclusions}

We studied interaction-limited safe continuous-time reinforcement learning for dynamical medical treatment. We formulated the problem as a continuous-time safety-constrained decision problem that jointly optimizes treatment administration and clinical interaction timing, developed an option-based reformulation together with a safety-tightening mechanism that guarantees safety between clinical interactions, and established finite-sample feasibility guarantees with a practical data-driven conservative surrogate. 
Experimental results on sepsis treatment showed that the proposed optimal time-adaptation mechanism improves both safety and treatment effectiveness over equidistant interaction schemes across different policy optimization methods.
We note that the absolute differences in SOFA score between methods are modest, and whether these translate into clinically meaningful improvements in patient outcomes requires further validation in collaboration with critical care specialists.

\subsubsection*{Broader Impact Statement}
The proposed framework aims to improve treatment decision-making in high-stakes clinical settings by jointly optimizing treatment administration and clinical interaction timing under trajectory-level safety constraints. This is particularly impactful in domains such as sepsis management in the ICU, where treatment decisions must balance efficacy against strict physiological safety requirements. The primary benefit lies in learning treatment policies that provably improve both safety and clinical effectiveness over standard interaction schemes, with finite-sample feasibility guarantees that ensure safety properties are preserved when learning from finite retrospective data. The safety constraints are grounded in clinically interpretable biomarkers such as blood lactate, making the framework more suitable for clinical decision-support than prior continuous-time treatment learning approaches that lack trajectory-level guarantees.

However, like all machine learning methods applied to healthcare, there are risks. The current framework assumes full state observability and synchronous treatment implementation, and safety guarantees apply to converged policies rather than during exploration. Improper deployment without adequate clinical oversight could introduce patient safety risks in settings where observations are partial or interventions asynchronous. We strongly emphasize that this framework is designed as a decision-support tool, not a substitute for clinical judgment, and responsible adoption requires rigorous prospective validation and close collaboration with healthcare professionals before any real-world deployment is considered.



\bibliography{references}
\bibliographystyle{tmlr}

\appendix


\section{Proof of Proposition~\ref{prop:oilscp_to_cttpil}}

\begin{proof}
Take any feasible solution of \ref{eq:problem_option_ils}. 
At each interaction time \(t_k\), let the selected option be \(o_k\in\mathcal O\).
By assumption, \(o_k\) induces an admissible segment controller
\[
\gamma_k:=(\pi_{o_k},\beta_{o_k})\in\Gamma.
\]
Define a high-level policy \(\nu\) for \ref{eq:problem_cttp_il} that selects exactly this controller \(\gamma_k\) whenever the option policy selects \(o_k\).

Then, on every interval \([t_k,t_{k+1})\), the control law under \ref{eq:problem_cttp_il} is 
\[
\bm{\mathrm{u}}_t=\phi_k(\bm{\mathrm{x}}_t)=\pi_{o_k}(\bm{\mathrm{x}}_t),
\]
which coincides with the intra-option control in \eqref{eq:intra_option_control}.
Likewise, the next interaction time is induced by the same termination rule, since \(\eta_k=\beta_{o_k}\).
Hence the resulting closed-loop state trajectory is identical pathwise, and therefore so are the accumulated reward and the trajectory-level safety event.
Thus every feasible solution of \ref{eq:problem_option_ils} induces a feasible solution of \ref{eq:problem_cttp_il} with identical performance.
Taking the supremum over feasible solutions yields
\[
V^\star_{\mathsf{O\text{-}ILSCP}}
\le
V^\star_{\mathsf{CTTP\text{-}IL}}.
\qedhere
\]
\end{proof}

\section{Proof of Theorem~\ref{thm:oilscp_exact}}
\begin{proof}
Proposition~\ref{prop:oilscp_to_cttpil} already shows that every feasible option-based policy induces a feasible interaction-limited treatment law.
It remains to prove the converse.

Take any feasible solution of \ref{eq:problem_cttp_il}.
At each interaction time \(t_k\), let the selected segment controller be
\[
\gamma_k=(\phi_k,\eta_k)\in\Gamma.
\]
By assumption, there exists an option \(o^{\gamma_k}\in\mathcal O\) such that
\[
\pi_{o^{\gamma_k}}=\phi_k,
\qquad
\beta_{o^{\gamma_k}}=\eta_k.
\]
Define a high-level option policy \(\mu\) that selects \(o^{\gamma_k}\) whenever the interaction-limited policy selects \(\gamma_k\).

Then, on each interval \([t_k,t_{k+1})\),
\[
\bm{\mathrm{u}}_t=\phi_k(\bm{\mathrm{x}}_t)=\pi_{o^{\gamma_k}}(\bm{\mathrm{x}}_t),
\]
and the next interaction time is induced by the same termination rule because
\(\eta_k=\beta_{o^{\gamma_k}}\).
Therefore the option-based system produces the same control trajectory, the same state trajectory, the same accumulated reward, and the same safety event pathwise.

Thus every feasible solution of \ref{eq:problem_cttp_il} is representable by \ref{eq:problem_option_ils}, and together with Proposition~\ref{prop:oilscp_to_cttpil}, the feasible closed-loop laws of the two problems coincide.
Hence
\[
V^\star_{\mathsf{O\text{-}ILSCP}}
=
V^\star_{\mathsf{CTTP\text{-}IL}}.
\qedhere
\]
\end{proof}

\section{Deriving \eqref{eq:cl_growth_bound} from a radial inequality}
\label{appendix:cl_growth_from_f}

Fix an interaction interval $k$ and the executed option $o_k\in\mathcal{O}$. 
Consider the closed-loop trajectory generated by $\bm{\mathrm{u}}_t=\pi_{o_k}(\bm{\mathrm{x}}_t)$ on $t\in[t_k,t_k+\delta t_k]$.

\begin{mylemma}[State deviation bound from a radial inequality]
\label{lemma:cl_growth_from_f}
Suppose that the closed-loop trajectory $\{\bm{\mathrm{x}}_t\}_{t\in[t_k,t_k+\delta t_k]}$ is absolutely continuous and satisfies \eqref{eq:system_dynamics}. 
Assume that there exist constants $\lambda_{o_k}\in\mathbb{R}$ and $\gamma_{o_k}\ge 0$, and a nonnegative quantity $b_{k,o_k}$ measurable with respect to $(\bm{\mathrm{s}}_k,o_k)$, such that for almost every $t\in[t_k,t_k+\delta t_k]$ with $\bm{\mathrm{x}}_t\ne \bm{\mathrm{x}}_k$,
\begin{equation}
\label{eq:radial_ineq}
\frac{\left\langle \bm{\mathrm{x}}_{t}-\bm{\mathrm{x}}_{k},\, f\!\left(\bm{\mathrm{x}}_{t},\pi_{o_k}(\bm{\mathrm{x}}_{t}),\bm{\mathrm{w}}_{t}\right)\right\rangle}{\|\bm{\mathrm{x}}_{t}-\bm{\mathrm{x}}_{k}\|}
\le
\lambda_{o_k}\|\bm{\mathrm{x}}_{t}-\bm{\mathrm{x}}_{k}\|
+
b_{k,o_k}
+
\gamma_{o_k}\|\bm{\mathrm{w}}_{t}\|.
\end{equation}
Then, for all $\tau\in[0,\delta t_k]$, the option-wise closed-loop growth bound \eqref{eq:cl_growth_bound} holds.
\end{mylemma}

\begin{proof}
Define $V(t):=\|\bm{\mathrm{x}}_{t}-\bm{\mathrm{x}}_{k}\|$ for $t\in[t_k,t_k+\delta t_k]$. 
Since $\bm{\mathrm{x}}_t$ is absolutely continuous, $V(t)$ is absolutely continuous and is differentiable for almost every $t$ with $\bm{\mathrm{x}}_t\ne \bm{\mathrm{x}}_k$. 
For such $t$,
\[
\frac{\mathsf{d}}{\mathsf{d}t}V(t)
=
\frac{\left\langle \bm{\mathrm{x}}_{t}-\bm{\mathrm{x}}_{k},\, \dot{\bm{\mathrm{x}}}_{t}\right\rangle}{\|\bm{\mathrm{x}}_{t}-\bm{\mathrm{x}}_{k}\|}
=
\frac{\left\langle \bm{\mathrm{x}}_{t}-\bm{\mathrm{x}}_{k},\, f\!\left(\bm{\mathrm{x}}_{t},\pi_{o_k}(\bm{\mathrm{x}}_{t}),\bm{\mathrm{w}}_{t}\right)\right\rangle}{\|\bm{\mathrm{x}}_{t}-\bm{\mathrm{x}}_{k}\|}.
\]
Combining with \eqref{eq:radial_ineq} yields, for almost every $t\in[t_k,t_k+\delta t_k]$,
\[
\frac{\mathsf{d}}{\mathsf{d}t}V(t)\le \lambda_{o_k}V(t)+b_{k,o_k}+\gamma_{o_k}\|\bm{\mathrm{w}}_{t}\|.
\]
Applying the integral form of Gr\"onwall inequality \citep{Gronwall1919} on $[t_k,t_k+\tau]$ and using $V(t_k)=0$ gives
\[
V(t_k+\tau)
\le
\frac{e^{\lambda_{o_k}\tau}-1}{\lambda_{o_k}}\,b_{k,o_k}
+
\gamma_{o_k} e^{\lambda_{o_k}\tau}\int_{t_k}^{t_k+\tau}\|\bm{\mathrm{w}}_{s}\|\mathsf{d}s,
\]
with the convention $\frac{e^{\lambda_{o_k}\tau}-1}{\lambda_{o_k}}=\tau$ when $\lambda_{o_k}=0$. 
Recalling $\bm{\mathrm{W}}_{k,\tau}:=\int_{t_k}^{t_k+\tau}\|\bm{\mathrm{w}}_{s}\|\mathsf{d}s$ yields \eqref{eq:cl_growth_bound}.
\end{proof}

\section{Proof of Theorem \ref{theo:selection_safety}}
\label{appendix:theo_selection_safety}

\subsection{Preparation: state deviation bound}

Fix an interaction interval $k$ and define $\delta t_k:=t_{k+1}-t_k$.
Note that $\delta t\in[0,\delta t_k]$.

Let
\[
V(t):=\|\bm{\mathrm{x}}_t-\bm{\mathrm{x}}_k\|.
\]

By Lemma~\ref{lemma:cl_growth_from_f}, the closed-loop trajectory satisfies
\[
\frac{\mathsf{d}}{\mathsf{d}t}V(t)
\le
\lambda_{o_k}V(t)+b_{k,o_k}+\gamma_{o_k}\|\bm{\mathrm{w}}_t\|
\]
for almost every $t\in[t_k,t_k+\delta t_k]$.

Since $V(t)\ge0$, we have
\[
\lambda_{o_k}V(t)\le |\lambda_{o_k}|V(t),
\]
and therefore
\[
\frac{\mathsf{d}}{\mathsf{d}t}V(t)
\le
|\lambda_{o_k}|V(t)+b_{k,o_k}+\gamma_{o_k}\|\bm{\mathrm{w}}_t\|.
\]

Applying the integral form of Gr\"onwall's inequality
\citep{Gronwall1919} on $[t_k,t_k+\tau]$ with $V(t_k)=0$
yields, for all $\tau\in[0,\delta t_k]$,
\begin{equation}
\label{eq:pathwise_deviation_bound}
\|\bm{\mathrm{x}}_{t_k+\tau}-\bm{\mathrm{x}}_{k}\|
\le
\frac{e^{|\lambda_{o_k}|\tau}-1}{|\lambda_{o_k}|}\,b_{k,o_k}
+
\gamma_{o_k}e^{|\lambda_{o_k}|\tau}\bm{\mathrm{W}}_{k,\tau},
\end{equation}
where
\[
\bm{\mathrm{W}}_{k,\tau}:=\int_{t_k}^{t_k+\tau}\|\bm{\mathrm{w}}_s\|\mathsf{d}s.
\]

Consequently,
\begin{equation}
\label{eq:sup_deviation_bound}
\sup_{t\in[t_k,t_{k+1}]}\|\bm{\mathrm{x}}_{t}-\bm{\mathrm{x}}_{k}\|
\le
\frac{e^{|\lambda_{o_k}|\delta t_k}-1}{|\lambda_{o_k}|}\,b_{k,o_k}
+
\gamma_{o_k}e^{|\lambda_{o_k}|\delta t_k}\bm{\mathrm{W}}_{k}.
\end{equation}

\subsection{Proof for (a)}

Fix $k$ and let $\varepsilon\in(0,1)$.
Choose $\overline{W}_k(\varepsilon)$ such that
\begin{equation}
\label{eq:W_quantile_event_refined}
\mathsf{Pr}\left\{\bm{\mathrm{W}}_{k}\le \overline{W}_k(\varepsilon)\right\}\ge 1-\varepsilon.
\end{equation}

For example, by Markov's inequality and Assumption~\ref{assump:system_property}(b),
\[
\overline{W}_k(\varepsilon)
=
\left(
\frac{\mathbb{E}\left[\left(\bm{\mathrm{W}}_{k}\right)^p\right]}{\varepsilon}
\right)^{1/p}.
\]

Define
\begin{equation}
\label{eq:Gamma_def}
\Gamma_k(\varepsilon):=
\frac{e^{|\lambda_{o_k}|\delta t_k}-1}{|\lambda_{o_k}|}\,b_{k,o_k}
+
\gamma_{o_k}e^{|\lambda_{o_k}|\delta t_k}\overline{W}_k(\varepsilon),
\end{equation}
and
\begin{equation}
\label{eq:H_def}
H(\bm{\mathrm{x}}_k,o_k):=
g(\bm{\mathrm{x}}_k)+L_{g,\bm{\mathrm{x}}}\Gamma_k(\varepsilon).
\end{equation}

Let
\[
\mathcal{E}_k(\varepsilon):=
\{\bm{\mathrm{W}}_{k}\le \overline{W}_k(\varepsilon)\}.
\]

On $\mathcal{E}_k(\varepsilon)$,
\eqref{eq:sup_deviation_bound} implies
\[
\sup_{t\in[t_k,t_{k+1}]}\|\bm{\mathrm{x}}_{t}-\bm{\mathrm{x}}_{k}\|
\le
\Gamma_k(\varepsilon).
\]

Using the Lipschitz continuity of $g$ (Assumption~\ref{assump:system_property}(a)),
for all $t\in[t_k,t_{k+1}]$,
\[
g(\bm{\mathrm{x}}_{t})
\le
g(\bm{\mathrm{x}}_{k})
+
L_{g,\bm{\mathrm{x}}}\|\bm{\mathrm{x}}_{t}-\bm{\mathrm{x}}_{k}\|
\le
g(\bm{\mathrm{x}}_{k})+L_{g,\bm{\mathrm{x}}}\Gamma_k(\varepsilon)
=
H(\bm{\mathrm{x}}_k,o_k).
\]

Therefore, since $\mathsf{Pr}\left\{H(\bm{\mathrm{x}}_k,o_k)\le 0\right\}\geq 1-\varepsilon$ holds for any feasible solution of Problem \ref{eq:problem_option_smdp},
then $g(\bm{\mathrm{x}}_{t})\le0$ for all $t\in[t_k,t_{k+1}]$
on $\mathcal{E}_k(\varepsilon)\cap \left\{H(\bm{\mathrm{x}}_k,o_k)\le 0\right\}$.
Hence
\[
\mathsf{Pr}\left\{
g(\bm{\mathrm{x}}_{t})\le0,\ \forall t\in[t_k,t_{k+1}]
\right\}
\ge
\mathsf{Pr}\left(\mathcal{E}_k(\varepsilon)\cap \left\{H(\bm{\mathrm{x}}_k,o_k)\le 0\right\}\right)
\ge
1-2\varepsilon.
\]

Since $\varepsilon\le\alpha/(2K)$, applying Boole's inequality
\citep{Galambos1977} over at most $K$ interaction intervals
ensures that the continuous-time chance constraint
\eqref{eq:problem_option_ils_chance} holds.

\subsection{Proof for (b)}

Define
\[
\Gamma_k^{\max}:=
\frac{e^{|\lambda_{o_k}|\delta t_k}-1}{|\lambda_{o_k}|}\,b_{k,o_k}
+
\gamma_{o_k}e^{|\lambda_{o_k}|\delta t_k}W_{\mathsf{max},k}.
\]

Since $\bm{\mathrm{W}}_k\le W_{\mathsf{max},k}$ almost surely,
the bound follows directly from \eqref{eq:sup_deviation_bound},
and the remainder proceeds as in part (a).

\subsection{Proof for (c)}

Suppose $\bm{\mathrm{W}}_k$ is sub-Gaussian around its mean,
i.e., there exists $\sigma_k>0$ such that
\[
\mathsf{Pr}\{\bm{\mathrm{W}}_k-\mathbb{E}[\bm{\mathrm{W}}_k]\ge s\}
\le
\exp\!\left(-\frac{s^2}{2\sigma_k^2}\right).
\]

Choosing
\[
\overline{W}_k(\varepsilon)
=
\mathbb{E}[\bm{\mathrm{W}}_k]
+
\sigma_k\sqrt{2\log\left(\frac{1}{\varepsilon}\right)}
\]
ensures
\[
\mathsf{Pr}\{\bm{\mathrm{W}}_k\le\overline{W}_k(\varepsilon)\}\ge1-\varepsilon.
\]

Substituting this bound into part (a)
yields the desired safety guarantee.

\section{Details of Section \ref{subsection:finite_samples}'s Theoretical Results}
\label{appendix:proofs_safety}

\subsection{Chance constrained optimization with decision-dependent uncertainty}
\label{appendix:chance_constrained_optimization}

Problem~\ref{eq:problem_option_smdp} seeks an option-SMDP policy that maps the state
$\bm{\mathrm{s}}_k$ defined in \eqref{eq:option_aug_state} to an option
$o_k=(\pi_{o_k},\beta_{o_k})$.
For practical implementation, both the high-level selector $\mu$ and the intra-option components
$(\pi_{o_k},\beta_{o_k})$ are parameterized.
Without loss of generality, let $\theta\in\mathbb{R}^{n_\theta}$ denote the parameter vector that specifies the entire option-SMDP policy, including:
(i) the high-level option-selection rule, and  
(ii) the intra-option feedback policy and termination rule.
Once $\theta$ is fixed, the induced policy (deterministic or stochastic) determines the distribution of trajectories $\{\bm{\mathrm{s}}_k\}$.

\paragraph{Option-induced transition under parameterization.}
Conditioned on the interaction state $\bm{\mathrm{s}}_k=(\bm{\mathrm{x}}_k,t_k,k)$, the executed option $o_k$ and its duration $\delta t_k$ are determined by $\theta$ together with exogenous randomness.
To make this dependence explicit, we introduce an auxiliary random element $\xi_k$ that collects all exogenous randomness over the interaction interval $[t_k,t_k+\delta t_k]$, including the disturbance sample path $\{\bm{\mathrm{w}}_t\}_{t\in[t_k,t_k+\delta t_k]}$ and any stochasticity in option termination.
We allow $\xi_k$ to be state-dependent and write
$\xi_k \sim P(\cdot \mid \bm{\mathrm{s}}_k).$
The next interaction state is then written in the form
\begin{equation}
\label{eq:param_option_transition_map}
\bm{\mathrm{s}}_{k+1}=\Phi_{\theta}(\bm{\mathrm{s}}_k,\xi_k),\quad \xi_k \sim P(\cdot \mid \bm{\mathrm{s}}_k),
\end{equation}
where $\Phi_{\theta}$ denotes the \emph{option-induced transition map} under parameter $\theta$.

Let $\bm{\xi}:=\left\{\xi_k\right\}_{k=0}^{N-1}$ denote the exogenous randomness over the entire time horizon.
The joint random trajectory $\bm{\xi}$ follows the conditional probability distribution
\begin{equation}
    \label{eq:bm_xi_cpd}
    \bm{\xi}\sim\mathbb{P}(\cdot|\bm{\mathrm{s}}_0,\theta).
\end{equation}

\paragraph{Definition of $J(\theta,\bm{\mathrm{s}}_0)$.}
Fix $\theta\in\Theta$.
Let $\bm{\xi}$ denote the exogenous randomness over the time horizon, distributed according to \eqref{eq:bm_xi_cpd}.
Given $\bm{\xi}$ and the initial state $\bm{\mathrm{s}}_0$, the interaction-state trajectory $\{\bm{\mathrm{s}}_k\}_{k=0}^{N}$ is generated recursively by \eqref{eq:param_option_transition_map}.
The objective function in Problem~\ref{eq:problem_option_smdp} is redefined by
\begin{equation}
\label{eq:J_theta_formal}
J(\theta,\bm{\mathrm{s}}_0)
:=
\mathbb{E}
\left[
\sum_{k=0}^{N-1}
r\!\left(\bm{\mathrm{s}}_k,o_k\right) \mid \bm{\mathrm{s}}_0,\theta
\right],\quad \bm{\xi}\sim\mathbb{P}(\cdot|\bm{\mathrm{s}}_0,\theta).
\end{equation}

\paragraph{Definition of $h_k$.}
Fix $\theta\in\Theta$ and an initial interaction state $\bm{\mathrm{s}}_0$.
Let $\bm{\xi}$ denote the exogenous randomness over the horizon with
$\bm{\xi}\sim\mathbb{P}(\cdot\mid \bm{\mathrm{s}}_0,\theta)$ as in \eqref{eq:bm_xi_cpd}.
Given $(\bm{\mathrm{s}}_0,\theta,\bm{\xi})$, define the interaction-state trajectory
$\{\bm{\mathrm{s}}_k\}_{k=0}^{N}$ recursively by \eqref{eq:param_option_transition_map}.
Then, for each interaction index $k\in\{1,\ldots,N\}$, define the constraint by
\begin{equation}
\label{eq:hk_def_random}
h_k(\theta,\bm{\mathrm{s}}_0,\bm{\xi})
:=
H\!\left(\bm{\mathrm{s}}_k,o_k\right),
\end{equation}
where $H(\cdot)$ is the safety function in \eqref{eq:problem_option_smdp_safe}.
Accordingly, the chance constraint in \eqref{eq:problem_option_smdp__function_CCOsafe} can be written as
\begin{equation}
\label{eq:hk_chance_form}
\mathsf{Pr}\left\{ h_k(\theta,\bm{\mathrm{s}}_0,\bm{\xi})\le \bm{0}\mid \bm{\mathrm{s}}_0,\theta\right\}\ge 1-\varepsilon,
\qquad \bm{\xi}\sim\mathbb{P}(\cdot\mid \bm{\mathrm{s}}_0,\theta).
\end{equation}

\paragraph{Problem formulation under parameterization.}
Under the parameterization described above, Problem~\ref{eq:problem_option_smdp}
reduces to a finite-dimensional chance-constrained optimization problem over
$\theta\in\Theta$, written as
\begin{align}
\max_{\theta\in\Theta}\ &\
J(\theta,\bm{\mathrm{s}}_0)
:=
\mathbb{E}
\left[
\sum_{k=0}^{N-1}
r\!\left(\bm{\mathrm{s}}_k,o_k\right) \mid \bm{\mathrm{s}}_0,\theta
\right],\quad \bm{\xi}\sim\mathbb{P}(\cdot|\bm{\mathrm{s}}_0,\theta)\tag{$\mathsf{CCO}$}\label{eq:problem_option_smdp_function_CCO}\\
\mathsf{s.t.}\ &\
\mathsf{Pr}\left\{ h_k(\theta,\bm{\mathrm{s}}_0,\bm{\xi})\le \bm{0}\mid \bm{\mathrm{s}}_0,\theta\right\}\ge 1-\varepsilon,\ \forall k=1,...,N,\quad \bm{\xi}\sim\mathbb{P}(\cdot|\bm{\mathrm{s}}_0,\theta).\label{eq:problem_option_smdp__function_CCOsafe}
\end{align}

\subsection{Sample-based approximate problem}

\paragraph{Dataset under decision-dependent randomness.}
Fix $\bm{\mathrm{s}}_0$.
A logged dataset of $M$ trajectories is
\begin{equation}
\label{eq:logged_dataset}
\mathcal{D}_M
:=
\left\{\left(\theta^{(i)},\bm{\xi}^{(i)}\right)\right\}_{i=1}^{M},
\qquad
\bm{\xi}^{(i)}\sim\mathbb{P}(\cdot|\bm{\mathrm{s}}_0,\theta^{(i)}),
\end{equation}
where $\theta^{(i)}$ is the parameter used to generate trajectory $i$, and $\bm{\xi}^{(i)}$ collects all exogenous randomness over the horizon as in \eqref{eq:bm_xi_cpd}.
Given $(\bm{\mathrm{s}}_0,\theta^{(i)},\bm{\xi}^{(i)})$, the corresponding interaction-state trajectory
$\{\bm{\mathrm{s}}_k^{(i)}\}_{k=0}^{N}$ is generated recursively by \eqref{eq:param_option_transition_map}. 
The dataset $\mathcal{D}_M=\{(\theta^{(i)},\bm{\xi}^{(i)})\}_{i=1}^{M}$ in \eqref{eq:logged_dataset}
is used here to model a logged-data setting in which trajectories are collected independently under one or multiple fixed data-generating parameters $\theta^{(i)}$.
The generating parameters may differ across trajectories, but they are treated as fixed for the purpose of the theoretical analysis.
This formulation covers the offline logged-trajectory setting studied in this paper and is sufficient for the subsequent importance-weighted analysis.

\paragraph{Sample-based approximation with important sampling.}
With dataset $\mathcal{D}_M$ defined by \eqref{eq:logged_dataset}, Problem~\ref{eq:problem_option_smdp_function_CCO}
can be approximated by
\begin{align}
\max_{\theta\in\Theta}\ &\
\widetilde{J}(\theta,\bm{\mathrm{s}}_0,\mathcal{D}_M)
:=
\sum_{i=1}^M \omega\left(\bm{\mathrm{s}}_0,\theta,\theta^{(i)}\right) \cdot
\left(
\sum_{k=0}^{N-1}
r\!\left(\bm{\mathrm{s}}_k^{(i)},o_k^{(i)}\right) 
\right) \tag{$\mathsf{SA}$}\label{eq:problem_option_smdp_function_CCO_SA}\\
\mathsf{s.t.}\ &\
\sum_{i=1}^M \omega\left(\bm{\mathrm{s}}_0,\theta,\theta^{(i)}\right)\cdot\mathbb{I}\left\{ h_k\left(\theta,\bm{\mathrm{s}}_0,\bm{\xi}^{(i)}\right)\le \bm{0}\right\}\ge 1-\varepsilon,\ \forall k=1,...,N.\label{eq:problem_option_smdp__function_CCO_SAsafe}
\end{align}
The weight $\omega\!\left(\bm{\mathrm{s}}_0,\theta,\theta^{(i)}\right)$ in Problem
\ref{eq:problem_option_smdp_function_CCO_SA} is an importance weight that compensates for the discrepancy between the trajectory distribution induced by the candidate parameter $\theta$ and that induced by the data-generating parameter $\theta^{(i)}$. 
Intuitively, it reweights each logged trajectory so that the weighted empirical distribution approximates $\mathbb{P}(\cdot|\bm{\mathrm{s}}_0,\theta)$.
More precisely, when the trajectory distribution under $\theta$ is absolutely continuous with respect to the data-generating distribution under $\theta^{(i)}$, the weight $\omega$ corresponds to a likelihood ratio between these two measures.
The explicit form of $\omega$ will be introduced later.
With properly defined importance weights and sufficient coverage of the parameter space by the logged dataset, the sample-based approximation in
Problem
\ref{eq:problem_option_smdp_function_CCO_SA}
can provide a consistent estimator of both the objective and the chance constraints for any $\theta\in\Theta$.
In particular, when the dataset is generated on-policy (i.e., $\theta^{(i)}=\theta$), the weights reduce to uniform averaging.

\subsection{Example of importance weighting}
\label{appendix:importance_weighting}

Let $\widehat p_M(\bm{\xi}\mid \bm{\mathrm{s}}_0,\theta;\mathcal{D}_M)$ denote an estimator of the conditional density
$p(\bm{\xi}\mid \bm{\mathrm{s}}_0,\theta)$ constructed from the logged dataset $\mathcal{D}_M$
(e.g., via kernel conditional density estimation or any other conditional density model).
This notation makes explicit that $\widehat p_M(\cdot\mid \bm{\mathrm{s}}_0,\theta;\mathcal{D}_M)$ depends on the entire dataset and hence, implicitly, on $\{\theta^{(i)}\}_{i=1}^M$.

We define the estimated importance weight associated with sample $i$ for a candidate parameter $\theta$ as
\[
\omega_M\!\left(\bm{\mathrm{s}}_0,\theta,\theta^{(i)};\mathcal{D}_M\right)
:=
\frac{
\widehat p_M\!\left(\bm{\xi}^{(i)}\mid \bm{\mathrm{s}}_0,\theta;\mathcal{D}_M\right)
}{
\frac{1}{M}\sum_{j=1}^M 
\widehat p_M\!\left(\bm{\xi}^{(i)}\mid \bm{\mathrm{s}}_0,\theta^{(j)};\mathcal{D}_M\right)
}.
\]
Define the normalized importance weights by
\[
\bar{\omega}_M\!\left(\bm{\mathrm{s}}_0,\theta,\theta^{(i)};\mathcal{D}_M\right)
:=
\frac{
\omega_M\!\left(\bm{\mathrm{s}}_0,\theta,\theta^{(i)};\mathcal{D}_M\right)
}{
\sum_{j=1}^M
\omega_M\!\left(\bm{\mathrm{s}}_0,\theta,\theta^{(j)};\mathcal{D}_M\right)
}.
\]
For notational simplicity, we write $\omega_M$ in place of $\bar{\omega}_M$ in the remainder. 
The denominator corresponds to an empirical (estimated) mixture approximation of the proposal distribution
\[
q_M(\bm{\xi}\mid \bm{\mathrm{s}}_0)
:=
\frac{1}{M}\sum_{j=1}^M 
p(\bm{\xi}\mid \bm{\mathrm{s}}_0,\theta^{(j)}),
\]
so that $\omega_M$ acts as a density-ratio estimator approximating
\[
\frac{
p(\bm{\xi}^{(i)}\mid \bm{\mathrm{s}}_0,\theta)
}{
q_M(\bm{\xi}^{(i)}\mid \bm{\mathrm{s}}_0)
}.
\]

Let $\varphi(\bm{\xi},\theta)$ be any measurable functional such that
$\sup_{\theta\in\Theta}\sup_{\bm{\xi}}|\varphi(\bm{\xi},\theta)|<\infty$.
The weighted empirical estimator
\[
\sum_{i=1}^M 
\omega_M\!\left(\bm{\mathrm{s}}_0,\theta,\theta^{(i)};\mathcal{D}_M\right)
\varphi(\bm{\xi}^{(i)},\theta)
\]
serves as an approximation of
\[
\mathbb{E}\!\left[\varphi(\bm{\xi},\theta)\mid \bm{\mathrm{s}}_0,\theta\right],
\]
provided that the conditional density estimator is consistent and the mixture proposal has sufficient support over the target distributions.


\paragraph{Objective process.}
For each $\theta \in \Theta$, define the population objective functional
\begin{equation}
\label{eq:J_process_population}
\mathcal{J}(\theta,\bm{\mathrm{s}}_0)
:=
\mathbb{E}
\left[
\sum_{k=0}^{N-1}
r(\bm{\mathrm{s}}_k,o_k)
\;\middle|\;
\bm{\mathrm{s}}_0,\theta
\right].
\end{equation}

Given the logged dataset $\mathcal{D}_M=\{(\theta^{(i)},\bm{\xi}^{(i)})\}_{i=1}^M$,
define the importance-weighted empirical objective process
\begin{equation}
\label{eq:J_process_empirical}
\mathcal{J}_M(\theta,\bm{\mathrm{s}}_0)
:=
\sum_{i=1}^M
\omega\!\left(\bm{\mathrm{s}}_0,\theta,\theta^{(i)}\right)
\sum_{k=0}^{N-1}
r\!\left(\bm{\mathrm{s}}_k^{(i)},o_k^{(i)}\right).
\end{equation}

\paragraph{Constraint process.}
For each interaction index $k \in \{1,\dots,N\}$ and $\theta \in \Theta$,
define the population constraint probability
\begin{equation}
\label{eq:p_process_population}
\mathcal{P}_k(\theta,\bm{\mathrm{s}}_0)
:=
\mathsf{Pr}
\left\{
h_k(\theta,\bm{\mathrm{s}}_0,\bm{\xi}) \le \bm{0}
\;\middle|\;
\bm{\mathrm{s}}_0,\theta
\right\}.
\end{equation}

The corresponding importance-weighted empirical constraint process is defined as
\begin{equation}
\label{eq:p_process_empirical}
\mathcal{P}_{M,k}(\theta,\bm{\mathrm{s}}_0)
:=
\sum_{i=1}^M
\omega\!\left(\bm{\mathrm{s}}_0,\theta,\theta^{(i)}\right)
\mathbb{I}
\left\{
h_k(\theta,\bm{\mathrm{s}}_0,\bm{\xi}^{(i)}) \le \bm{0}
\right\}.
\end{equation}


\paragraph{Residual processes.}
For each $\theta\in\Theta$, define the objective residual process
\begin{equation}
\label{eq:obj_residual_pointwise}
\mathcal{R}^{\mathsf{obj}}_M(\theta,\bm{\mathrm{s}}_0)
:=
\left|
\mathcal{J}_M(\theta,\bm{\mathrm{s}}_0)
-
\mathcal{J}(\theta,\bm{\mathrm{s}}_0)
\right|,
\end{equation}
and its uniform version over the parameter space
\begin{equation}
\label{eq:obj_residual_uniform}
\overline{\mathcal{R}}^{\mathsf{obj}}_M(\bm{\mathrm{s}}_0)
:=
\sup_{\theta\in\Theta}
\mathcal{R}^{\mathsf{obj}}_M(\theta,\bm{\mathrm{s}}_0).
\end{equation}

For each interaction index $k\in\{1,\dots,N\}$ and $\theta\in\Theta$, define the constraint residual process
\begin{equation}
\label{eq:con_residual_pointwise}
\mathcal{R}^{\mathsf{con}}_{M,k}(\theta,\bm{\mathrm{s}}_0)
:=
\left|
\mathcal{P}_{M,k}(\theta,\bm{\mathrm{s}}_0)
-
\mathcal{P}_k(\theta,\bm{\mathrm{s}}_0)
\right|,
\end{equation}
and its uniform version
\begin{equation}
\label{eq:con_residual_uniform}
\overline{\mathcal{R}}^{\mathsf{con}}_{M,k}(\bm{\mathrm{s}}_0)
:=
\sup_{\theta\in\Theta}
\mathcal{R}^{\mathsf{con}}_{M,k}(\theta,\bm{\mathrm{s}}_0).
\end{equation}



\subsection{Uniform exponential deviation for residual processes}
\label{appendix:uniform_exp_dev_residuals}

Throughout, fix $\bm{\mathrm{s}}_0$ and assume that $\Theta$ is compact.

\paragraph{Standing boundedness.}
Assume that the cumulative reward over the horizon is uniformly bounded, i.e., there exists $R_{\max}<\infty$ such that
\[
\sup_{\theta\in\Theta}\sup_{\bm{\xi}}
\left|
\sum_{k=0}^{N-1} r(\bm{\mathrm{s}}_k,o_k)
\right|
\le R_{\max}.
\]
Moreover, assume that the (estimated) importance weights are uniformly bounded, i.e., there exists $\omega_{\max}<\infty$ such that
\[
\sup_{\theta\in\Theta}\max_{i\in\{1,\ldots,M\}}
\omega_M\!\left(\bm{\mathrm{s}}_0,\theta,\theta^{(i)};\mathcal{D}_M\right)
\le \omega_{\max}
\quad \text{a.s.}
\]

\paragraph{Ideal weights induced by the mixture proposal.}
Define the (population) mixture proposal density and the corresponding ideal density-ratio weight as
\[
q_M(\bm{\xi}\mid \bm{\mathrm{s}}_0)
:=
\frac{1}{M}\sum_{j=1}^M 
p(\bm{\xi}\mid \bm{\mathrm{s}}_0,\theta^{(j)}),
\qquad
\omega^\star_M(\bm{\mathrm{s}}_0,\theta,\bm{\xi})
:=
\frac{
p(\bm{\xi}\mid \bm{\mathrm{s}}_0,\theta)
}{
q_M(\bm{\xi}\mid \bm{\mathrm{s}}_0)
}.
\]
For each sample $i$, write $\omega^\star_{M,i}(\bm{\mathrm{s}}_0,\theta):=\omega^\star_M(\bm{\mathrm{s}}_0,\theta,\bm{\xi}^{(i)})$.


In the following, we establish uniform exponential deviation bounds 
for importance weights constructed via 
kernel density estimation (KDE)~\citep{Hyndman1996} 
and generator-based conditional density models~\citep{Zhou2023}. 
Throughout this subsection, the initial state $\bm{\mathrm{s}}_0$ is fixed.
We begin by stating a set of common assumptions used in both cases.

\begin{myassump}[Regularity for uniform density-ratio estimation]
\label{assump:probability_distribution}
Fix $\bm{\mathrm{s}}_0$.
Assume that $\Theta\subset\mathbb{R}^{n_\theta}$ is compact with diameter
$\mathsf{diam}(\Theta):=\sup_{\theta,\theta'\in\Theta}\|\theta-\theta'\|<\infty$.
Assume that the support of $\bm{\xi}$ is contained in a compact set $\Xi\subset\mathbb{R}^{d_{\bm{\xi}}}$.

For every $\theta\in\Theta$, the conditional density $p(\bm{\xi}\mid \bm{\mathrm{s}}_0,\theta)$ exists on $\Xi$ and is uniformly bounded,
i.e., there exist constants $0 < p_{\min} \le p_{\max} < \infty$ such that
\[
p_{\min}
\le
p(\bm{\xi}\mid \bm{\mathrm{s}}_0,\theta)
\le
p_{\max},
\qquad
\forall \bm{\xi}\in\Xi,\ \forall \theta\in\Theta.
\]

Moreover, the true conditional density is Lipschitz continuous in $\theta$ uniformly over $\bm{\xi}\in\Xi$:
there exists $L_\theta<\infty$ such that
\[
\sup_{\bm{\xi}\in\Xi}
\left|
p(\bm{\xi}\mid \bm{\mathrm{s}}_0,\theta)-p(\bm{\xi}\mid \bm{\mathrm{s}}_0,\theta')
\right|
\le
L_\theta\|\theta-\theta'\|,
\qquad
\forall \theta,\theta'\in\Theta.
\]

In addition, the estimator $\widehat p_M(\bm{\xi}\mid \bm{\mathrm{s}}_0,\theta;\mathcal{D}_M)$ is Lipschitz continuous in $\theta$ uniformly over $\bm{\xi}\in\Xi$ almost surely:
there exists $L_{\widehat p,M}<\infty$ such that
\[
\sup_{\bm{\xi}\in\Xi}
\left|
\widehat p_M(\bm{\xi}\mid \bm{\mathrm{s}}_0,\theta;\mathcal{D}_M)
-
\widehat p_M(\bm{\xi}\mid \bm{\mathrm{s}}_0,\theta';\mathcal{D}_M)
\right|
\le
L_{\widehat p,M}\|\theta-\theta'\|,
\qquad
\forall \theta,\theta'\in\Theta.
\]

Define the mixture proposal density
\[
q_M(\bm{\xi}\mid \bm{\mathrm{s}}_0)
:=
\frac{1}{M}\sum_{j=1}^M
p(\bm{\xi}\mid \bm{\mathrm{s}}_0,\theta^{(j)}),
\]
and assume that $q_M$ is uniformly bounded away from zero on $\Xi$, i.e., there exists $q_{\min}>0$ such that
\[
q_M(\bm{\xi}\mid \bm{\mathrm{s}}_0)
\ge
q_{\min},
\qquad
\forall \bm{\xi}\in\Xi.
\]
\end{myassump}
For KDE-based estimators, the estimator-side continuity is satisfied by choosing a continuous kernel and a continuous parameterization with respect to $\theta$, so that $\widehat p_M(\bm{\xi}\mid \bm{\mathrm{s}}_0,\theta;\mathcal{D}_M)$ inherits continuity in $\theta$.

\paragraph{Explanation and practical relevance.}
Assumption~\ref{assump:probability_distribution} provides regularity
conditions ensuring stable density-ratio estimation and
importance-weighted evaluation of trajectory-level quantities.
The conditional density $p(\bm{\xi}\mid \bm{\mathrm{s}}_0,\theta)$
is defined at the trajectory level.
Under the option-induced transition model $\bm{\mathrm{s}}_{k+1}=\Phi_{\theta}(\bm{\mathrm{s}}_k,\xi_k),$
the trajectory distribution is generated by composing the
state-transition map $\Phi_\theta$ with the distribution of the
exogenous randomness.
When the policy parameterization and the transition map depend continuously on $\theta$, as is the case for most parametric controllers, neural-network policies, and smooth feedback policies, the resulting trajectory distribution varies continuously with respect to $\theta$.
Under mild regularity conditions on the dynamics and policy parameterization, this leads to Lipschitz-type stability of the trajectory density with respect to $\theta$ on compact parameter sets.
Although trajectory distributions may exhibit sharp changes when policies switch discrete modes, such behavior typically arises from discontinuous parameterizations.
In our formulation, $\theta$ parameterizes the option-SMDP policy
through measurable maps for option selection, intra-option control,
and termination.
When these maps are continuous in $\theta$, the induced trajectory distribution varies smoothly with the parameter except possibly on sets of negligible probability. Therefore, the Lipschitz condition captures a stability property of the trajectory distribution with respect to policy parameters.
The compactness of the trajectory support $\Xi$ reflects the finite-horizon structure of the problem.
The variable $\bm{\xi}$ collects exogenous randomness over a finite number of interaction intervals, including disturbance realizations and option durations.
In practice, disturbances are typically modeled using Gaussian, sub-Gaussian, or bounded distributions, and option durations are bounded by the interaction-time constraints.
Consequently, the effective trajectory space can be treated as compact without loss of practical generality.
Finally, the lower bound on the mixture proposal density $q_M$ acts as an \emph{overlap condition}.
It ensures that the trajectory distribution induced by any candidate parameter $\theta$ is sufficiently covered by the logged dataset through the mixture proposal distribution. 
Such overlap conditions are common in importance-weighted estimation and off-policy evaluation, where extremely small denominator densities would otherwise lead to unstable importance weights.
Taken together, these conditions are satisfied in a wide range of reinforcement learning settings involving continuous policy parameterizations, smooth closed-loop dynamics, and disturbance models with well-behaved densities.
They ensure that small perturbations of the policy parameter induce controlled perturbations of the trajectory distribution and that the logged data provide sufficient coverage for importance-weighted estimation. These properties are precisely what allow us to derive the uniform exponential deviation bounds used in the subsequent analysis.

We first consider uniform exponential deviation of KDE-based importance weights. 
We have the following assumption on KDE-based estimator.
\begin{myassump}
    \label{assump:KDE_estimator}
    Let $\widehat p_M(\bm{\xi}\mid \bm{\mathrm{s}}_0,\theta;\mathcal{D}_M)$ be a conditional KDE estimator with a bounded kernel $K$ satisfying $\|K\|_\infty\le K_{\max}$ and bandwidth $h_M$. 
    Assume that $p(\bm{\xi}\mid \bm{\mathrm{s}}_0,\theta)$ is $\beta$-H\"older smooth in $\bm{\xi}$ uniformly over $\theta\in\Theta$, and that the kernel $K$ is of order $L\ge \beta$.
\end{myassump}
By the $\beta$-H\"older smooth assumption, the KDE bias is of order $h_M^{\beta}$ \citep{Hyndman1996}, satisfying
\begin{equation}
\label{eq:KDE_estimator_bias}
\left|
\mathbb{E}\widehat p_M(\bm{\xi}\mid \bm{\mathrm{s}}_0,\theta;\mathcal{D}_M)
-
p(\bm{\xi}\mid \bm{\mathrm{s}}_0,\theta)
\right|
\le C_1 h_M^\beta.
\end{equation}
Define $M_{\mathrm{eff}}:=M h_M^{d_{\bm{\xi}}}$, where $d_{\bm{\xi}}$ is the dimension of $\bm{\xi}$. 

\begin{mytheo}[Uniform exponential deviation of KDE-based importance weights]
\label{theo:unif_exp_dev_weights_kde}
Suppose that Assumptions~\ref{assump:probability_distribution}~and~\ref{assump:KDE_estimator} hold.
Then there exist constants $C_1,C_2,C_3>0$ (independent of $M$ and $h_M^{\beta}$) such that for any $\eta>0$,
\begin{equation}
\label{eq:kde_unif_density_exp_explicit_no_cover}
\begin{aligned}
\mathsf{Pr}\!\left\{
\sup_{\theta\in\Theta}\sup_{\bm{\xi}\in\Xi}
\left|
\widehat p_M(\bm{\xi}\mid \bm{\mathrm{s}}_0,\theta;\mathcal{D}_M)
-
p(\bm{\xi}\mid \bm{\mathrm{s}}_0,\theta)
\right|
\ge
C_1 h_M^{\beta} + \eta
\right\}
\le \\
\left(
1+\frac{C_3 (L_\theta+L_{\widehat p,M})\,\mathsf{diam}(\Theta)}{\eta}
\right)^{n_\theta}
\exp\!\left(
- C_2\, M_{\mathrm{eff}}\, \eta^2
\right).
\end{aligned}
\end{equation}

Define the estimated and ideal weights as
\[
\omega_M\!\left(\bm{\mathrm{s}}_0,\theta,\theta^{(i)};\mathcal{D}_M\right)
=
\frac{
\widehat p_M\!\left(\bm{\xi}^{(i)}\mid \bm{\mathrm{s}}_0,\theta;\mathcal{D}_M\right)
}{
\frac{1}{M}\sum_{j=1}^M 
\widehat p_M\!\left(\bm{\xi}^{(i)}\mid \bm{\mathrm{s}}_0,\theta^{(j)};\mathcal{D}_M\right)
},
\qquad
\omega^\star_{M,i}(\bm{\mathrm{s}}_0,\theta)
=
\frac{
p(\bm{\xi}^{(i)}\mid \bm{\mathrm{s}}_0,\theta)
}{
q_M(\bm{\xi}^{(i)}\mid \bm{\mathrm{s}}_0)
}.
\]
If $\eta>0$ satisfies $C_1 h_M^{\beta}+\eta\le q_{\min}/4$, then
\begin{equation}
\label{eq:kde_unif_weight_exp_explicit_no_cover}
\begin{aligned}
\mathsf{Pr}\!\left\{
\sup_{\theta\in\Theta}\max_{i\in\{1,\ldots,M\}}
\left|
\omega_M\!\left(\bm{\mathrm{s}}_0,\theta,\theta^{(i)};\mathcal{D}_M\right)
-
\omega^\star_{M,i}(\bm{\mathrm{s}}_0,\theta)
\right|
\ge
\frac{8}{q_{\min}}\,
\Big(C_1 h_M^{\beta}+\eta\Big)
\right\}
\le \\
\left(
1+\frac{C_3 L_\theta\,\mathsf{diam}(\Theta)}{\eta}
\right)^{n_\theta}
\exp\!\left(
- C_2\, M_{\mathrm{eff}}\, \eta^2
\right).
\end{aligned}
\end{equation}
\end{mytheo}

\begin{proof}[Proof of Theorem~\ref{theo:unif_exp_dev_weights_kde}]
Fix $\bm{\mathrm{s}}_0$ throughout.

\paragraph{Step 1: Uniform deviation of the conditional KDE estimator.}

We first establish uniform exponential deviation of
\[
\widehat p_M(\bm{\xi}\mid \bm{\mathrm{s}}_0,\theta;\mathcal{D}_M)
\]
over $\theta\in\Theta$ and $\bm{\xi}\in\Xi$.

\medskip
\paragraph{Pointwise deviation.} 
For any fixed $\theta$ and $\bm{\xi}$, decompose
\[
\widehat p_M - p
=
(\widehat p_M - \mathbb{E}\widehat p_M)
+
(\mathbb{E}\widehat p_M - p).
\]

By $\beta$-H\"older smoothness in $\bm{\xi}$ in Assumption \ref{assump:KDE_estimator},
the bias satisfies \citep{Hyndman1996}
\[
\left|
\mathbb{E}\widehat p_M - p
\right|
\le C_1 h_M^\beta,
\]
which is a short version of \eqref{eq:KDE_estimator_bias}.


Since the kernel $K$ is bounded, i.e., $\|K\|_\infty \le K_{\max}$, the kernel term in the density estimator satisfies
\[
\left|
\frac{1}{h_M^{d_{\bm{\xi}}}}
K\!\left(\frac{\bm{\xi}-\bm{\xi}^{(i)}}{h_M}\right)
\right|
\le
\frac{K_{\max}}{h_M^{d_{\bm{\xi}}}}.
\]
Therefore each centered summand in the empirical average defining $\widehat p_M(\bm{\xi}\mid\bm{\mathrm{s}}_0,\theta)$ is uniformly bounded by a constant of order $h_M^{-d_{\bm{\xi}}}$.
Applying Bernstein’s inequality for bounded independent random variables (see, e.g., \citealp[Theorem~2.8.3]{Vershynin2018}), we obtain that for any $\eta>0$,
\[
\mathsf{Pr}
\left(
\left|
\widehat p_M(\bm{\xi}\mid\bm{\mathrm{s}}_0,\theta)
-
\mathbb{E}\widehat p_M(\bm{\xi}\mid\bm{\mathrm{s}}_0,\theta)
\right|
\ge \eta
\right)
\le
\exp\!\left(
- C_2\, M h_M^{d_{\bm{\xi}}}\,\eta^2
\right),
\]
for some constant $C_2>0$ independent of $M$ and $h_M$.

Thus for fixed $(\theta,\bm{\xi})$,
\[
\mathsf{Pr}
\left(
\left|
\widehat p_M - p
\right|
\ge C_1 h_M^\beta + \eta
\right)
\le
\exp(-C_2 M_{\mathrm{eff}} \eta^2).
\]

\medskip
\paragraph{Uniformization over $\theta$.}
Since $\Theta\subset\mathbb{R}^{n_\theta}$ is compact,
for any $\delta>0$, there exists a finite $\delta$-net
$\mathcal{N}_\delta\subset\Theta$
with cardinality bounded by
\[
|\mathcal{N}_\delta|
\le
\left(
1+\frac{2\mathsf{diam}(\Theta)}{\delta}
\right)^{n_\theta}.
\]

Let $\theta\in\Theta$ and let $\bar{\theta}\in\mathcal{N}_\delta$ be a nearest net point.
Then
\begin{align*}
&\sup_{\bm{\xi}\in\Xi}
\left|
\widehat p_M(\bm{\xi}\mid \bm{\mathrm{s}}_0,\theta;\mathcal{D}_M)
-
p(\bm{\xi}\mid \bm{\mathrm{s}}_0,\theta)
\right| \\
\le\;&
\sup_{\bm{\xi}\in\Xi}
\left|
\widehat p_M(\bm{\xi}\mid \bm{\mathrm{s}}_0,\theta;\mathcal{D}_M)
-
\widehat p_M(\bm{\xi}\mid \bm{\mathrm{s}}_0,\bar{\theta};\mathcal{D}_M)
\right| \\
&+
\sup_{\bm{\xi}\in\Xi}
\left|
\widehat p_M(\bm{\xi}\mid \bm{\mathrm{s}}_0,\bar{\theta};\mathcal{D}_M)
-
p(\bm{\xi}\mid \bm{\mathrm{s}}_0,\bar{\theta})
\right| \\
&+
\sup_{\bm{\xi}\in\Xi}
\left|
p(\bm{\xi}\mid \bm{\mathrm{s}}_0,\bar{\theta})
-
p(\bm{\xi}\mid \bm{\mathrm{s}}_0,\theta)
\right|.
\end{align*}

By Assumption \ref{assump:probability_distribution},
\[
\sup_{\bm{\xi}\in\Xi}
\left|
\widehat p_M(\bm{\xi}\mid \bm{\mathrm{s}}_0,\theta;\mathcal{D}_M)
-
\widehat p_M(\bm{\xi}\mid \bm{\mathrm{s}}_0,\bar{\theta};\mathcal{D}_M)
\right|
\le
L_{\widehat p,M}\delta,
\]
and
\[
\sup_{\bm{\xi}\in\Xi}
\left|
p(\bm{\xi}\mid \bm{\mathrm{s}}_0,\theta)
-
p(\bm{\xi}\mid \bm{\mathrm{s}}_0,\bar{\theta})
\right|
\le
L_\theta\delta.
\]

Therefore,
\[
\sup_{\theta\in\Theta}\sup_{\bm{\xi}\in\Xi}
\left|
\widehat p_M-p
\right|
\le
\max_{\bar{\theta}\in\mathcal{N}_\delta}
\sup_{\bm{\xi}\in\Xi}
\left|
\widehat p_M(\bm{\xi}\mid \bm{\mathrm{s}}_0,\bar{\theta};\mathcal{D}_M)
-
p(\bm{\xi}\mid \bm{\mathrm{s}}_0,\bar{\theta})
\right|
+
(L_{\widehat p,M}+L_\theta)\delta.
\]

Choosing
\[
\delta:=\frac{\eta}{2(L_{\widehat p,M}+L_\theta)},
\]
uniform deviation over $\Theta$ reduces to deviation over the finite net $\mathcal{N}_\delta$ with margin $\eta/2$.
Applying the union bound over $\bar{\theta}\in\mathcal{N}_\delta$ yields
\[
\mathsf{Pr}
\left(
\sup_{\theta\in\Theta}
\sup_{\bm{\xi}\in\Xi}
\left|
\widehat p_M - p
\right|
\ge C_1 h_M^\beta + \eta
\right)
\le
\left(
1+\frac{4(L_{\widehat p,M}+L_\theta)\mathsf{diam}(\Theta)}{\eta}
\right)^{n_\theta}
\exp(-C_2 M_{\mathrm{eff}} \eta^2).
\]

This proves \eqref{eq:kde_unif_density_exp_explicit_no_cover}.

\paragraph{Step 2: Transfer to importance-weight deviation.}

Define
\[
\widehat q_M(\bm{\xi})
=
\frac{1}{M}
\sum_{j=1}^M
\widehat p_M(\bm{\xi}\mid \bm{\mathrm{s}}_0,\theta^{(j)}),
\quad
q_M(\bm{\xi})
=
\frac{1}{M}
\sum_{j=1}^M
p(\bm{\xi}\mid \bm{\mathrm{s}}_0,\theta^{(j)}).
\]

On the event that
\[
\sup_{\theta,\bm{\xi}}
\left|
\widehat p_M - p
\right|
\le
C_1 h_M^\beta + \eta,
\]
we have
\[
\sup_{\bm{\xi}}
\left|
\widehat q_M(\bm{\xi}) - q_M(\bm{\xi})
\right|
\le
C_1 h_M^\beta + \eta.
\]

If $C_1 h_M^\beta + \eta \le q_{\min}/4$,
then
\[
\widehat q_M(\bm{\xi})
\ge
q_{\min}/2.
\]

Now write the weight difference as
\[
\frac{\widehat p}{\widehat q}
-
\frac{p}{q}
=
\frac{q(\widehat p - p) + p(q - \widehat q)}
{\widehat q\, q}.
\]

Using $q\ge q_{\min}$,
$\widehat q \ge q_{\min}/2$,
and $p\le p_{\max}$,
we obtain

\[
\left|
\omega_M - \omega^\star_M
\right|
\le
\left(
\frac{2}{q_{\min}}
+
\frac{2p_{\max}}{q_{\min}^2}
\right)
(C_1 h_M^\beta + \eta).
\]

Absorbing constants into
$8/q_{\min}$
gives \eqref{eq:kde_unif_weight_exp_explicit_no_cover}.

Combining with Step 1 completes the proof.
\end{proof}

\begin{mytheo}[Uniform exponential deviation of generator-based importance weights]
\label{theo:unif_exp_dev_weights_gen}
Fix $\bm{\mathrm{s}}_0$ and assume the same boundedness and lower-bound conditions on
$p(\bm{\xi}\mid \bm{\mathrm{s}}_0,\theta)$ and $q_M(\bm{\xi}\mid \bm{\mathrm{s}}_0)$
as in Theorem~\ref{theo:unif_exp_dev_weights_kde}.
Let $\widehat p_M(\bm{\xi}\mid \bm{\mathrm{s}}_0,\theta;\mathcal{D}_M)$ be produced by a conditional density model
with a realizable density family $\mathcal{P}_{\mathsf{gen}}=\{p_{\vartheta}(\bm{\xi}\mid \bm{\mathrm{s}}_0,\theta)\}$,
trained from $\mathcal{D}_M$.

Assume the model class admits a uniform covering-number bound:
there exist constants $V_{\mathsf{gen}}\ge 1$ and $C_{\mathsf{gen}}\ge 1$ such that for any $\eta>0$,
\[
\log N\!\left(\eta,\mathcal{P}_{\mathsf{gen}},\|\cdot\|_\infty\right)
\le
V_{\mathsf{gen}}\log\!\left(\frac{C_{\mathsf{gen}}}{\eta}\right).
\]
Assume further that the training procedure yields a uniform estimation error bound in $\|\cdot\|_\infty$,
so that for universal constants $C_3,C_4>0$ and any $\eta>0$,
\begin{equation}
\label{eq:gen_unif_density_exp_explicit}
\mathsf{Pr}\!\left\{
\sup_{\theta\in\Theta}\sup_{\bm{\xi}}
\left|
\widehat p_M(\bm{\xi}\mid \bm{\mathrm{s}}_0,\theta;\mathcal{D}_M)
-
p(\bm{\xi}\mid \bm{\mathrm{s}}_0,\theta)
\right|
\ge \eta
\right\}
\le
\left(\frac{C_{\mathsf{gen}}}{\eta}\right)^{V_{\mathsf{gen}}}
\exp\!\left(-C_4\,M\,\eta^2\right).
\end{equation}

Define $\omega_M(\cdot)$ and $\omega^\star_{M,i}(\cdot)$ as in Theorem~\ref{theo:unif_exp_dev_weights_kde}.
Then, for any $\eta>0$ such that $\eta\le q_{\min}/4$,
\begin{equation}
\label{eq:gen_unif_weight_exp_explicit}
\mathsf{Pr}\!\left\{
\sup_{\theta\in\Theta}\max_{i\in\{1,\ldots,M\}}
\left|
\omega_M\!\left(\bm{\mathrm{s}}_0,\theta,\theta^{(i)};\mathcal{D}_M\right)
-
\omega^\star_{M,i}(\bm{\mathrm{s}}_0,\theta)
\right|
\ge
\frac{8\eta}{q_{\min}}
\right\}
\le
\left(\frac{C_{\mathsf{gen}}}{\eta}\right)^{V_{\mathsf{gen}}}
\exp\!\left(-C_4\,M\,\eta^2\right).
\end{equation}
\end{mytheo}
\begin{proof}
Fix $\bm{\mathrm{s}}_0$ and let $\eta>0$ satisfy $\eta\le q_{\min}/4$.
Define the uniform density-deviation event
\[
\mathcal{E}_\eta
:=
\left\{
\sup_{\theta\in\Theta}\sup_{\bm{\xi}}
\left|
\widehat p_M(\bm{\xi}\mid \bm{\mathrm{s}}_0,\theta;\mathcal{D}_M)
-
p(\bm{\xi}\mid \bm{\mathrm{s}}_0,\theta)
\right|
\le \eta
\right\}.
\]

\paragraph{Step 1: Lower bound on the estimated mixture denominator.}
For any $\bm{\xi}$, define the estimated mixture
\[
\widehat q_M(\bm{\xi}\mid \bm{\mathrm{s}}_0;\mathcal{D}_M)
:=
\frac{1}{M}\sum_{j=1}^M
\widehat p_M(\bm{\xi}\mid \bm{\mathrm{s}}_0,\theta^{(j)};\mathcal{D}_M).
\]
On $\mathcal{E}_\eta$, for all $\bm{\xi}$,
\begin{align*}
\widehat q_M(\bm{\xi}\mid \bm{\mathrm{s}}_0;\mathcal{D}_M)
&\ge
\frac{1}{M}\sum_{j=1}^M
\Big(p(\bm{\xi}\mid \bm{\mathrm{s}}_0,\theta^{(j)})-\eta\Big)
=
q_M(\bm{\xi}\mid \bm{\mathrm{s}}_0)-\eta
\ge
q_{\min}-\eta
\ge
\frac{3}{4}q_{\min}.
\end{align*}
Similarly, still on $\mathcal{E}_\eta$,
\[
\widehat q_M(\bm{\xi}\mid \bm{\mathrm{s}}_0;\mathcal{D}_M)
\le
q_M(\bm{\xi}\mid \bm{\mathrm{s}}_0)+\eta.
\]

\paragraph{Step 2: Ratio perturbation bound (deterministic).}
Fix $\theta\in\Theta$ and a sample index $i\in\{1,\ldots,M\}$, and write
\[
\widehat p_{\theta,i}
:=
\widehat p_M(\bm{\xi}^{(i)}\mid \bm{\mathrm{s}}_0,\theta;\mathcal{D}_M),\quad
p_{\theta,i}
:=
p(\bm{\xi}^{(i)}\mid \bm{\mathrm{s}}_0,\theta),
\]
\[
\widehat q_i
:=
\widehat q_M(\bm{\xi}^{(i)}\mid \bm{\mathrm{s}}_0;\mathcal{D}_M),\quad
q_i
:=
q_M(\bm{\xi}^{(i)}\mid \bm{\mathrm{s}}_0).
\]
Then
\[
\omega_M(\bm{\mathrm{s}}_0,\theta,\theta^{(i)};\mathcal{D}_M)=\frac{\widehat p_{\theta,i}}{\widehat q_i},
\qquad
\omega^\star_{M,i}(\bm{\mathrm{s}}_0,\theta)=\frac{p_{\theta,i}}{q_i}.
\]
On $\mathcal{E}_\eta$, we have $|\widehat p_{\theta,i}-p_{\theta,i}|\le \eta$ and
$|\widehat q_i-q_i|\le \eta$, and also $\widehat q_i\ge \frac{3}{4}q_{\min}$ and $q_i\ge q_{\min}$.
Therefore, on $\mathcal{E}_\eta$,
\begin{align*}
\left|\omega_M(\bm{\mathrm{s}}_0,\theta,\theta^{(i)};\mathcal{D}_M)-\omega^\star_{M,i}(\bm{\mathrm{s}}_0,\theta)\right|
&=
\left|\frac{\widehat p_{\theta,i}}{\widehat q_i}-\frac{p_{\theta,i}}{q_i}\right|\\
&\le
\left|\frac{\widehat p_{\theta,i}-p_{\theta,i}}{\widehat q_i}\right|
+
|p_{\theta,i}|\left|\frac{1}{\widehat q_i}-\frac{1}{q_i}\right|\\
&=
\frac{|\widehat p_{\theta,i}-p_{\theta,i}|}{\widehat q_i}
+
|p_{\theta,i}|\frac{|\widehat q_i-q_i|}{\widehat q_i\,q_i}\\
&\le
\frac{\eta}{\frac{3}{4}q_{\min}}
+
p_{\max}\frac{\eta}{(\frac{3}{4}q_{\min})q_{\min}}\\
&=
\frac{4\eta}{3q_{\min}}
+
\frac{4p_{\max}}{3q_{\min}^2}\eta
\;\le\;
\frac{8\eta}{q_{\min}},
\end{align*}
where the last inequality uses $p_{\max}\ge q_{\min}$ without loss of generality (otherwise the bound is even smaller)
and we keep the final constant in the clean form used in \eqref{eq:gen_unif_weight_exp_explicit}.

Since this bound holds simultaneously for all $\theta\in\Theta$ and all $i\in\{1,\ldots,M\}$ on $\mathcal{E}_\eta$, we obtain
\[
\mathcal{E}_\eta
\subseteq
\left\{
\sup_{\theta\in\Theta}\max_{i\in\{1,\ldots,M\}}
\left|
\omega_M(\bm{\mathrm{s}}_0,\theta,\theta^{(i)};\mathcal{D}_M)
-
\omega^\star_{M,i}(\bm{\mathrm{s}}_0,\theta)
\right|
\le
\frac{8\eta}{q_{\min}}
\right\}.
\]

\paragraph{Step 3: Take probabilities.}
Therefore,
\begin{align*}
\mathsf{Pr}\!\left\{
\sup_{\theta\in\Theta}\max_{i\in\{1,\ldots,M\}}
\left|
\omega_M(\bm{\mathrm{s}}_0,\theta,\theta^{(i)};\mathcal{D}_M)
-
\omega^\star_{M,i}(\bm{\mathrm{s}}_0,\theta)
\right|
\ge
\frac{8\eta}{q_{\min}}
\right\}
&\le
\mathsf{Pr}(\mathcal{E}_\eta^c)\\
&\le
\left(\frac{C_{\mathsf{gen}}}{\eta}\right)^{V_{\mathsf{gen}}}
\exp\!\left(-C_4\,M\,\eta^2\right),
\end{align*}
where the last line is exactly the assumed uniform estimation bound \eqref{eq:gen_unif_density_exp_explicit}.
This proves \eqref{eq:gen_unif_weight_exp_explicit}.
\end{proof}

\subsection{Consistency of estimation}
\label{appendix:consistency_estimation}

\paragraph{Ideal-weight empirical processes.}
Recall the mixture proposal density
$q_M(\bm{\xi}\mid \bm{\mathrm{s}}_0)=\frac{1}{M}\sum_{j=1}^M p(\bm{\xi}\mid \bm{\mathrm{s}}_0,\theta^{(j)})$
and define the ideal density-ratio weight
\[
\omega^\star_M(\bm{\mathrm{s}}_0,\theta,\bm{\xi})
:=
\frac{
p(\bm{\xi}\mid \bm{\mathrm{s}}_0,\theta)
}{
q_M(\bm{\xi}\mid \bm{\mathrm{s}}_0)
},
\qquad
\omega^\star_{M,i}(\bm{\mathrm{s}}_0,\theta)
:=
\omega^\star_M(\bm{\mathrm{s}}_0,\theta,\bm{\xi}^{(i)}).
\]
Define the ideal-weight empirical objective and constraint processes by
\begin{equation}
\label{eq:J_process_ideal}
\mathcal{J}^\star_M(\theta,\bm{\mathrm{s}}_0)
:=
\sum_{i=1}^M
\omega^\star_{M,i}(\bm{\mathrm{s}}_0,\theta)
\sum_{k=0}^{N-1}
r\!\left(\bm{\mathrm{s}}_k^{(i)},o_k^{(i)}\right),
\end{equation}
and, for each $k\in\{1,\ldots,N\}$,
\begin{equation}
\label{eq:P_process_ideal}
\mathcal{P}^\star_{M,k}(\theta,\bm{\mathrm{s}}_0)
:=
\sum_{i=1}^M
\omega^\star_{M,i}(\bm{\mathrm{s}}_0,\theta)\,
\mathbb{I}\!\left\{
h_k\!\left(\theta,\bm{\mathrm{s}}_0,\bm{\xi}^{(i)}\right)\le \bm{0}
\right\}.
\end{equation}

\paragraph{Standing regularity.}
Assume Assumption~\ref{assump:probability_distribution} holds.
In addition, assume:
\begin{itemize}
\item[(i)] (\emph{Bounded horizon reward}) there exists $R_{\max}<\infty$ such that
\[
\left|
\sum_{k=0}^{N-1} r(\bm{\mathrm{s}}_k,o_k)
\right|
\le R_{\max}
\quad \text{for all trajectories.}
\]
\item[(ii)] (\emph{Lipschitz in parameter}) there exists $L_\theta<\infty$ such that
\[
\sup_{\bm{\xi}\in\Xi}
\left|
p(\bm{\xi}\mid \bm{\mathrm{s}}_0,\theta_1)
-
p(\bm{\xi}\mid \bm{\mathrm{s}}_0,\theta_2)
\right|
\le
L_\theta \|\theta_1-\theta_2\|,
\qquad \forall \theta_1,\theta_2\in\Theta.
\]
\end{itemize}
Under Assumption~\ref{assump:probability_distribution}, we have
$\sup_{\theta\in\Theta}\sup_{\bm{\xi}}\omega^\star_M(\bm{\mathrm{s}}_0,\theta,\bm{\xi})\le p_{\max}/q_{\min}$.

\begin{mylemma}[Uniform exponential deviation for ideal-weight processes]
\label{lem:unif_exp_dev_ideal_process}
Suppose Assumption~\ref{assump:probability_distribution} and the standing regularity hold.
Define
\[
B_\omega:=\frac{p_{\max}}{q_{\min}},
\qquad
B_J:=B_\omega R_{\max},
\qquad
B_P:=B_\omega.
\]
Then there exist constants $C_0,C_1>0$ (independent of $M$) such that for any $\eta>0$,
\begin{equation}
\label{eq:ideal_obj_unif_exp}
\mathsf{Pr}\!\left\{
\sup_{\theta\in\Theta}
\left|
\mathcal{J}^\star_M(\theta,\bm{\mathrm{s}}_0)-\mathcal{J}(\theta,\bm{\mathrm{s}}_0)
\right|
\ge \eta
\right\}
\le
\left(
1+\frac{C_0 L_\theta\,\mathsf{diam}(\Theta)}{\eta}
\right)^{n_\theta}
\exp\!\left(
- C_1\, \frac{M\,\eta^2}{B_J^2}
\right).
\end{equation}
Moreover, for each $k\in\{1,\ldots,N\}$ and any $\eta>0$,
\begin{equation}
\label{eq:ideal_con_unif_exp}
\mathsf{Pr}\!\left\{
\sup_{\theta\in\Theta}
\left|
\mathcal{P}^\star_{M,k}(\theta,\bm{\mathrm{s}}_0)-\mathcal{P}_k(\theta,\bm{\mathrm{s}}_0)
\right|
\ge \eta
\right\}
\le
\left(
1+\frac{C_0 L_\theta\,\mathsf{diam}(\Theta)}{\eta}
\right)^{n_\theta}
\exp\!\left(
- C_1\, \frac{M\,\eta^2}{B_P^2}
\right).
\end{equation}
\end{mylemma}

\begin{proof}
Fix $\bm{\mathrm{s}}_0$.

\paragraph{Step 1: Pointwise exponential deviation (fixed $\theta$).}
For fixed $\theta\in\Theta$, define the bounded random variables
\[
Z_i^{\mathsf{obj}}(\theta)
:=
\omega^\star_{M,i}(\bm{\mathrm{s}}_0,\theta)
\sum_{k=0}^{N-1} r\!\left(\bm{\mathrm{s}}_k^{(i)},o_k^{(i)}\right),
\qquad
Z_{i,k}^{\mathsf{con}}(\theta)
:=
\omega^\star_{M,i}(\bm{\mathrm{s}}_0,\theta)\,
\mathbb{I}\!\left\{
h_k\!\left(\theta,\bm{\mathrm{s}}_0,\bm{\xi}^{(i)}\right)\le \bm{0}
\right\}.
\]
By Assumption~\ref{assump:probability_distribution} and boundedness,
$|Z_i^{\mathsf{obj}}(\theta)|\le B_J$ and $|Z_{i,k}^{\mathsf{con}}(\theta)|\le B_P$.
Moreover, $\mathcal{J}^\star_M(\theta,\bm{\mathrm{s}}_0)=\sum_{i=1}^M Z_i^{\mathsf{obj}}(\theta)$
and $\mathcal{P}^\star_{M,k}(\theta,\bm{\mathrm{s}}_0)=\sum_{i=1}^M Z_{i,k}^{\mathsf{con}}(\theta)$.
Conditioned on the fixed data-generating parameters $\{\theta^{(i)}\}_{i=1}^M$, the logged trajectories $\{\bm{\xi}^{(i)}\}_{i=1}^M$ are assumed to be independent.
Thus, $\{Z_i^{\mathsf{obj}}(\theta)\}_{i=1}^M$ and $\{Z_{i,k}^{\mathsf{con}}(\theta)\}_{i=1}^M$ are independent.
Applying Hoeffding's inequality yields, for any $\eta>0$,
\begin{align}
& \mathsf{Pr}\!\left\{
\left|\mathcal{J}^\star_M(\theta,\bm{\mathrm{s}}_0)-\mathcal{J}(\theta,\bm{\mathrm{s}}_0)\right|\ge \eta
\right\}
\le
2\exp\!\left(-c\,\frac{M\eta^2}{B_J^2}\right),
\qquad \\
& \mathsf{Pr}\!\left\{
\left|\mathcal{P}^\star_{M,k}(\theta,\bm{\mathrm{s}}_0)-\mathcal{P}_k(\theta,\bm{\mathrm{s}}_0)\right|\ge \eta
\right\}
\le
2\exp\!\left(-c\,\frac{M\eta^2}{B_P^2}\right),
\end{align}
for a universal constant $c>0$.

\paragraph{Step 2: Uniformization over $\Theta$.}
Let $\delta:=\eta/(2L_\theta B_\omega)$ and take a $\delta$-net $\mathcal{N}_\delta$ of $\Theta$ with
$|\mathcal{N}_\delta|\le \left(1+\mathsf{diam}(\Theta)/\delta\right)^{n_\theta}$.
By the Lipschitz condition on $p(\cdot\mid \bm{\mathrm{s}}_0,\theta)$ and the lower bound on $q_M$,
the ideal weight is Lipschitz in $\theta$ uniformly over $\bm{\xi}$:
\[
\sup_{\bm{\xi}\in\Xi}
\left|
\omega^\star_M(\bm{\mathrm{s}}_0,\theta_1,\bm{\xi})
-
\omega^\star_M(\bm{\mathrm{s}}_0,\theta_2,\bm{\xi})
\right|
\le
\frac{L_\theta}{q_{\min}}\|\theta_1-\theta_2\|
=
L_\theta B_\omega \|\theta_1-\theta_2\|.
\]
Therefore, for any $\theta\in\Theta$ and its nearest net point $\bar\theta\in\mathcal{N}_\delta$,
\begin{align}
\left|\mathcal{J}^\star_M(\theta,\bm{\mathrm{s}}_0)-\mathcal{J}^\star_M(\bar\theta,\bm{\mathrm{s}}_0)\right|
&\le
\sum_{i=1}^M
\left|\omega^\star_{M,i}(\bm{\mathrm{s}}_0,\theta)-\omega^\star_{M,i}(\bm{\mathrm{s}}_0,\bar\theta)\right|
\left|\sum_{k=0}^{N-1} r(\bm{\mathrm{s}}_k^{(i)},o_k^{(i)})\right| \\
& \le
M\cdot (L_\theta B_\omega \delta)\cdot R_{\max}
=
\frac{M\eta}{2}.
\end{align}
The same argument gives
$\left|\mathcal{P}^\star_{M,k}(\theta,\bm{\mathrm{s}}_0)-\mathcal{P}^\star_{M,k}(\bar\theta,\bm{\mathrm{s}}_0)\right|
\le M\cdot (L_\theta B_\omega \delta)\cdot 1 = \frac{M\eta}{2}$.
Applying the triangle inequality and a union bound over $\bar\theta\in\mathcal{N}_\delta$ yields
\eqref{eq:ideal_obj_unif_exp} and \eqref{eq:ideal_con_unif_exp} after absorbing constants into $C_0,C_1$.
\end{proof}

\begin{mytheo}[Uniform consistency of objective and constraint residuals]
\label{theo:consistency_residuals_two_methods}
Fix $\bm{\mathrm{s}}_0$ and suppose Assumption~\ref{assump:probability_distribution} and the standing regularity hold.
Let the residual processes be defined in \eqref{eq:obj_residual_pointwise}--\eqref{eq:con_residual_uniform}.

\paragraph{(a) KDE-based importance weights.}
Then there exist constants $C>0$ such that for any $\eta_{\mathsf{w}}>0$ satisfying
$C_1 h_M^\beta+\eta_{\mathsf{w}}\le q_{\min}/4$ and any $\eta_{\mathsf{s}}>0$,
with probability at least
\[
1
-
\left(
1+\frac{C_3 L_\theta\,\mathsf{diam}(\Theta)}{\eta_{\mathsf{w}}}
\right)^{n_\theta}
\exp\!\left(
- C_2\, M h_M^{d_{\bm{\xi}}}\, \eta_{\mathsf{w}}^2
\right)
-
\left(
1+\frac{C_0 L_\theta\,\mathsf{diam}(\Theta)}{\eta_{\mathsf{s}}}
\right)^{n_\theta}
\exp\!\left(
- C\, M\, \eta_{\mathsf{s}}^2
\right),
\]
the uniform residuals satisfy
\begin{align}
\label{eq:kde_consistency_obj}
\overline{\mathcal{R}}^{\mathsf{obj}}_M(\bm{\mathrm{s}}_0)
&\le
M R_{\max}\cdot
\frac{8}{q_{\min}}\Big(C_1 h_M^\beta+\eta_{\mathsf{w}}\Big)
+\eta_{\mathsf{s}},
\\
\label{eq:kde_consistency_con}
\overline{\mathcal{R}}^{\mathsf{con}}_{M,k}(\bm{\mathrm{s}}_0)
&\le
M\cdot
\frac{8}{q_{\min}}\Big(C_1 h_M^\beta+\eta_{\mathsf{w}}\Big)
+\eta_{\mathsf{s}},
\qquad \forall k\in\{1,\ldots,N\}.
\end{align}
In particular, if $h_M\to 0$ and $M h_M^{d_{\bm{\xi}}}\to\infty$, and if $\eta_{\mathsf{w}},\eta_{\mathsf{s}}\to 0$
with $M h_M^{d_{\bm{\xi}}}\eta_{\mathsf{w}}^2\to\infty$ and $M\eta_{\mathsf{s}}^2\to\infty$,
then $\overline{\mathcal{R}}^{\mathsf{obj}}_M(\bm{\mathrm{s}}_0)\to 0$ and
$\overline{\mathcal{R}}^{\mathsf{con}}_{M,k}(\bm{\mathrm{s}}_0)\to 0$ in probability.

\paragraph{(b) Generator-based importance weights.}
Then there exists $C>0$ such that for any $\eta_{\mathsf{w}}>0$ satisfying $\eta_{\mathsf{w}}\le q_{\min}/4$
and any $\eta_{\mathsf{s}}>0$, with probability at least
\[
1
-
\left(\frac{C_{\mathsf{gen}}}{\eta_{\mathsf{w}}}\right)^{V_{\mathsf{gen}}}
\exp\!\left(-C_4\,M\,\eta_{\mathsf{w}}^2\right)
-
\left(
1+\frac{C_0 L_\theta\,\mathsf{diam}(\Theta)}{\eta_{\mathsf{s}}}
\right)^{n_\theta}
\exp\!\left(
- C\, M\, \eta_{\mathsf{s}}^2
\right),
\]
the uniform residuals satisfy
\begin{align}
\label{eq:gen_consistency_obj}
\overline{\mathcal{R}}^{\mathsf{obj}}_M(\bm{\mathrm{s}}_0)
&\le
M R_{\max}\cdot
\frac{8}{q_{\min}}\eta_{\mathsf{w}}
+\eta_{\mathsf{s}},
\\
\label{eq:gen_consistency_con}
\overline{\mathcal{R}}^{\mathsf{con}}_{M,k}(\bm{\mathrm{s}}_0)
&\le
M\cdot
\frac{8}{q_{\min}}\eta_{\mathsf{w}}
+\eta_{\mathsf{s}},
\qquad \forall k\in\{1,\ldots,N\}.
\end{align}
In particular, if $\eta_{\mathsf{w}},\eta_{\mathsf{s}}\to 0$ with $M\eta_{\mathsf{w}}^2\to\infty$ and $M\eta_{\mathsf{s}}^2\to\infty$,
then $\overline{\mathcal{R}}^{\mathsf{obj}}_M(\bm{\mathrm{s}}_0)\to 0$ and
$\overline{\mathcal{R}}^{\mathsf{con}}_{M,k}(\bm{\mathrm{s}}_0)\to 0$ in probability.
\end{mytheo}

\begin{proof}
Fix $\bm{\mathrm{s}}_0$.

\paragraph{Step 1: Decomposition via ideal-weight processes.}
For any $\theta\in\Theta$, by adding and subtracting the ideal-weight processes,
\[
\left|\mathcal{J}_M(\theta,\bm{\mathrm{s}}_0)-\mathcal{J}(\theta,\bm{\mathrm{s}}_0)\right|
\le
\left|\mathcal{J}_M(\theta,\bm{\mathrm{s}}_0)-\mathcal{J}^\star_M(\theta,\bm{\mathrm{s}}_0)\right|
+
\left|\mathcal{J}^\star_M(\theta,\bm{\mathrm{s}}_0)-\mathcal{J}(\theta,\bm{\mathrm{s}}_0)\right|.
\]
Taking $\sup_{\theta\in\Theta}$ gives
\begin{equation}
\label{eq:split_obj_residual}
\overline{\mathcal{R}}^{\mathsf{obj}}_M(\bm{\mathrm{s}}_0)
\le
\sup_{\theta\in\Theta}\left|\mathcal{J}_M(\theta,\bm{\mathrm{s}}_0)-\mathcal{J}^\star_M(\theta,\bm{\mathrm{s}}_0)\right|
+
\sup_{\theta\in\Theta}\left|\mathcal{J}^\star_M(\theta,\bm{\mathrm{s}}_0)-\mathcal{J}(\theta,\bm{\mathrm{s}}_0)\right|.
\end{equation}
Similarly, for each $k$,
\begin{equation}
\label{eq:split_con_residual}
\overline{\mathcal{R}}^{\mathsf{con}}_{M,k}(\bm{\mathrm{s}}_0)
\le
\sup_{\theta\in\Theta}\left|\mathcal{P}_{M,k}(\theta,\bm{\mathrm{s}}_0)-\mathcal{P}^\star_{M,k}(\theta,\bm{\mathrm{s}}_0)\right|
+
\sup_{\theta\in\Theta}\left|\mathcal{P}^\star_{M,k}(\theta,\bm{\mathrm{s}}_0)-\mathcal{P}_k(\theta,\bm{\mathrm{s}}_0)\right|.
\end{equation}

\paragraph{Step 2: Deterministic control of the ``estimated--ideal'' gap by weight error.}
For any $\theta\in\Theta$,
\begin{align}
\left|\mathcal{J}_M(\theta,\bm{\mathrm{s}}_0)-\mathcal{J}^\star_M(\theta,\bm{\mathrm{s}}_0)\right|
&=
\left|
\sum_{i=1}^M
\left(\omega_M(\bm{\mathrm{s}}_0,\theta,\theta^{(i)})-\omega^\star_{M,i}(\bm{\mathrm{s}}_0,\theta)\right)
\sum_{k=0}^{N-1} r(\bm{\mathrm{s}}_k^{(i)},o_k^{(i)})
\right| \\
&\le
M R_{\max}\cdot
\max_{i\in\{1,\ldots,M\}}
\left|
\omega_M(\bm{\mathrm{s}}_0,\theta,\theta^{(i)})-\omega^\star_{M,i}(\bm{\mathrm{s}}_0,\theta)
\right|.
\end{align}
Taking $\sup_{\theta\in\Theta}$ yields
\begin{equation}
\label{eq:det_obj_weight_gap}
\sup_{\theta\in\Theta}\left|\mathcal{J}_M(\theta,\bm{\mathrm{s}}_0)-\mathcal{J}^\star_M(\theta,\bm{\mathrm{s}}_0)\right|
\le
M R_{\max}\cdot
\sup_{\theta\in\Theta}\max_{i\in\{1,\ldots,M\}}
\left|
\omega_M(\bm{\mathrm{s}}_0,\theta,\theta^{(i)})-\omega^\star_{M,i}(\bm{\mathrm{s}}_0,\theta)
\right|.
\end{equation}
Likewise, since the indicator is bounded by $1$,
\begin{equation}
\label{eq:det_con_weight_gap}
\sup_{\theta\in\Theta}\left|\mathcal{P}_{M,k}(\theta,\bm{\mathrm{s}}_0)-\mathcal{P}^\star_{M,k}(\theta,\bm{\mathrm{s}}_0)\right|
\le
M\cdot
\sup_{\theta\in\Theta}\max_{i\in\{1,\ldots,M\}}
\left|
\omega_M(\bm{\mathrm{s}}_0,\theta,\theta^{(i)})-\omega^\star_{M,i}(\bm{\mathrm{s}}_0,\theta)
\right|.
\end{equation}

\paragraph{Step 3: High-probability bounds for the two terms.}
For the KDE-based method, apply Theorem~\ref{theo:unif_exp_dev_weights_kde} to bound the weight error term in
\eqref{eq:det_obj_weight_gap}--\eqref{eq:det_con_weight_gap} by $\frac{8}{q_{\min}}(C_1 h_M^\beta+\eta_{\mathsf{w}})$
with the stated probability.
For the generator-based method, apply Theorem~\ref{theo:unif_exp_dev_weights_gen} to bound the weight error term by
$\frac{8}{q_{\min}}\eta_{\mathsf{w}}$ with the stated probability.
In both cases, apply Lemma~\ref{lem:unif_exp_dev_ideal_process} to bound the ideal-weight sampling deviation terms in
\eqref{eq:split_obj_residual}--\eqref{eq:split_con_residual} by $\eta_{\mathsf{s}}$ with the stated probability.
A union bound over the two events yields \eqref{eq:kde_consistency_obj}--\eqref{eq:kde_consistency_con} and
\eqref{eq:gen_consistency_obj}--\eqref{eq:gen_consistency_con}, respectively.

\paragraph{Step 4: Consistency.}
The convergence-in-probability statements follow by choosing sequences
$(h_M,\eta_{\mathsf{w}},\eta_{\mathsf{s}})$ satisfying the conditions so that the right-hand sides vanish while the failure probabilities vanish.
\end{proof}

\subsection{Near optimality}
\label{appendix:near_optimality}

Let $\theta^\star$ denote an optimal solution of the true chance-constrained problem
\eqref{eq:problem_option_smdp_function_CCO}, and let
\[
\widehat{\theta}_M
\in
\arg\max_{\theta\in\Theta}
\widetilde{J}(\theta,\bm{\mathrm{s}}_0,\mathcal{D}_M)
\]
denote an optimal solution of the sample-based approximation
\ref{eq:problem_option_smdp_function_CCO_SA}. 
In the following argument, $\theta^\star$ is understood as an optimal solution of problem \ref{eq:problem_option_smdp_function_CCO} with an arbitrarily small tightening level in the safety constraint. Since this tightening level can be chosen arbitrarily close to zero, the corresponding optimal value converges to that of the original problem. 
Recall the uniform residual processes defined in
\eqref{eq:obj_residual_uniform} and \eqref{eq:con_residual_uniform}.

\begin{mytheo}[Near optimality under uniform residual bounds]
\label{theo:near_optimality_general}
Fix $\bm{\mathrm{s}}_0$.
Suppose Assumption~\ref{assump:probability_distribution} and the standing regularity conditions hold.
Let $\delta_{\mathsf{obj}}>0$ and $\delta_{\mathsf{con}}>0$.

Assume that the following uniform residual bounds hold:
\[
\overline{\mathcal{R}}^{\mathsf{obj}}_M(\bm{\mathrm{s}}_0)\le \delta_{\mathsf{obj}},
\qquad
\overline{\mathcal{R}}^{\mathsf{con}}_{M,k}(\bm{\mathrm{s}}_0)\le \delta_{\mathsf{con}},
\quad \forall k\in\{1,\ldots,N\}.
\]

Then the following statements hold:

\paragraph{(i) Feasibility.}
If the sample-based solution $\widehat{\theta}_M$ satisfies
\[
\mathcal{P}_{M,k}(\widehat{\theta}_M,\bm{\mathrm{s}}_0)
\ge
1-\varepsilon+\delta_{\mathsf{con}},
\qquad
\forall k\in\{1,\ldots,N\},
\]
then $\widehat{\theta}_M$ is feasible for the original chance constraints
\eqref{eq:problem_option_smdp__function_CCOsafe}, that is,
\[
\mathcal{P}_k(\widehat{\theta}_M,\bm{\mathrm{s}}_0)
\ge
1-\varepsilon,
\qquad
\forall k\in\{1,\ldots,N\}.
\]

\paragraph{(ii) Near optimality.}
The objective value achieved by $\widehat{\theta}_M$ satisfies
\[
J(\widehat{\theta}_M,\bm{\mathrm{s}}_0)
\ge
J(\theta^\star,\bm{\mathrm{s}}_0)-2\delta_{\mathsf{obj}}.
\]
\end{mytheo}

\begin{proof}
Fix $\bm{\mathrm{s}}_0$.

\paragraph{Feasibility.}
For any $\theta\in\Theta$ and each $k$,
\[
\left|
\mathcal{P}_{M,k}(\theta,\bm{\mathrm{s}}_0)
-
\mathcal{P}_k(\theta,\bm{\mathrm{s}}_0)
\right|
\le
\delta_{\mathsf{con}}.
\]
Therefore
\[
\mathcal{P}_k(\widehat{\theta}_M,\bm{\mathrm{s}}_0)
\ge
\mathcal{P}_{M,k}(\widehat{\theta}_M,\bm{\mathrm{s}}_0)-\delta_{\mathsf{con}}
\ge
1-\varepsilon.
\]

\paragraph{Near optimality.}
By definition of the uniform residual,
\[
\left|
\widetilde{J}(\theta,\bm{\mathrm{s}}_0,\mathcal{D}_M)
-
J(\theta,\bm{\mathrm{s}}_0)
\right|
\le
\delta_{\mathsf{obj}}
\qquad
\forall \theta\in\Theta.
\]
Since $\widehat{\theta}_M$ maximizes the empirical objective,
\[
\widetilde{J}(\widehat{\theta}_M,\bm{\mathrm{s}}_0,\mathcal{D}_M)
\ge
\widetilde{J}(\theta^\star,\bm{\mathrm{s}}_0,\mathcal{D}_M).
\]
Combining the above inequalities yields
\[
J(\widehat{\theta}_M,\bm{\mathrm{s}}_0)
\ge
\widetilde{J}(\widehat{\theta}_M,\bm{\mathrm{s}}_0,\mathcal{D}_M)-\delta_{\mathsf{obj}}
\ge
\widetilde{J}(\theta^\star,\bm{\mathrm{s}}_0,\mathcal{D}_M)-\delta_{\mathsf{obj}}
\ge
J(\theta^\star,\bm{\mathrm{s}}_0)-2\delta_{\mathsf{obj}}.
\]
Finally, because the tightening level is arbitrary and can be taken to zero, the optimal value of the tightened population problem approaches that of the original population problem. 
Hence the additional gap induced by tightening vanishes, and the claimed near-optimality bound for the original problem follows.
\end{proof}

We now specialize the above result to the two importance-weight estimation
methods considered in this paper.

\begin{myprop}[Near optimality with KDE-based importance weights]
\label{prop:near_opt_kde}
Suppose the conditions of Theorem~\ref{theo:unif_exp_dev_weights_kde} hold.
Let $\widehat{\theta}_M$ be the solution of the sample-based problem
\eqref{eq:problem_option_smdp_function_CCO_SA} with KDE-based importance weights.
Let $\delta_{\mathsf{obj}}>0$ and $\delta_{\mathsf{con}}>0$.

Then
\begin{align}
\label{eq:kde_near_opt_prob}
\mathsf{Pr}\Big\{
J(\widehat{\theta}_M,\bm{\mathrm{s}}_0)
<
J(\theta^\star,\bm{\mathrm{s}}_0)-2\delta_{\mathsf{obj}}
\Big\}
\le\ &
\left(
1+\frac{C_3 L_\theta\,\mathsf{diam}(\Theta)}{\eta_{\mathsf{w}}}
\right)^{n_\theta}
\exp\!\left(
- C_2\, M h_M^{d_{\bm{\xi}}}\, \eta_{\mathsf{w}}^2
\right)
\nonumber\\
&+
N
\left(
1+\frac{C_0 L_\theta\,\mathsf{diam}(\Theta)}{\eta_{\mathsf{s}}}
\right)^{n_\theta}
\exp\!\left(
- C\, M\, \eta_{\mathsf{s}}^2
\right),
\end{align}
provided that
\[
C_1 h_M^\beta+\eta_{\mathsf{w}}\le q_{\min}/4,
\]
and
\[
M R_{\max}\cdot
\frac{8}{q_{\min}}\Big(C_1 h_M^\beta+\eta_{\mathsf{w}}\Big)
+\eta_{\mathsf{s}}
\le
\delta_{\mathsf{obj}},
\]
\[
M\cdot
\frac{8}{q_{\min}}\Big(C_1 h_M^\beta+\eta_{\mathsf{w}}\Big)
+\eta_{\mathsf{s}}
\le
\delta_{\mathsf{con}}.
\]
\end{myprop}

\begin{proof}
By Theorem~\ref{theo:consistency_residuals_two_methods}, with probability at least
\[
1-
\left(
1+\frac{C_3 L_\theta\,\mathsf{diam}(\Theta)}{\eta_{\mathsf{w}}}
\right)^{n_\theta}
\exp\!\left(
- C_2\, M h_M^{d_{\bm{\xi}}}\, \eta_{\mathsf{w}}^2
\right)
-
\left(
1+\frac{C_0 L_\theta\,\mathsf{diam}(\Theta)}{\eta_{\mathsf{s}}}
\right)^{n_\theta}
\exp\!\left(
- C\, M\, \eta_{\mathsf{s}}^2
\right),
\]
the residual bounds hold. Therefore
Theorem~\ref{theo:near_optimality_general} implies the near-optimality bound.
Applying Boole's inequality over the $N$ constraint residual events yields
\eqref{eq:kde_near_opt_prob}.
\end{proof}

\begin{myprop}[Near optimality with generator-based importance weights]
\label{prop:near_opt_gen}
Suppose the conditions of Theorem~\ref{theo:unif_exp_dev_weights_gen} hold.
Let $\widehat{\theta}_M$ be the solution of the sample-based problem
\eqref{eq:problem_option_smdp_function_CCO_SA} with generator-based importance weights.
Let $\delta_{\mathsf{obj}}>0$ and $\delta_{\mathsf{con}}>0$.

Then
\begin{align}
\label{eq:gen_near_opt_prob}
\mathsf{Pr}\Big\{
J(\widehat{\theta}_M,\bm{\mathrm{s}}_0)
<
J(\theta^\star,\bm{\mathrm{s}}_0)-2\delta_{\mathsf{obj}}
\Big\}
\le\ &
\left(\frac{C_{\mathsf{gen}}}{\eta_{\mathsf{w}}}\right)^{V_{\mathsf{gen}}}
\exp\!\left(-C_4\,M\,\eta_{\mathsf{w}}^2\right)
\nonumber\\
&+
N
\left(
1+\frac{C_0 L_\theta\,\mathsf{diam}(\Theta)}{\eta_{\mathsf{s}}}
\right)^{n_\theta}
\exp\!\left(
- C\, M\, \eta_{\mathsf{s}}^2
\right),
\end{align}
provided that
\[
\eta_{\mathsf{w}}\le q_{\min}/4,
\]
and
\[
M R_{\max}\cdot
\frac{8}{q_{\min}}\eta_{\mathsf{w}}
+\eta_{\mathsf{s}}
\le
\delta_{\mathsf{obj}},
\]
\[
M\cdot
\frac{8}{q_{\min}}\eta_{\mathsf{w}}
+\eta_{\mathsf{s}}
\le
\delta_{\mathsf{con}}.
\]
\end{myprop}

\begin{proof}
The proof is identical to that of Proposition~\ref{prop:near_opt_kde}.
The result follows from Theorem~\ref{theo:near_optimality_general}
and the explicit probability bounds obtained from
Theorem~\ref{theo:unif_exp_dev_weights_gen} and
Theorem~\ref{theo:consistency_residuals_two_methods}.
\end{proof}

\subsection{Feasibility result}
\label{appendix:feasibility_result}

In this subsection we provide feasibility guarantees for the solution obtained
from the sample-based approximation problem
\eqref{eq:problem_option_smdp_function_CCO_SA}.
The result follows directly from the uniform residual bounds introduced in
Section~\ref{appendix:near_optimality}.

Let $\widehat{\theta}_M$ denote an optimal solution of the sample-based problem
\eqref{eq:problem_option_smdp_function_CCO_SA}, i.e.,
\[
\widehat{\theta}_M
\in
\arg\max_{\theta\in\Theta}
\widetilde{J}(\theta,\bm{\mathrm{s}}_0,\mathcal{D}_M).
\]

Recall that $\mathcal{P}_k(\theta,\bm{\mathrm{s}}_0)$ denotes the true feasibility
probability of constraint $k$, and
$\mathcal{P}_{M,k}(\theta,\bm{\mathrm{s}}_0)$ denotes its sample-based estimate.
We also recall the uniform constraint residual defined in
\eqref{eq:con_residual_uniform}.

\begin{mytheo}[Feasibility under uniform residual bounds]
\label{theo:feasibility_uniform}
Fix $\bm{\mathrm{s}}_0$.
Suppose Assumption~\ref{assump:probability_distribution} and the standing
regularity conditions hold.
Let $\delta_{\mathsf{con}}>0$.

Assume that the uniform constraint residual satisfies
\[
\overline{\mathcal{R}}^{\mathsf{con}}_{M,k}(\bm{\mathrm{s}}_0)
\le
\delta_{\mathsf{con}},
\qquad
\forall k\in\{1,\ldots,N\}.
\]

If the sample-based solution $\widehat{\theta}_M$ satisfies
\[
\mathcal{P}_{M,k}(\widehat{\theta}_M,\bm{\mathrm{s}}_0)
\ge
1-\varepsilon+\delta_{\mathsf{con}},
\qquad
\forall k\in\{1,\ldots,N\},
\]
then $\widehat{\theta}_M$ is feasible for the original chance constraints
\eqref{eq:problem_option_smdp__function_CCOsafe}, that is,
\[
\mathcal{P}_k(\widehat{\theta}_M,\bm{\mathrm{s}}_0)
\ge
1-\varepsilon,
\qquad
\forall k\in\{1,\ldots,N\}.
\]
\end{mytheo}

\begin{proof}
For any $\theta\in\Theta$ and each constraint $k$, the definition of the
uniform residual implies
\[
\left|
\mathcal{P}_{M,k}(\theta,\bm{\mathrm{s}}_0)
-
\mathcal{P}_k(\theta,\bm{\mathrm{s}}_0)
\right|
\le
\delta_{\mathsf{con}}.
\]

Applying this bound at $\theta=\widehat{\theta}_M$ yields
\[
\mathcal{P}_k(\widehat{\theta}_M,\bm{\mathrm{s}}_0)
\ge
\mathcal{P}_{M,k}(\widehat{\theta}_M,\bm{\mathrm{s}}_0)
-
\delta_{\mathsf{con}}.
\]

Since
\[
\mathcal{P}_{M,k}(\widehat{\theta}_M,\bm{\mathrm{s}}_0)
\ge
1-\varepsilon+\delta_{\mathsf{con}},
\]
we obtain
\[
\mathcal{P}_k(\widehat{\theta}_M,\bm{\mathrm{s}}_0)
\ge
1-\varepsilon.
\]
Hence $\widehat{\theta}_M$ satisfies the original chance constraints.
\end{proof}

We next combine the above deterministic feasibility result with the
probabilistic bounds on the uniform residual.

\begin{mytheo}[Finite-sample feasibility guarantee]
\label{theo:feasibility_prob}
Suppose the assumptions of
Theorem~\ref{theo:consistency_residuals_two_methods}
hold.
Let $\widehat{\theta}_M$ be the solution of the sample-based problem
\eqref{eq:problem_option_smdp_function_CCO_SA}.

Then for any $\delta_{\mathsf{con}}>0$,
\begin{align*}
\mathsf{Pr}\Big\{
\exists k :
\mathcal{P}_k(\widehat{\theta}_M,\bm{\mathrm{s}}_0)
<
1-\varepsilon
\Big\}
\le
\sum_{k=1}^{N}
\mathsf{Pr}\Big\{
\overline{\mathcal{R}}^{\mathsf{con}}_{M,k}(\bm{\mathrm{s}}_0)
>
\delta_{\mathsf{con}}
\Big\}.
\end{align*}

Under the exponential deviation bounds derived earlier,
the right-hand side admits an explicit exponential upper bound
that decays with the sample size $M$.
\end{mytheo}

\begin{proof}
If the residual bounds
\[
\overline{\mathcal{R}}^{\mathsf{con}}_{M,k}(\bm{\mathrm{s}}_0)
\le
\delta_{\mathsf{con}}
\]
hold for all $k$, then Theorem~\ref{theo:feasibility_uniform}
guarantees that $\widehat{\theta}_M$ is feasible.

Therefore the failure event
\[
\Big\{
\exists k :
\mathcal{P}_k(\widehat{\theta}_M,\bm{\mathrm{s}}_0)
<
1-\varepsilon
\Big\}
\]
can occur only if at least one residual bound fails.
Applying the union bound over $k=1,\ldots,N$ yields the result.
\end{proof}

\subsection{Extension to random initial states}
\label{appendix:random_initial_states}

All previous results in this appendix are stated for a fixed initial interaction state $\bm{\mathrm{s}}_0$.
This conditional viewpoint is practically meaningful, since in many deployment settings the decision maker starts from a realized current state and seeks guarantees conditioned on that state.
For example, in treatment recommendation or intervention scheduling, the policy is applied to a patient with a given current condition rather than to an averaged population state.
Therefore, the fixed-$\bm{\mathrm{s}}_0$ analysis already provides a practically relevant finite-sample guarantee.

Nevertheless, it is also common in reinforcement learning to evaluate performance under a random initial state
\[
\bm{\mathrm{s}}_0 \sim \rho_0,
\]
where $\rho_0$ is an initial-state distribution.
In this case, the objective becomes
\begin{equation}
\label{eq:J_rho0}
J_{\rho_0}(\theta)
:=
\mathbb{E}_{\bm{\mathrm{s}}_0\sim\rho_0}
\left[
J(\theta,\bm{\mathrm{s}}_0)
\right],
\end{equation}
and similarly the empirical objective is defined by
\begin{equation}
\label{eq:J_M_rho0}
\mathcal{J}_{M,\rho_0}(\theta)
:=
\mathbb{E}_{\bm{\mathrm{s}}_0\sim\rho_0}
\left[
\mathcal{J}_M(\theta,\bm{\mathrm{s}}_0)
\right].
\end{equation}
For the constraint process, one may analogously define
\begin{equation}
\label{eq:P_rho0}
\mathcal{P}_{k,\rho_0}(\theta)
:=
\mathbb{E}_{\bm{\mathrm{s}}_0\sim\rho_0}
\left[
\mathcal{P}_k(\theta,\bm{\mathrm{s}}_0)
\right],
\qquad
\mathcal{P}_{M,k,\rho_0}(\theta)
:=
\mathbb{E}_{\bm{\mathrm{s}}_0\sim\rho_0}
\left[
\mathcal{P}_{M,k}(\theta,\bm{\mathrm{s}}_0)
\right].
\end{equation}

The corresponding residual processes are then given by
\begin{equation}
\label{eq:obj_residual_rho0}
\mathcal{R}^{\mathsf{obj}}_{M,\rho_0}(\theta)
:=
\left|
\mathcal{J}_{M,\rho_0}(\theta)-J_{\rho_0}(\theta)
\right|,
\qquad
\overline{\mathcal{R}}^{\mathsf{obj}}_{M,\rho_0}
:=
\sup_{\theta\in\Theta}
\mathcal{R}^{\mathsf{obj}}_{M,\rho_0}(\theta),
\end{equation}
and
\begin{equation}
\label{eq:con_residual_rho0}
\mathcal{R}^{\mathsf{con}}_{M,k,\rho_0}(\theta)
:=
\left|
\mathcal{P}_{M,k,\rho_0}(\theta)-\mathcal{P}_{k,\rho_0}(\theta)
\right|,
\qquad
\overline{\mathcal{R}}^{\mathsf{con}}_{M,k,\rho_0}
:=
\sup_{\theta\in\Theta}
\mathcal{R}^{\mathsf{con}}_{M,k,\rho_0}(\theta).
\end{equation}

The extension from fixed $\bm{\mathrm{s}}_0$ to random initial states is obtained by averaging the conditional bounds derived earlier.
Indeed, by Jensen's inequality,
\begin{align}
\label{eq:Jensen_obj_rho0}
\overline{\mathcal{R}}^{\mathsf{obj}}_{M,\rho_0}
&=
\sup_{\theta\in\Theta}
\left|
\mathbb{E}_{\bm{\mathrm{s}}_0\sim\rho_0}
\left[
\mathcal{J}_M(\theta,\bm{\mathrm{s}}_0)-J(\theta,\bm{\mathrm{s}}_0)
\right]
\right|
\le
\mathbb{E}_{\bm{\mathrm{s}}_0\sim\rho_0}
\left[
\overline{\mathcal{R}}^{\mathsf{obj}}_M(\bm{\mathrm{s}}_0)
\right],
\\
\label{eq:Jensen_con_rho0}
\overline{\mathcal{R}}^{\mathsf{con}}_{M,k,\rho_0}
&=
\sup_{\theta\in\Theta}
\left|
\mathbb{E}_{\bm{\mathrm{s}}_0\sim\rho_0}
\left[
\mathcal{P}_{M,k}(\theta,\bm{\mathrm{s}}_0)-\mathcal{P}_k(\theta,\bm{\mathrm{s}}_0)
\right]
\right|
\le
\mathbb{E}_{\bm{\mathrm{s}}_0\sim\rho_0}
\left[
\overline{\mathcal{R}}^{\mathsf{con}}_{M,k}(\bm{\mathrm{s}}_0)
\right].
\end{align}
Therefore, if the fixed-state residual bounds established in
Appendix~\ref{appendix:consistency_estimation} hold uniformly in $\bm{\mathrm{s}}_0$ (or admit $\rho_0$-integrable envelopes), then the same rates immediately transfer to the averaged residuals in \eqref{eq:obj_residual_rho0}--\eqref{eq:con_residual_rho0}.

This observation yields the following extension.

\begin{myprop}[Extension to random initial states]
\label{prop:random_initial_states}
Assume that the conclusions of Theorem~\ref{theo:consistency_residuals_two_methods},
Theorem~\ref{theo:near_optimality_general}, and Theorem~\ref{theo:feasibility_uniform}
hold for every $\bm{\mathrm{s}}_0$ in the support of $\rho_0$, and that the corresponding bounds are uniform in $\bm{\mathrm{s}}_0$ or dominated by $\rho_0$-integrable envelopes.
Then the same conclusions extend to the averaged problem with initial state $\bm{\mathrm{s}}_0\sim\rho_0$.
In particular:

\paragraph{(i) Uniform consistency under $\rho_0$.}
The averaged residuals
$\overline{\mathcal{R}}^{\mathsf{obj}}_{M,\rho_0}$
and
$\overline{\mathcal{R}}^{\mathsf{con}}_{M,k,\rho_0}$
converge to zero in probability under the same asymptotic conditions as in Theorem~\ref{theo:consistency_residuals_two_methods}.

\paragraph{(ii) Near optimality under $\rho_0$.}
If $\widehat{\theta}_M$ is an optimal solution of the sample-based averaged problem and $\theta^\star_{\rho_0}$ is an optimal solution of the true averaged problem, then
\[
J_{\rho_0}(\widehat{\theta}_M)
\ge
J_{\rho_0}(\theta^\star_{\rho_0})-2\delta_{\mathsf{obj}},
\]
whenever the averaged objective residual is bounded by $\delta_{\mathsf{obj}}$.

\paragraph{(iii) Feasibility under $\rho_0$.}
If the averaged empirical constraints are satisfied with margin $\delta_{\mathsf{con}}$, then the averaged true constraints are satisfied as well.
\end{myprop}

\begin{proof}
The proof follows by averaging the fixed-$\bm{\mathrm{s}}_0$ bounds with respect to $\rho_0$.
Equations~\eqref{eq:Jensen_obj_rho0} and \eqref{eq:Jensen_con_rho0} show that the averaged residuals are controlled by the $\rho_0$-average of the fixed-state residuals.
Hence the consistency result follows immediately from Theorem~\ref{theo:consistency_residuals_two_methods} under the stated uniformity or integrability conditions.

Once the averaged residual bounds are established, the proofs of near optimality and feasibility are identical to those of Theorem~\ref{theo:near_optimality_general} and Theorem~\ref{theo:feasibility_uniform}, with $J(\theta,\bm{\mathrm{s}}_0)$ and $\mathcal{P}_k(\theta,\bm{\mathrm{s}}_0)$ replaced by their averaged counterparts $J_{\rho_0}(\theta)$ and $\mathcal{P}_{k,\rho_0}(\theta)$.
\end{proof}

Proposition~\ref{prop:random_initial_states} shows that the fixed-initial-state analysis adopted in this appendix is not restrictive.
It captures the practically relevant conditional guarantee for a realized current state, while also extending to the more classical reinforcement-learning setting with random initial states under mild additional regularity conditions.

\section{Detailed Theoretical Results of Section \ref{subsec:practical_implementation}}
\label{appendix:proofs_practical_implementation}

\subsection{Endpoint-Deviation Approximation}
\label{appendix:endpoint_deviation}

The theoretical safety-tightening term in Theorem \ref{theo:selection_safety} is derived from an upper bound on the maximum state deviation over an interaction interval, namely,
\[
\sup_{t\in[t_k,t_k+\delta t_k]}\|\bm{\mathrm{x}}_t-\bm{\mathrm{x}}_k\|.
\]
In the practical implementation, however, the dataset $\mathcal{D}^{\mathsf{pr}}_{N_{\mathsf{pr}}}$ uses the endpoint deviation
\[
d^{(i)}:=\|\bm{\mathrm{x}}^{(i)}_{+}-\bm{\mathrm{x}}^{(i)}\|
\]
as a computationally simple proxy. 
This approximation is adopted for implementation convenience and does not claim exact equivalence to the interval-supremum deviation in the general case. Accordingly, the empirical surrogate used in the experiments should be interpreted as a pragmatic approximation of the theoretical conservative surrogate. A more faithful construction based directly on the interval maximum deviation is an important direction for future improvement.

\subsection{Exact Conservative Surrogate}
\label{appendix:exact_conservative_surrogate}

For each $(\bm{\mathrm{x}},\delta t)$, the quantity $h(\bm{\mathrm{x}},\delta t)$ is deterministic once $\lambda_{o_k}$, $b_{k,o_k}$, and $\overline{W}_k(\varepsilon)$ are fixed. 
Define
\begin{equation}
\label{eq:epsilon_h_definition}
\epsilon_h(\bm{\mathrm{x}},\delta t)
:=
\mathsf{Pr}\left\{
d\ge h(\bm{\mathrm{x}},\delta t)\mid \bm{\mathrm{x}},\delta t
\right\},
\end{equation}
which is the conditional probability that the sampled finite-interval deviation exceeds the analytical safety-tightening term $h(\bm{\mathrm{x}},\delta t)$ at the given pair $(\bm{\mathrm{x}},\delta t)$.
We further define the uniform lower bound $\bar{\epsilon}_h
:=
\inf_{\bm{\mathrm{x}},\delta t}
\epsilon_h(\bm{\mathrm{x}},\delta t).$.
In practical implementation, we only consider $(\bm{\mathrm{x}},\delta t)$ in the operational region of the proposed policy, where $\delta t\in[0,t_{\max}]$ and $\bm{\mathrm{x}}$ is restricted to a compact set of nominal patient states. 
It is thus reasonable to assume that $\bar{\epsilon}_h>0$, namely, there exists a nonvanishing lower bound on the probability that the sampled deviation exceeds $h(\bm{\mathrm{x}},\delta t)$. 
If the patient state leaves this nominal region, the proposed policy is no longer intended to be used, and a separate emergency mechanism or backup controller should be activated instead in practice.
For each $(\bm{\mathrm{x}},\delta t)$, define
\begin{equation}
\label{eq:def_v_epsilon_h}
v_{\bar{\epsilon}_h}(\bm{\mathrm{x}},\delta t)
:=
\max_{d_{\mathsf{s}}\in D_{\bar{\epsilon}_h}(\bm{\mathrm{x}},\delta t)} d_{\mathsf{s}},\ \text{where}\ D_{\bar{\epsilon}_h}(\bm{\mathrm{x}},\delta t)
:=
\left\{
d_{\mathsf{s}}:
\mathsf{Pr}\left\{
d\ge d_{\mathsf{s}}\mid \bm{\mathrm{x}},\delta t
\right\}
=
\bar{\epsilon}_h
\right\}.
\end{equation}
The quantity $v_{\bar{\epsilon}_h}(\bm{\mathrm{x}},\delta t)$ is the largest deviation threshold whose conditional exceedance probability equals $\bar{\epsilon}_h$.
Lemma \ref{lemma:v_epsilon_h} shows that $v_{\bar{\epsilon}_h}(\bm{\mathrm{x}},\delta t)$ is a conservative surrogate of $h(\bm{\mathrm{x}},\delta t)$.

\begin{mylemma}
\label{lemma:v_epsilon_h}
For any $(\bm{\mathrm{x}},\delta t)$, we have
\begin{equation}
\label{eq:lemma_v_epsilon_h_result}
v_{\bar{\epsilon}_h}(\bm{\mathrm{x}},\delta t)\ge h(\bm{\mathrm{x}},\delta t).
\end{equation}
\end{mylemma}
\begin{proof}
Fix any $(\bm{\mathrm{x}},\delta t)$.
By the definition of $\epsilon_h(\bm{\mathrm{x}},\delta t)$ in \eqref{eq:epsilon_h_definition}.
By the definition of $\bar{\epsilon}_h$, we have
\begin{equation}
\epsilon_h(\bm{\mathrm{x}},\delta t)\ge \bar{\epsilon}_h.
\end{equation}
Since the conditional tail probability $d_{\mathsf{s}}\mapsto
\mathsf{Pr}\left\{
d\ge d_{\mathsf{s}}\mid \bm{\mathrm{x}},\delta t
\right\}$
is continuous in $d_{\mathsf{s}}$, there exists some $\bar d_{\mathsf{s}}\ge h(\bm{\mathrm{x}},\delta t)$ such that
\begin{equation}
\mathsf{Pr}\left\{
d\ge \bar d_{\mathsf{s}}\mid \bm{\mathrm{x}},\delta t
\right\}
=
\bar{\epsilon}_h.
\end{equation}
Hence $\bar d_{\mathsf{s}}\in D_{\bar{\epsilon}_h}(\bm{\mathrm{x}},\delta t)$, which implies
\begin{equation}
v_{\bar{\epsilon}_h}(\bm{\mathrm{x}},\delta t)
=
\max_{d_{\mathsf{s}}\in D_{\bar{\epsilon}_h}(\bm{\mathrm{x}},\delta t)} d_{\mathsf{s}}
\ge
\bar d_{\mathsf{s}}
\ge
h(\bm{\mathrm{x}},\delta t).
\end{equation}
This proves \eqref{eq:lemma_v_epsilon_h_result}.
\end{proof}

Lemma \ref{lemma:v_epsilon_h} shows that $v_{\bar{\epsilon}_h}(\bm{\mathrm{x}},\delta t)$ uniformly upper-bounds the analytical safety-tightening term $h(\bm{\mathrm{x}},\delta t)$ for every $(\bm{\mathrm{x}},\delta t)$ in the considered operating region in a point-wise way. 
Therefore, $v_{\bar{\epsilon}_h}(\bm{\mathrm{x}},\delta t)$ can be regarded as a conservative surrogate of $h(\bm{\mathrm{x}},\delta t)$ according to Definition \ref{def:conservative_surrogate}.
This interpretation is important for practical implementation, because $v_{\bar{\epsilon}_h}(\bm{\mathrm{x}},\delta t)$ is characterized through the observable conditional deviation distribution, whereas $h(\bm{\mathrm{x}},\delta t)$ depends on the unknown closed-loop growth coefficient $\lambda_{o_k}$ and is not directly computable online.
Consequently, if one can further construct a computable approximation that upper-bounds $v_{\bar{\epsilon}_h}(\bm{\mathrm{x}},\delta t)$, then that approximation also serves as a conservative practical surrogate of the original analytical tightening term.

\subsection{Probabilistic Guarantee of Approximate Conservative Surrogate}
\label{appendix:prob_guarantee_acs}

\paragraph{Problem reformulation.} 
For any given $(\bm{\mathrm{x}},\delta t)$, the quantity $v_{\bar{\epsilon}_h}(\bm{\mathrm{x}},\delta t)$ can be equivalently viewed as the solution of the following chance-constrained optimization problem, referred to as the Exact Conservative Surrogate problem:
\begin{equation}
\label{eq:prob_reform_v_h}
\tag{$\mathsf{ECS}$}
\begin{aligned}
\min_{d_{\mathsf{s}}}\ &\ d_{\mathsf{s}} \\
\mathsf{s.t.}\ &\ \mathsf{Pr}\left\{
d\le d_{\mathsf{s}}\mid \bm{\mathrm{x}},\delta t
\right\}
\ge 1-\bar{\epsilon}_h, \\
&\ d_{\mathsf{s}}\in[0,d_{\max}].
\end{aligned}
\end{equation}
That is, $v_{\bar{\epsilon}_h}(\bm{\mathrm{x}},\delta t)$ is the smallest threshold $d_{\mathsf{s}}$ such that the conditional probability of the sampled deviation being no larger than $d_{\mathsf{s}}$ is at least $1-\bar{\epsilon}_h$.

On the other hand, the sample-based approximation $\tilde{v}_{\bar{\epsilon}_h}^{N_{\mathsf{pr}},N_{\mathsf{re}}}(\bm{\mathrm{x}},\delta t)$ can be interpreted as the solution of the following sample-based approximate problem, referred to as the Approximate Conservative Surrogate problem:
\begin{equation}
\label{eq:prob_reform_v_h_approximate}
\tag{$\mathsf{ACS}$}
\begin{aligned}
\min_{d_{\mathsf{s}}}\ &\ d_{\mathsf{s}} \\
\mathsf{s.t.}\ &\ \frac{1}{N_{\mathsf{re}}}\sum_{j=1}^{N_{\mathsf{re}}}\mathbb{I}\left\{
\tilde{d}_{\mathsf{re}}^{(j)}\le d_{\mathsf{s}}
\right\}
\ge 1-\bar{\epsilon}_h', \\
&\ \tilde{d}_{\mathsf{re}}^{(j)}\sim \tilde{p}_{\mathsf{d}}(\cdot\mid\bm{\mathrm{x}},\delta t),\qquad j=1,\ldots,N_{\mathsf{re}}, \\
&\ d_{\mathsf{s}}\in[0,d_{\max}].
\end{aligned}
\end{equation}
In other words, $\tilde{v}_{\bar{\epsilon}_h}^{N_{\mathsf{pr}},N_{\mathsf{re}}}(\bm{\mathrm{x}},\delta t)$ is the smallest threshold $d_{\mathsf{s}}$ such that at least a proportion $1-\bar{\epsilon}_h'$ of the resampled deviations from the approximate conditional density are no larger than $d_{\mathsf{s}}$. 

Let us define the feasibility sets of problems $(\mathsf{ECS})$ and $(\mathsf{ACS})$ by
\begin{equation}
\label{eq:def_feasible_sets_ecs_acs}
\Theta^{\mathsf{ECS}}_{\bar{\epsilon}_h}(\bm{\mathrm{x}},\delta t)
:=
\left\{
d_{\mathsf{s}}\in[0,d_{\max}] :
\mathsf{Pr}\left\{
d\le d_{\mathsf{s}}\mid \bm{\mathrm{x}},\delta t
\right\}
\ge 1-\bar{\epsilon}_h
\right\},
\end{equation}
and
\begin{equation}
\label{eq:def_feasible_sets_ecs_acs_approx}
\widetilde{\Theta}^{\mathsf{ACS}}_{\bar{\epsilon}_h'}(\bm{\mathrm{x}},\delta t,\mathcal{D}^{\mathsf{pr}}_{N_{\mathsf{pr}}},\widetilde{\mathcal{D}}^{\mathsf{re}}_{N_{\mathsf{re}}})
:=
\left\{
d_{\mathsf{s}}\in[0,d_{\max}] :
\frac{1}{N_{\mathsf{re}}}\sum_{j=1}^{N_{\mathsf{re}}}\mathbb{I}\left\{
\tilde{d}_{\mathsf{re}}^{(j)}\le d_{\mathsf{s}}
\right\}
\ge 1-\bar{\epsilon}_h'
\right\},
\end{equation}
where $\widetilde{\mathcal{D}}^{\mathsf{re}}_{N_{\mathsf{re}}}:=\{\tilde{d}_{\mathsf{re}}^{(j)}\}_{j=1}^{N_{\mathsf{re}}}$ and $\bar{\epsilon}_h'>\bar{\epsilon}_h$.
Let
\begin{equation}
\label{eq:def_kappa_h}
\kappa_h:=\bar{\epsilon}_h-\bar{\epsilon}_h' > 0,
\qquad\text{so that}\qquad
\bar{\epsilon}_h' < \bar{\epsilon}_h.
\end{equation}
The quantity $\bar{\epsilon}_h'$ is introduced to define a stricter empirical constraint than the target true constraint associated with $\bar{\epsilon}_h$. 
Since $\bar{\epsilon}_h' < \bar{\epsilon}_h$, we have
\[
1-\bar{\epsilon}_h' > 1-\bar{\epsilon}_h,
\]
so the empirical chance constraint is tighter and can compensate for sampling error. 
The approximate conservative surrogate uses the stricter probability level $1-\bar{\epsilon}_h'$ with $\bar{\epsilon}_h' < \bar{\epsilon}_h$ to offset sampling error and ensure that the empirical solution remains conservative for the target level $1-\bar{\epsilon}_h$ with high probability.
Furthermore, define
\begin{equation}
\label{eq:def_M_h}
M_h:=
\inf_{d_{\mathsf{s}}\in[0,d_{\max}]}
\left(
\mathsf{Pr}\left\{
d\le d_{\mathsf{s}}\mid \bm{\mathrm{x}},\delta t
\right\}
-
\widehat{\mathsf{Pr}}_{N_{\mathsf{pr}}}\left\{
d\le d_{\mathsf{s}}\mid \bm{\mathrm{x}},\delta t
\right\}
\right),
\end{equation}
where $\widehat{\mathsf{Pr}}_{N_{\mathsf{pr}}}\{\cdot\mid \bm{\mathrm{x}},\delta t\}$ denotes the probability induced by the estimated conditional density $\tilde{p}_{\mathsf{d}}(\cdot\mid \bm{\mathrm{x}},\delta t)$ obtained from the dataset $\mathcal{D}^{\mathsf{pr}}_{N_{\mathsf{pr}}}$.
The following lemma is a direct application of Theorem~3 of \citet{Shen2025PRS}, after identifying $(\mathsf{ECS})$ as the original contextual chance-constrained optimization problem and $(\mathsf{ACS})$ as its resampling-based approximate problem with confidence gap $\kappa_h$; see Theorem~3 of \citet{Shen2025PRS} for the corresponding infeasibility-to-feasibility probability bound. 

\begin{mylemma}
\label{lemma:acs_feasible_for_ecs}
Let $\tilde{v}_{\bar{\epsilon}_h}^{N_{\mathsf{pr}},N_{\mathsf{re}}}(\bm{\mathrm{x}},\delta t)$ be the solution of problem \ref{eq:prob_reform_v_h_approximate}.
Then,
\begin{equation}
\label{eq:lemma_acs_feasible_for_ecs_main}
\mathsf{Pr}\left\{
\tilde{v}_{\bar{\epsilon}_h}^{N_{\mathsf{pr}},N_{\mathsf{re}}}(\bm{\mathrm{x}},\delta t)
\in
\Theta^{\mathsf{ECS}}_{\bar{\epsilon}_h}(\bm{\mathrm{x}},\delta t)
\right\}
\ge
1-\rho_h,
\end{equation}
where
\begin{equation}
\label{eq:rho_h_bound}
\rho_h
:=
\exp\left\{-2N_{\mathsf{re}}(M_h+\kappa_h)^2\right\}
+
A_{h,1}\exp\left\{
-\frac{\kappa_h^2}{A_{h,2}N_{\mathsf{pr}}^{-\frac{2+\gamma}{4}}}
\right\},
\end{equation}
for some positive constants $A_{h,1},A_{h,2}$ and $\gamma\in(0,2)$.
\end{mylemma}

\begin{proof}
Since
$\tilde{v}_{\bar{\epsilon}_h}^{N_{\mathsf{pr}},N_{\mathsf{re}}}(\bm{\mathrm{x}},\delta t)$
is the optimizer of problem $(\mathsf{ACS})$, it is feasible for $(\mathsf{ACS})$, namely,
\begin{equation}
\tilde{v}_{\bar{\epsilon}_h}^{N_{\mathsf{pr}},N_{\mathsf{re}}}(\bm{\mathrm{x}},\delta t)
\in
\widetilde{\Theta}^{\mathsf{ACS}}_{\bar{\epsilon}_h'}(\bm{\mathrm{x}},\delta t,\mathcal{D}^{\mathsf{pr}}_{N_{\mathsf{pr}}},\widetilde{\mathcal{D}}^{\mathsf{re}}_{N_{\mathsf{re}}}).
\end{equation}
Applying Theorem~3 of \citet{Shen2025PRS} to the contextual chance-constrained pair $(\mathsf{ECS})$ and $(\mathsf{ACS})$, with the identification of the confidence gap as $\kappa_h=\bar{\epsilon}_h-\bar{\epsilon}_h'$, yields that the probability for an infeasible point of $(\mathsf{ECS})$ to be feasible for $(\mathsf{ACS})$ is upper-bounded by $\rho_h$ in \eqref{eq:rho_h_bound}. 
Therefore,
\begin{equation}
\mathsf{Pr}\left\{
\tilde{v}_{\bar{\epsilon}_h}^{N_{\mathsf{pr}},N_{\mathsf{re}}}(\bm{\mathrm{x}},\delta t)
\notin
\Theta^{\mathsf{ECS}}_{\bar{\epsilon}_h}(\bm{\mathrm{x}},\delta t)
\right\}
\le
\rho_h,
\end{equation}
which is equivalent to \eqref{eq:lemma_acs_feasible_for_ecs_main}.
\end{proof}

With finite-sample feasibility of the approximate conservative surrogate given by Lemma \ref{lemma:acs_feasible_for_ecs}, we further have the following theorem that claims that $\tilde{v}_{\bar{\epsilon}_h}^{N_{\mathsf{pr}},N_{\mathsf{re}}}(\bm{\mathrm{x}},\delta t)$ is a conservative surrogate of $h(\bm{\mathrm{x}},\delta t)$ with high probability. 
\begin{mytheo}
\label{theo:acs_conservative_surrogate}
Assume that, for each $(\bm{\mathrm{x}},\delta t)$ in the operating region, the conditional cumulative distribution function
\[
F(d_{\mathsf{s}}\mid \bm{\mathrm{x}},\delta t)
:=
\mathsf{Pr}\left\{
d\le d_{\mathsf{s}}\mid \bm{\mathrm{x}},\delta t
\right\}
\]
is continuous and strictly increasing in $d_{\mathsf{s}}$ on $[0,d_{\max}]$.
Let $v_{\bar{\epsilon}_h}(\bm{\mathrm{x}},\delta t)$ be the solution of $(\mathsf{ECS})$, and let $\tilde{v}_{\bar{\epsilon}_h}^{N_{\mathsf{pr}},N_{\mathsf{re}}}(\bm{\mathrm{x}},\delta t)$ be the solution of $(\mathsf{ACS})$.
Then:
\begin{enumerate}
    \item $v_{\bar{\epsilon}_h}(\bm{\mathrm{x}},\delta t)$ is the unique solution of $(\mathsf{ECS})$;
    \item every feasible solution of $(\mathsf{ECS})$ is a conservative surrogate of $h(\bm{\mathrm{x}},\delta t)$;
    \item with probability at least $1-\rho_h$, the approximate solution $\tilde{v}_{\bar{\epsilon}_h}^{N_{\mathsf{pr}},N_{\mathsf{re}}}(\bm{\mathrm{x}},\delta t)$ is a conservative surrogate of $h(\bm{\mathrm{x}},\delta t)$.
\end{enumerate}
\end{mytheo}

\begin{proof}
Since $F(\cdot\mid \bm{\mathrm{x}},\delta t)$ is continuous and strictly increasing, the equation
\[
F(d_{\mathsf{s}}\mid \bm{\mathrm{x}},\delta t)=1-\bar{\epsilon}_h
\]
has a unique solution. Hence the minimization problem $(\mathsf{ECS})$ has a unique optimizer, namely $v_{\bar{\epsilon}_h}(\bm{\mathrm{x}},\delta t)$.

Now let $d_{\mathsf{s}}$ be any feasible solution of $(\mathsf{ECS})$. Then
\[
F(d_{\mathsf{s}}\mid \bm{\mathrm{x}},\delta t)\ge 1-\bar{\epsilon}_h.
\]
By monotonicity of $F$, this implies
\[
d_{\mathsf{s}}\ge v_{\bar{\epsilon}_h}(\bm{\mathrm{x}},\delta t).
\]
By Lemma \ref{lemma:v_epsilon_h},
\[
v_{\bar{\epsilon}_h}(\bm{\mathrm{x}},\delta t)\ge h(\bm{\mathrm{x}},\delta t).
\]
Therefore
\[
d_{\mathsf{s}}\ge h(\bm{\mathrm{x}},\delta t),
\]
so every feasible solution of $(\mathsf{ECS})$ is a conservative surrogate of $h(\bm{\mathrm{x}},\delta t)$.

Finally, by Lemma \ref{lemma:acs_feasible_for_ecs}, with probability at least $1-\rho_h$,
\[
\tilde{v}_{\bar{\epsilon}_h}^{N_{\mathsf{pr}},N_{\mathsf{re}}}(\bm{\mathrm{x}},\delta t)
\in
\Theta^{\mathsf{ECS}}_{\bar{\epsilon}_h}(\bm{\mathrm{x}},\delta t).
\]
Hence, with the same probability, $\tilde{v}_{\bar{\epsilon}_h}^{N_{\mathsf{pr}},N_{\mathsf{re}}}(\bm{\mathrm{x}},\delta t)$ is feasible for $(\mathsf{ECS})$, and thus is a conservative surrogate of $h(\bm{\mathrm{x}},\delta t)$.
\end{proof}

\subsection{Proof of Theorem \ref{theo:selection_safety_empirical_main}}
\label{appendix:proof_theo_selection_safety_empirical_main}
Define the event
\begin{equation}
\label{eq:event_empirical_surrogate_conservative}
\mathcal{E}_h
:=
\left\{
\tilde{v}_{\bar{\epsilon}_h}^{N_{\mathsf{pr}},N_{\mathsf{re}}}(\bm{\mathrm{x}}_k,\delta t_k)
\ge
h(\bm{\mathrm{x}}_k,\delta t_k),
\ \forall k=0,\ldots,N-1
\right\},
\end{equation}
where $\{(\bm{\mathrm{x}}_k,\delta t_k)\}_{k=0}^{N-1}$ denotes the finite sequence of state-duration pairs induced by the learned solution
$\widetilde{\vartheta}_{M,\bar{\epsilon}_h,\delta_{\mathsf{con}}}^{N_{\mathsf{pr}},N_{\mathsf{re}}}$.
By Theorem \ref{theo:acs_conservative_surrogate}, for each fixed interaction index $k$, there exists a pointwise failure probability $\rho_h^{\mathsf{pt}}$ such that
\[
\mathsf{Pr}\left\{
\tilde{v}_{\bar{\epsilon}_h}^{N_{\mathsf{pr}},N_{\mathsf{re}}}(\bm{\mathrm{x}}_k,\delta t_k)
\ge
h(\bm{\mathrm{x}}_k,\delta t_k)
\right\}
\ge 1-\rho_h^{\mathsf{pt}}.
\]
Since at most $K$ interaction steps are executed, Boole's inequality implies
\begin{equation}
\label{eq:event_empirical_surrogate_conservative_prob}
\mathsf{Pr}\left\{\mathcal{E}_h\right\}\ge 1-K\rho_h^{\mathsf{pt}}.
\end{equation}
Define $\rho_h:=K\rho_h^{\mathsf{pt}}$.

On the event $\mathcal{E}_h$, for every interaction index $k=0,\ldots,N-1$ along the learned trajectory,
\begin{equation}
\label{eq:empirical_surrogate_dominates_h}
\tilde{v}_{\bar{\epsilon}_h}^{N_{\mathsf{pr}},N_{\mathsf{re}}}(\bm{\mathrm{x}}_k,\delta t_k)
\ge
h(\bm{\mathrm{x}}_k,\delta t_k).
\end{equation}
Recalling the definitions of the empirical safety function \eqref{eq:H_selection_empirical_main} and the analytical safety function \eqref{eq:H_selection_main}, \eqref{eq:empirical_surrogate_dominates_h} implies
\begin{equation}
\label{eq:empirical_H_dominates_true_H}
\widetilde{H}_{\bar{\epsilon}_h}^{N_{\mathsf{pr}},N_{\mathsf{re}}}(\bm{\mathrm{s}}_k)=g(\bm{\mathrm{x}}_k)+L_{g,\bm{\mathrm{x}}}\,\tilde{v}_{\bar{\epsilon}_h}^{N_{\mathsf{pr}},N_{\mathsf{re}}}(\bm{\mathrm{x}}_k,\delta t_k)\ge g(\bm{\mathrm{x}}_k)+L_{g,\bm{\mathrm{x}}}\,h(\bm{\mathrm{x}}_k,\delta t_k)=H(\bm{\mathrm{s}}_k,o_k),
\end{equation}
since $L_{g,\bm{\mathrm{x}}}\ge 0$ by Assumption \ref{assump:system_property}(a). 
Therefore, on $\mathcal{E}_h$, if\[\widetilde{H}_{\bar{\epsilon}_h}^{N_{\mathsf{pr}},N_{\mathsf{re}}}(\bm{\mathrm{s}}_k)\le 0,\]then necessarily\[H(\bm{\mathrm{s}}_k,o_k)\le 0.\]
Equivalently,
\begin{equation}
\label{eq:indicator_empirical_implies_true}
\mathbb{I}\left\{\widetilde{H}_{\bar{\epsilon}_h}^{N_{\mathsf{pr}},N_{\mathsf{re}}}(\bm{\mathrm{s}}_k)\le 0\right\}\le\mathbb{I}\left\{H(\bm{\mathrm{s}}_k,o_k)\le 0\right\},\qquad \forall k,
\end{equation}
on the event $\mathcal{E}_h$.
Now let $\mathcal{E}_M$
denote the event that the solution $\widetilde{\vartheta}_{M,\bar{\epsilon}_h,\delta_{\mathsf{con}}}^{N_{\mathsf{pr}},N_{\mathsf{re}}}$ obtained from the sample-based Problem \ref{eq:problem_option_smdp_sample_main} with$\widetilde{H}_{\bar{\epsilon}_h}^{N_{\mathsf{pr}},N_{\mathsf{re}}}(\cdot)$in place of $H(\cdot)$ and with safety level $1-\varepsilon+\delta_{\mathsf{con}}$ satisfies the corresponding empirical constraints.On $\mathcal{E}_h\cap\mathcal{E}_M$, by \eqref{eq:indicator_empirical_implies_true}, the same solution also satisfies 
\begin{equation}
\label{eq:sample_constraint_true_H}
\sum_{i=1}^M\omega_i\cdot\mathbb{I}\!\left\{H(\bm{\mathrm{s}}_k^{(i)},o_k^{(i)})\le \bm{0}\right\}\ge1-\varepsilon+\delta_{\mathsf{con}},\qquad \forall k=1,\ldots,N.
\end{equation}

That is, on $\mathcal{E}_h\cap\mathcal{E}_M$, the learned solution
$\widetilde{\vartheta}_{M,\bar{\epsilon}_h,\delta_{\mathsf{con}}}^{N_{\mathsf{pr}},N_{\mathsf{re}}}$
satisfies the sample-based constraints associated with the analytical safety function $H(\cdot)$ at all interaction steps generated by that solution and the tightened probability level $1-\varepsilon+\delta_{\mathsf{con}}$.
Applying Theorem \ref{theo:finite_sample_main}, there exists a failure probability $\rho_M^{\mathsf{con}}$ such that
\begin{equation}
\label{eq:event_sample_feasibility_prob}
\mathsf{Pr}\left\{\mathcal{E}_M\right\}\ge 1-\rho_M^{\mathsf{con}},
\end{equation}
and on $\mathcal{E}_M$ the learned solution is feasible for Problem \ref{eq:problem_option_smdp}.
Combining \eqref{eq:event_empirical_surrogate_conservative_prob} and \eqref{eq:event_sample_feasibility_prob} by Boole's inequality yields
\begin{equation}
\label{eq:joint_event_bound_empirical_main}\mathsf{Pr}\left\{\mathcal{E}_h\cap\mathcal{E}_M\right\}\ge1-\rho_h-\rho_M^{\mathsf{con}}.
\end{equation}
Hence, with probability at least $1-\rho_h-\rho_M^{\mathsf{con}}$, the solution$\widetilde{\vartheta}_{M,\bar{\epsilon}_h,\delta_{\mathsf{con}}}^{N_{\mathsf{pr}},N_{\mathsf{re}}}$ is feasible for Problem \ref{eq:problem_option_smdp}.
This proves the theorem.

\subsection{Approximate Conservative Surrogate by Gaussian Process}
\label{appendix:acs_gp}

While $\tilde{v}_{\bar{\epsilon}_h}^{N_{\mathsf{pr}},N_{\mathsf{re}}}(\bm{\mathrm{x}},\delta t)$ provides a sample-based approximation of the exact conservative surrogate, its construction requires conditional density estimation together with resampling, which may be computationally demanding in online implementation.
To obtain a more convenient approximation, we introduce a Gaussian process (GP) model for the sampled finite-interval deviation.

The input of the GP model is $(\bm{\mathrm{x}},\delta t)$, and its output is
\[
d:=\|\bm{\mathrm{x}}_+-\bm{\mathrm{x}}\|,
\]
where $\bm{\mathrm{x}}_+$ denotes the successor state of $\bm{\mathrm{x}}$ after an inter-interaction duration $\delta t$.
The training dataset is given by
\[
\mathcal{D}^{\mathsf{pr}}_{N_{\mathsf{pr}}}
:=
\left\{
\left(\bm{\mathrm{x}}^{(i)},\delta t^{(i)},d^{(i)}\right)
\right\}_{i=1}^{N_{\mathsf{pr}}}.
\]
Based on this dataset, the GP model provides the predictive mean
\[
\mu_{\mathsf{gp}}(\bm{\mathrm{x}},\delta t)
\]
and predictive variance
\[
\sigma_{\mathsf{gp}}(\bm{\mathrm{x}},\delta t).
\]

Using these quantities, we define the GP-based approximate conservative surrogate by
\begin{equation}
\label{eq:def_gp_surrogate}
\tilde{v}_{\mathsf{gp}}^{\beta_{\mathsf{gp}}}(\bm{\mathrm{x}},\delta t)
:=
\mu_{\mathsf{gp}}(\bm{\mathrm{x}},\delta t)
+
\beta_{\mathsf{gp}}\cdot\sigma_{\mathsf{gp}}(\bm{\mathrm{x}},\delta t),
\end{equation}
where $\beta_{\mathsf{gp}}>0$ is a tunable inflation parameter.
The rationale behind \eqref{eq:def_gp_surrogate} is that the predictive mean captures the nominal sampled finite-interval deviation, while the variance quantifies the uncertainty induced by limited data, exogenous disturbances, and the unmodeled influence of the option-policy input component.
Hence, by choosing $\beta_{\mathsf{gp}}$ sufficiently large, $\tilde{v}_{\mathsf{gp}}^{\beta_{\mathsf{gp}}}(\bm{\mathrm{x}},\delta t)$ can be made to act as a practical conservative surrogate of $h(\bm{\mathrm{x}},\delta t)$.

Compared with $\tilde{v}_{\bar{\epsilon}_h}^{N_{\mathsf{pr}},N_{\mathsf{re}}}(\bm{\mathrm{x}},\delta t)$, the GP-based surrogate \eqref{eq:def_gp_surrogate} is more convenient in practice, since it avoids explicit conditional density estimation and resampling at each online query.
Instead, it directly produces a mean-variance description of the deviation from the current input pair $(\bm{\mathrm{x}},\delta t)$.
Intuitively, $\tilde{v}_{\mathsf{gp}}^{\beta_{\mathsf{gp}}}(\bm{\mathrm{x}},\delta t)$ plays the same role as $\tilde{v}_{\bar{\epsilon}_h}^{N_{\mathsf{pr}},N_{\mathsf{re}}}(\bm{\mathrm{x}},\delta t)$, namely, both are designed to upper-bound the unknown analytical quantity $h(\bm{\mathrm{x}},\delta t)$.
The difference is that $\tilde{v}_{\mathsf{gp}}^{\beta_{\mathsf{gp}}}(\bm{\mathrm{x}},\delta t)$ achieves this through a parametric uncertainty quantifier, whereas $\tilde{v}_{\bar{\epsilon}_h}^{N_{\mathsf{pr}},N_{\mathsf{re}}}(\bm{\mathrm{x}},\delta t)$ relies on empirical quantile approximation.

A fixed large $\beta_{\mathsf{gp}}$ generally increases conservativeness, but may also lead to overly restrictive decisions.
To mitigate this issue, it is natural to adapt $\beta_{\mathsf{gp}}$ according to the local safety risk of the current state.
Specifically, when the current state is close to the unsafe region, the safety margin should be enlarged by increasing $\beta_{\mathsf{gp}}$.
Conversely, when the current state is well inside the safe region, one may reduce $\beta_{\mathsf{gp}}$ to avoid unnecessary conservativeness.
A simple risk-adaptive choice is
\begin{equation}
\label{eq:def_beta_gp_adaptive}
\beta_{\mathsf{gp}}(\bm{\mathrm{x}})
:=
\beta_{\mathsf{gp}}^{\min}
+
\left(\beta_{\mathsf{gp}}^{\max}-\beta_{\mathsf{gp}}^{\min}\right)
\cdot
\chi(\bm{\mathrm{x}}),
\end{equation}
where $0<\beta_{\mathsf{gp}}^{\min}\le \beta_{\mathsf{gp}}^{\max}$ are design parameters and $\chi(\bm{\mathrm{x}})\in[0,1]$ is a risk score satisfying:
larger $\chi(\bm{\mathrm{x}})$ indicates that $\bm{\mathrm{x}}$ is closer to the unsafe region.
For example, $\chi(\bm{\mathrm{x}})$ may be chosen as a monotone transformation of $g(\bm{\mathrm{x}})$, such as
\begin{equation}
\label{eq:def_risk_score_gp}
\chi(\bm{\mathrm{x}})
:=
\mathrm{clip}\left(
\frac{g(\bm{\mathrm{x}})-g_{\min}}{g_{\max}-g_{\min}},
\,0,\,1
\right),
\end{equation}
where $g_{\min}<g_{\max}$ are user-specified reference levels and $\mathrm{clip}(a,0,1):=\min\{1,\max\{0,a\}\}$.

Using \eqref{eq:def_beta_gp_adaptive}, the GP-based surrogate becomes
\begin{equation}
\label{eq:def_gp_surrogate_adaptive}
\tilde{v}_{\mathsf{gp}}(\bm{\mathrm{x}},\delta t)
:=
\mu_{\mathsf{gp}}(\bm{\mathrm{x}},\delta t)
+
\beta_{\mathsf{gp}}(\bm{\mathrm{x}})\cdot\sigma_{\mathsf{gp}}(\bm{\mathrm{x}},\delta t).
\end{equation}
The corresponding empirical safety function is then defined by
\begin{equation}
\label{eq:def_gp_empirical_safety_function}
\widetilde{H}_{\mathsf{gp}}(\bm{\mathrm{s}}_k)
:=
g(\bm{\mathrm{x}}_k)
+
L_{g,\bm{\mathrm{x}}}\cdot
\tilde{v}_{\mathsf{gp}}(\bm{\mathrm{x}}_k,\delta t_k).
\end{equation}

A practical implementation procedure is summarized as follows.

\begin{algorithm}[H]
\caption{Risk-adaptive GP-based conservative surrogate}
\label{alg:gp_conservative_surrogate}
\begin{algorithmic}[1]
\Require Dataset $\mathcal{D}^{\mathsf{pr}}_{N_{\mathsf{pr}}}$, current state $\bm{\mathrm{x}}_k$, candidate duration $\delta t_k$, design parameters $\beta_{\mathsf{gp}}^{\min},\beta_{\mathsf{gp}}^{\max},g_{\min},g_{\max}$
\State Train or update the GP model using $\mathcal{D}^{\mathsf{pr}}_{N_{\mathsf{pr}}}$
\State Compute $\mu_{\mathsf{gp}}(\bm{\mathrm{x}}_k,\delta t_k)$ and $\sigma_{\mathsf{gp}}(\bm{\mathrm{x}}_k,\delta t_k)$
\State Evaluate the risk score $\chi(\bm{\mathrm{x}}_k)$ using \eqref{eq:def_risk_score_gp}
\State Set $\beta_{\mathsf{gp}}(\bm{\mathrm{x}}_k)$ using \eqref{eq:def_beta_gp_adaptive}
\State Compute $\tilde{v}_{\mathsf{gp}}(\bm{\mathrm{x}}_k,\delta t_k)$ by \eqref{eq:def_gp_surrogate_adaptive}
\State Compute $\widetilde{H}_{\mathsf{gp}}(\bm{\mathrm{s}}_k)$ by \eqref{eq:def_gp_empirical_safety_function}
\State Accept the candidate duration if $\widetilde{H}_{\mathsf{gp}}(\bm{\mathrm{s}}_k)\le 0$; otherwise, increase conservativeness or reduce $\delta t_k$
\end{algorithmic}
\end{algorithm}

The above construction is intended as a practical approximation of the exact conservative surrogate framework.
Although no formal theorem is given here, the design is consistent with the earlier analysis: the GP mean approximates the nominal finite-interval deviation, the predictive variance quantifies uncertainty, and the inflation factor $\beta_{\mathsf{gp}}$ is used to turn this uncertainty quantification into a conservative safety margin.
The adaptive update of $\beta_{\mathsf{gp}}$ further reduces over-conservativeness by allocating larger uncertainty margins only when the current state is near the unsafe region.

\section{Details of Sepsis Dynamical Treatment Experiment Construction}
\label{appendix:experimental_details}
\subsection{Dataset for Experimental Environment Constructions}
\label{appendix:details_dataset}
MIMIC-III (Medical Information Mart for Intensive Care) is a large-scale, publicly available clinical database containing detailed records from over $40{,}000$ intensive care admissions~\citep{Alistair}. 
The database includes longitudinal information such as demographics, vital signs, laboratory tests, medications, procedures, and other time-stamped clinical events, thereby providing a realistic foundation for modeling treatment processes in critical care. 
For the sepsis treatment study in this paper, MIMIC-III is particularly suitable because it reflects the temporal structure encountered in real ICU practice: patient conditions evolve continuously over time, while measurements and interventions are recorded at irregular intervals determined by clinical needs rather than by a fixed decision schedule. 
This makes the dataset appropriate for constructing a continuous-time treatment environment in which the elapsed time between observations and interventions meaningfully affects state evolution. 

In addition, MIMIC-III contains rich physiological trajectories and treatment histories, enabling the construction of a clinically meaningful and high-dimensional experimental environment from real patient data rather than from a hand-crafted simulator. 
Its widespread use in healthcare machine learning research, together with its extensive documentation and established preprocessing practices, also supports reproducibility and comparison with prior studies. 
These properties make MIMIC-III an appropriate data foundation for the construction of the sepsis treatment environment used in this paper.

\subsection{Sepsis Treatment Formulation: State, Reward, Safety, Other Evaluation Metrics}
\label{appendix:spesis_treatment_formulation}

\paragraph{Overview.} 
The MIMIC-III Sepsis dataset provides a rich set of physiological variables that form the basis for our state space construction. We identify septic patients and extract trajectories from 24 hours before to 48 hours after sepsis onset. From the available variables, we select a compact subset of clinically relevant features to represent patient states, focusing on organ dysfunction. The treatment actions are modeled as continuous clinical interventions, and the reward is defined using a differentiable surrogate of the SOFA score. Table~\ref{tab:state_selection} summarizes the state, action, and reward design used in our study.

\begin{table}[h]
\centering
\caption{Summary of feature space, actions, and reward}
\begin{tabular}{p{3.2cm}p{5.4cm}p{6.3cm}}
\toprule
\textbf{Category}& \textbf{Variables} & \textbf{Description} \\
\midrule
\textbf{Dynamic Features}& 
SpO\textsubscript{2}, PaO\textsubscript{2}, Total bilirubin, GCS, Urine output, Lactate & 
Key physiological variables reflecting respiratory, hepatic, neurological, renal, and metabolic status, consistent with SOFA-related organ dysfunction. \\
\midrule
\textbf{Actions}& 
FiO\textsubscript{2}, Vasopressor rate, Intravenous fluids & 
Continuous control inputs representing oxygen support, vasopressor administration, and fluid resuscitation in ICU treatment. \\
\midrule
\textbf{Reward}& 
SOFA score with vasopressor penalty & 
A smooth surrogate of SOFA providing dense feedback on organ dysfunction, with an additional penalty to regularize vasopressor usage. \\
\bottomrule
\end{tabular}
\label{tab:state_selection}
\end{table}

\paragraph{Reward Function.} 
The Sequential Organ Failure Assessment (SOFA) score is used as the reward signal. 
The original SOFA is a discrete step function that assigns integer scores 
$\{0,\dots,4\}$ to each variable based on clinical thresholds:
\begin{equation}
\label{eq:sofa_discrete}
\text{SOFA} = \sum_i \sum_k \mathbf{1}(x_i \ge \tau_{i,k}),
\end{equation}
where $\mathbf{1}(\cdot)$ is an indicator function and $\tau_{i,k}$ are predefined cutoffs, 
with $x_i$ denoting physiological variables, $i$ indexing these variables, and $k$ indexing their thresholds. 
In our implementation, thresholds are defined per variable. 
For SpO$_2$, the thresholds are 94, 90, 85, and 80 in percent. 
For bilirubin, the thresholds are 1.2, 2.0, 6.0, and 12.0 mg/dL. 
For GCS, the thresholds are 15, 13, 10, and 6. 
For urine output, the thresholds are 500 and 200 mL/day. 
For vasopressor dosage, the thresholds are 0.0, 0.05, 0.1, and 0.25 $\mu$g/kg/min.
These follow standard SOFA definitions where applicable, with minor adaptations to match available variables.
To obtain a differentiable surrogate, we replace each indicator with a smooth approximation:
\begin{equation}
\label{eq:smooth_step}
\text{smooth\_step}(x;\tau)
= \alpha\,\sigma\!\left(\frac{x-\tau}{w_s}\right)
+ (1-\alpha)\,\sigma\!\left(\frac{x-\tau}{w_l}\right),
\end{equation}
where $\sigma(\cdot)$ is the logistic function. The smooth SOFA is then defined as:
\begin{equation}
\label{eq:sofa_smooth}
\text{SOFA}_{\text{smooth}} = \sum_i \sum_k \text{smooth\_step}(x_i, \tau_{i,k}).
\end{equation}

In practice, we define a modified, data-driven variant of SOFA, 
which considers a subset of physiological systems based on data availability, including respiratory, 
hepatic, neurological, renal, and cardiovascular systems corresponding to SpO$_2$, bilirubin, GCS, urine output, and vasopressor usage, respectively. 
The coagulation component is excluded due to substantial missingness and irregular sampling.
This formulation retains the threshold-based additive structure of SOFA while providing a smooth and continuous approximation tailored to the available data.

The reward is defined as:
\begin{equation}
\label{eq:reward}
r = -\text{SOFA}_{\text{smooth}} - \lambda \cdot \text{vaso},
\end{equation}
which encourages reduction of organ dysfunction while penalizing excessive vasopressor use.

\paragraph{State Space.} 
For treatment optimization, we select state variables through a clinically grounded and data-driven approach. The state representation focuses on features closely related to organ dysfunction and highly correlated with the SOFA score, ensuring alignment with the reward design. These variables capture key aspects of patient health relevant to clinical outcomes and are routinely measured in ICU practice. This design provides a compact yet informative representation, supporting efficient learning while preserving the essential dynamics of sepsis progression.

\paragraph{Explicit Clinical Safety Constraints.} 
We impose a safety constraint based on blood lactate levels, a key biomarker of tissue hypoxia and circulatory failure in sepsis. Lactate provides a sensitive indicator of inadequate perfusion and is strongly associated with patient deterioration, making it well-suited for safety monitoring. Constraining lactate prevents the policy from entering clinically unstable regions not fully captured by the reward, while remaining interpretable and readily measurable in ICU practice.
We set the lactate safety threshold to 8.5 mmol/L. Clinical evidence shows that elevated lactate is strongly associated with increased mortality risk, with risk rising progressively at higher levels~\citep{christie2009serum}. A threshold around 8 mmol/L corresponds to severe physiological derangement and critical instability. Setting the cutoff at 8.5 mmol/L allows the constraint to act as a conservative safety boundary while avoiding overly restrictive behavior during policy optimization.

\paragraph{Appropriate Intensification Rate (AIR).} 
The Appropriate Intensification Rate evaluates whether the model appropriately escalates treatment in response to physiological deterioration.
We define the Urine Output Rate (UOR) as the volume of urine output normalized by patient weight per hour (mL/kg/hr).
A need for intensification arises when either the oxygen saturation (\( \text{SpO}_2 \)) or the UOR falls below a clinically significant threshold.
Formally, AIR is defined as:
\[
\text{AIR} =
\frac{ \sum_{i,t} \mathbb{I} \left\{ \left( \text{SpO}_2(i,t) < \tau_{\text{SpO}_2} \lor \text{UOR}(i,t) < \tau_{\text{UOR}} \right) \land \text{Intensified}(i,t) \right\} }
{ \sum_{i,t} \mathbb{I} \left\{ \text{SpO}_2(i,t) < \tau_{\text{SpO}_2} \lor \text{UOR}(i,t) < \tau_{\text{UOR}} \right\} },
\]
where \( \tau_{\text{SpO}_2} \) is set to 92\%, \( \tau_{\text{UOR}} \) is set to 0.5 mL/kg/hr, and \( \text{Intensified}(i,t) \) is an indicator function equal to 1 if the model recommends an increased treatment intensity at time \( t \).

\subsection{Environment Construction}
\label{appendix:environment_construction}

To evaluate the proposed method in a continuous-time treatment setting, we construct a patient-specific simulation environment from a physiological dynamics model. 
Different from a comprehensive benchmark environment with personalization across multiple patients, the present environment is built for one representative patient. 
This choice allows us to focus on validating the proposed interaction-limited safe continuous-time reinforcement learning framework under a realistic high-dimensional and stochastic treatment dynamics model.

The environment is based on a continuous-time state evolution model of the form
\begin{equation}
\label{eq:appendix_env_dynamics}
\bm{\mathrm{x}}_{t+\delta t}
=
\bm{\mathrm{F}}(\bm{\mathrm{x}}_{t},\bm{\mathrm{u}}_{t},\delta t)
+
\bm{\mathrm{w}}_{t,\delta t},
\end{equation}
where $\bm{\mathrm{x}}_{t}$ is the patient state, $\bm{\mathrm{u}}_{t}$ is the treatment action, $\delta t$ is the elapsed interaction interval, and $\bm{\mathrm{w}}_{t,\delta t}$ denotes stochastic residual uncertainty. 
To obtain a data-driven simulator, we approximate the deterministic part of the dynamics by a Physics-Informed Neural Network (PINN) \citep{RAISSI2019686}, leading to
\begin{equation}
\label{eq:appendix_env_pinn}
\widetilde{\bm{\mathrm{x}}}_{t+\delta t}
=
\bm{\mathrm{F}}^{\mathsf{nn}}(\bm{\mathrm{x}}_{t},\bm{\mathrm{u}}_{t},\delta t)
+
\bm{\mathrm{W}}_{\mathsf{c}},
\qquad
\bm{\mathrm{W}}_{\mathsf{c}}\sim \widetilde{p}(\cdot\mid \bm{\mathrm{x}}_{t},\delta t).
\end{equation}
Here, $\bm{\mathrm{F}}^{\mathsf{nn}}(\bm{\mathrm{x}}_{t},\bm{\mathrm{u}}_{t},\delta t)$ is the PINN approximation of the deterministic dynamics, while $\widetilde{p}(\cdot\mid \bm{\mathrm{x}}_{t},\delta t)$ models the residual stochasticity. 
The PINN is trained from the underlying continuous-time physiological model, so the resulting simulator preserves the continuous-time treatment structure and is more faithful than commonly used coarse discrete-time approximations.

In the experiments, we use two related PINN-based environments. 
The first environment is used for policy training, where the stochastic term in \eqref{eq:appendix_env_pinn} is generated with a relatively small covariance. 
After training, the learned policies are evaluated in a second environment that uses the same deterministic dynamics model $\bm{\mathrm{F}}^{\mathsf{nn}}(\bm{\mathrm{x}}_{t},\bm{\mathrm{u}}_{t},\delta t)$ but a larger covariance in the residual noise term. 
Therefore, the evaluation environment is intentionally more challenging and is used to test the robustness of the learned policies under stronger stochastic perturbations.
In our implementation, the residual term $\widetilde{p}(\cdot\mid \bm{\mathrm{x}}_{t},\delta t)$ is instantiated as an isotropic Gaussian distribution, i.e., $\bm{\mathrm{W}}_{\mathsf{c}} \sim \mathcal{N}(\bm{0}, \delta t\cdot\sigma^2 \bm{I})$, providing a stochastic perturbation to the deterministic PINN dynamics. 
It is reasonable to incorporate $\delta t$ into the variance scale, since a longer interaction interval may lead to greater uncertainty. 
Besides, we choose $\sigma=0.001$ and $\sigma=0.02$ for training and evaluation environments, respectively.
The state variables used in this simulator are exactly those defined in Appendix \ref{appendix:spesis_treatment_formulation}, and are not repeated here. 
The detailed architecture and training procedure of the PINN are omitted, since they are not the focus of this paper; we only use the PINN as a continuous-time surrogate model for environment construction. 
Overall, this environment provides a practically meaningful testbed for the proposed method: it preserves continuous-time dynamics, supports irregular interaction intervals, and includes substantial stochasticity, while remaining simple enough to isolate the effect of the proposed safe continuous-time control framework.

\paragraph{Physical Model for PINN.} We model patient dynamics as a continuous-time ODE system encoding
coarse-grained interactions across respiratory, circulatory, metabolic,
and hepato-renal organ systems. A structured physiological-inspired component
provides inductive bias toward clinically plausible state evolution, and
a learnable residual component captures dynamics beyond this simplified
prior. State derivatives are integrated via explicit Euler with box
constraints.
The physiological state is $\mathbf{x} = [\mathrm{SpO_2},\,
\mathrm{PaO_2},\, \mathrm{Bili},\, \mathrm{GCS},\, \mathrm{Urine},\,
\mathrm{Lac}]^\top$ and the treatment action is $\mathbf{u} =
[\mathrm{FiO_2},\, \mathrm{Vaso},\, \mathrm{Fluid}]^\top$, where
$\mathrm{SpO_2}$ is peripheral oxygen saturation,
$\mathrm{PaO_2}$ arterial oxygen tension,
$\mathrm{Bili}$ serum bilirubin,
$\mathrm{GCS}$ the Glasgow Coma Scale,
$\mathrm{Urine}$ urine output,
$\mathrm{Lac}$ blood lactate,
$\mathrm{FiO_2}$ the fraction of inspired oxygen,
$\mathrm{Vaso}$ vasopressor dose (norepinephrine-equivalent), and
$\mathrm{Fluid}$ intravenous fluid rate.
All variables are standardised to zero mean and unit variance; 
$[\cdot]^+ = \max(\cdot,0)$ denotes rectification. 
The ODE system is defined in the normalized space, where variables represent 
relative deviations from their nominal levels rather than absolute physiological quantities. 
Accordingly, these operations capture directional dependencies rather than strict physical relationships.

Oxygenation is modeled as a first-order drive toward an $\mathrm{FiO_2}$-dependent equilibrium, 
modulated by metabolic stress and fluid support, providing a coarse-grained approximation 
of the oxygen cascade mechanism \citep{west2020respiratory}:
\begin{align}
\label{eq:PaO2_SpO2}
  \frac{d\,\mathrm{PaO}_2}{dt}
    &= k_1(\mathrm{FiO}_2 - \mathrm{PaO}_2)
       - k_2\,[\mathrm{Lac}]^+
       + k_3\,\mathrm{Fluid}, \\[3pt]
  \frac{d\,\mathrm{SpO}_2}{dt}
    &= k_4(\mathrm{PaO}_2 - \mathrm{SpO}_2).
\end{align}

Neurological function is impaired by metabolic acidosis and cerebral
hypoxia \citep{sonneville2023spectrum}:
\begin{align}
  \frac{d\,\mathrm{GCS}}{dt}
    = -k_5\,[\mathrm{Lac}]^+ - k_6\,h(\mathrm{SpO}_2),
\end{align}
where $h(\mathrm{SpO}_2)=\max(-\mathrm{SpO}_2,0)$ is a rectified function that captures hypoxic effects. Here $\mathrm{SpO}_2$ denotes the normalized variable, and the rectification reflects deviations below the nominal level rather than a fixed clinical threshold.
Lactate dynamics reflect vasopressor-induced hypoperfusion, anaerobic
glycolysis, hepatic clearance, and fluid dilution
\citep{prescott2026surviving}:
\begin{align}
  \frac{d\,\mathrm{Lac}}{dt}
    = k_7\,\mathrm{Vaso}
      + k_8\,h(\mathrm{SpO}_2)
      - k_9\,[\mathrm{Lac}]^+
      - k_{10}\,\mathrm{Fluid}.
\end{align}

Renal output responds to fluid administration and is suppressed by
metabolic and vasomotor stress \citep{bellomo2012acute}:
\begin{align}
  \frac{d\,\mathrm{Urine}}{dt}
    = k_{11}\,\mathrm{Fluid}
      - k_{12}\,[\mathrm{Lac}]^+
      - k_{13}\,\mathrm{Vaso}.
\end{align}

Hepatic injury accumulates via lactate-mediated ischaemia and clears
slowly \citep{singer2016third}:
\begin{align}
\label{eq:lac}
  \frac{d\,\mathrm{Bili}}{dt}
    = k_{14}\,[\mathrm{Lac}]^+
      - k_{15}\,[\mathrm{Bili}]^+.
\end{align}

The rate constants $\{k_i\}$ are learnable parameters constrained to
$(0,\,c_i)$ via a scaled sigmoid reparameterisation, where the upper
bounds $\{c_i\}$ encode clinical prior knowledge about interaction
magnitudes. 
The true underlying physiological dynamics are unknown. In particular, the rate constants $\{k_i\}$ are not directly available, and there may also exist a modeling residual between the real dynamics and the mechanistic dynamics described by \eqref{eq:PaO2_SpO2}--\eqref{eq:lac}. 
To account for this mismatch, we adopt the following hybrid physical model:
\begin{equation}
\label{eq:Med_NN}
   \frac{d\bm{\mathrm{x}}}{dt}=\mathsf{MedODE}(\bm{\mathrm{x}},\bm{\mathrm{u}})+\mathsf{NN}(\bm{\mathrm{x}},\bm{\mathrm{u}}).
\end{equation}
Here, $\mathsf{MedODE}(\bm{\mathrm{x}},\bm{\mathrm{u}})$ represents the known mechanistic component corresponding to \eqref{eq:PaO2_SpO2}--\eqref{eq:lac}, while $\mathsf{NN}(\bm{\mathrm{x}},\bm{\mathrm{u}})$ captures the residual dynamics not explained by the physical model. 
The neural component and the unknown parameters $\{k_i\}$ are optimized jointly in an end-to-end manner. 
Based on \eqref{eq:Med_NN}, we further train a neural network surrogate for its solution map over time, so that for any fixed $\bm{\mathrm{u}}$, the state trajectory can be directly approximated as a function of time, whose detailed algorithm refers to \citep{RAISSI2019686}.

\paragraph{Performance of the proposed PINN.} 
Figure~\ref{fig:pinn_fit} shows the full-horizon autoregressive rollout of the trained PINN for a representative ICU stay.
Starting from the initial state, the model integrates the learned ODE
forward using only observed treatment actions, with no data correction.
The rollout closely tracks sparse clinical observations across all six
variables and remains stable over the full trajectory, validating the
PINN as a reliable patient simulator for the RL stage.

\begin{figure}[htbp]
\centering
\includegraphics[width=\linewidth]{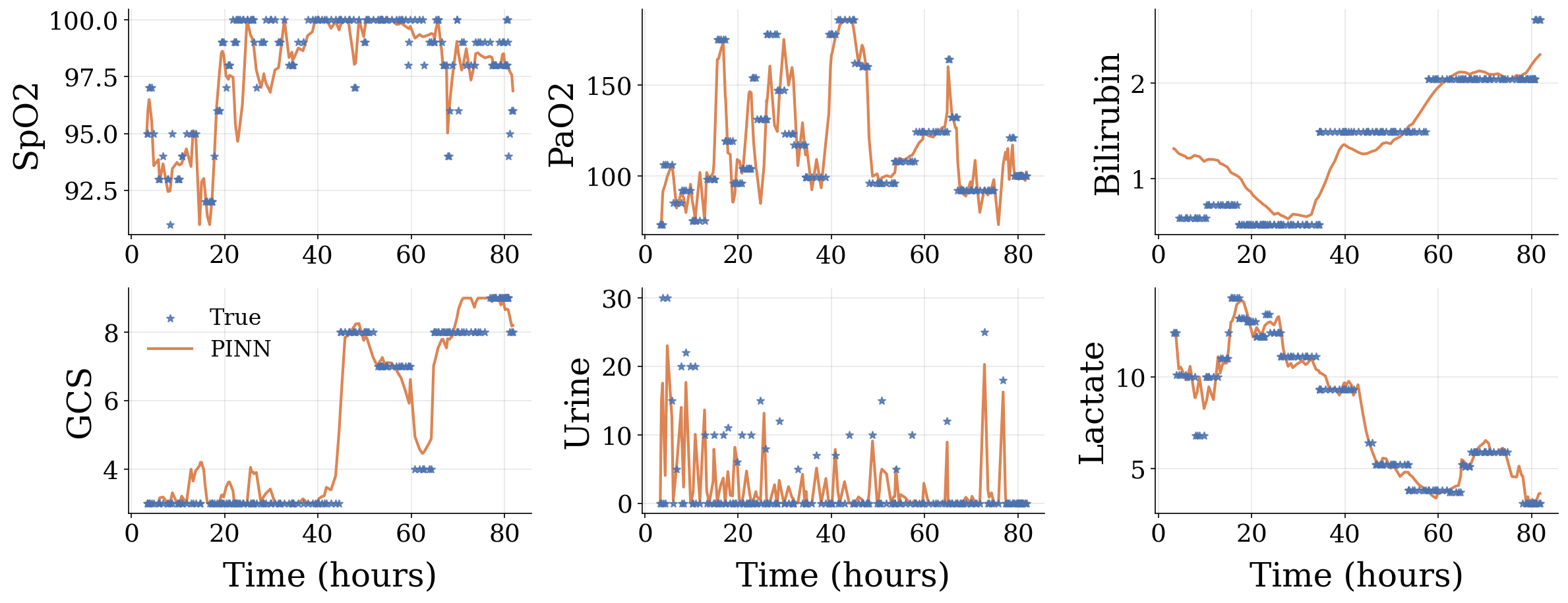}
\caption{Full-horizon PINN rollout versus observed clinical measurements
           for a representative ICU stay. Orange curves show the
           autoregressive rollout; stars denote sparse observations.}
\label{fig:pinn_fit}
\end{figure}

\section{Detailed Validation Results and Discussions}
\label{appendix:detailed_validation_results}

For each method, we train one policy and evaluate the learned policy under five different random seeds in the stochastic environment. 
The reported mean and standard deviation therefore reflect variability from the evaluation process rather than retraining variability.

\subsection{Training Process Setting}
\label{appendix:training}
\paragraph{Overview of training.} Training proceeds in two sequential stages. First, a Physics-Informed Neural Network (PINN) is pre-trained on the recorded clinical trajectory of a single ICU patient to learn a compact patient-specific dynamics model. Second, the frozen PINN is used as a differentiable patient simulator, inside which the RL policy is trained entirely in silico without further interaction with real patient data. This two-stage design decouples physiological modelling from policy optimisation and ensures that the RL agent never queries real clinical observations during training.

\paragraph{PINN Pre-training.} The PINN is trained by minimising a weighted combination of a one-step derivative-consistency loss, a multi-step rollout loss, and a smoothness regulariser. The neural residual component and the learnable physiological rate constants $\{k_i\}$ are optimised jointly via Adam with cosine-annealing learning rate decay and separate weight decay coefficients. Training runs for up to 50{,}000 epochs with early stopping, after which the best checkpoint is restored. Hyperparameters are listed in Table~\ref{tab:hparam_pinn}.

\begin{table}[h]
\centering
\caption{PINN pre-training hyperparameters.}
\label{tab:hparam_pinn}
\small
\begin{tabular}{lr}
\toprule
\textbf{Hyperparameter}& \textbf{Value} \\
\midrule
Maximum epochs                       & 50{,}000 \\
Early-stopping patience              & 800      \\
Learning rate (neural component)     & $1\times10^{-4}$ \\
Learning rate (ODE rate constants)   & $5\times10^{-5}$ \\
Optimiser                            & Adam     \\
\bottomrule
\end{tabular}
\end{table}

\paragraph{RL Policy Training.} 
With the PINN fixed, an ICU treatment environment is constructed as an Option-SMDP. The policy operates on an augmented action space that extends the original treatment actions with an interaction duration 
$\delta t_k \in [\delta t_{\min}, \delta t_{\max}]$ for variable-duration decisions. The PINN then rolls out the patient state over the chosen duration using explicit Euler integration with box constraints, and a scalar reward signal is computed from the resulting physiological trajectory. A safety constraint is enforced on the lactate level throughout each inter-decision interval.
To study the effect of decision frequency, a separate policy is trained for each $K$; full per-$K$ settings are detailed in Section~\ref{appendix:evaluation}.
\begin{table}[h]
\centering
\caption{Option-SMDP environment and RL training hyperparameters.}
\label{tab:hparam_rl}
\small
\begin{tabular}{llr}
\toprule
\textbf{Component}& \textbf{Hyperparameter} & \textbf{Value} \\
\midrule
\multirow{5}{*}{Option-SMDP}
  & $\delta t_{\min}$                & 0.5\,h            \\
  & $\delta t_{\max}$                & 36.0\,h           \\
  & Max steps per episode $K$        & 5--20                \\
  & Total horizon $T$                & 96\,h             \\
\midrule
\multirow{2}{*}{Shared}
  & Hidden layer width                     & 256                       \\
  & Discount factor $\gamma$               & 0.997                     \\
\midrule
\multirow{6}{*}{SAC}
  & Total environment steps                & 100{,}000                 \\
  & Replay buffer capacity / warm-up       & 100{,}000 / 20{,}000      \\
  & Batch size                             & 512                       \\
  & Learning rate                          & $1\times10^{-4}$          \\
  & Soft-update rate $\tau$                & 0.002                     \\
  & Entropy temperature $\alpha$           & auto-tuned                \\
\midrule
\multirow{6}{*}{CPO}
  & Total training episodes                & 12{,}000                  \\
  & Rollouts per policy update             & 20                        \\
  & Trust-region radius $\delta$           & 0.01                      \\
  & Value-function learning rate           & $3\times10^{-4}$          \\
  & Value-function update steps            & 80                        \\
  & Cost limit $d$                         & 2.8--3.5                  \\
\midrule
\multirow{8}{*}{PCPO}
  & Total training episodes                & 12{,}000                  \\
  & Rollouts per policy update             & 20                        \\
  & Trust-region radius $\delta$           & 0.01                      \\
  & Value-function learning rate           & $3\times10^{-4}$          \\
  & Value-function update steps            & 80                        \\
  & Cost limit $d$                         & 2.8--3.5                  \\
  & Projection coefficient scale           & 1.0                       \\
  & Line search                            & off by default            \\
\midrule
\multirow{8}{*}{TRPO}
  & Total training episodes                & 12{,}000                  \\
  & Rollouts per policy update             & 20                        \\
  & Trust-region radius $\delta$           & 0.01                      \\
  & Value-function learning rate           & $3\times10^{-4}$          \\
  & Value-function update steps            & 80                        \\
  & Cost limit $d$                         & none                      \\
  & Safety projection                      & none                      \\
  & Line search                            & on                        \\
\bottomrule
\end{tabular}
\end{table}

\paragraph{Hyper Parameter Settings for All Methods.} 
Hyperparameters for all RL algorithms used in this paper have been summarized in Table \ref{tab:hparam_rl}.

\paragraph{Constant Safety Margin.} 
In the experiments, instead of using the state- and duration-dependent GP-based safety margin in Algorithm \ref{alg:gp_conservative_surrogate}, we adopt a simpler constant safety margin. The reason is practical: when the safety constraint explicitly depends on $\delta t_k$, the interaction duration becomes coupled with both the reward and the safety term, which makes the policy optimization by $\mathsf{CPO}$ and $\mathsf{PCPO}$ numerically unstable and prevents the solver from consistently improving the solution. To avoid this issue, we replace the adaptive margin with a constant conservative threshold. For example, when the lactate safety bound is set to $8.5$, we may impose an effective threshold such as $3.5$ or $2.8$, corresponding to constant safety margins of $5.0$ and $5.7$, respectively. This can be interpreted as a robust approximation of the GP-based surrogate, in which the uncertainty margin is chosen conservatively enough to cover a relevant region of $(\bm{\mathrm{x}}_t,\delta t)$ rather than being adjusted pointwise. Although this simplification sacrifices adaptivity, it yields a numerically stable implementation and still preserves the intended role of the safety margin, namely, to enforce a conservative buffer that aims to mitigate the risk in between interactions. In addition, the cost limit is adjusted according to the number of interaction steps $K$ to maintain a consistent level of conservativeness. Specifically, for $K=\{5,8,10,12,15,17,20\}$, the cost limits are set to $\{2.8, 2.9, 3.0, 3.1, 3.3, 3.4, 3.5\}$, respectively.

\subsection{Evaluation Process Setting}
\label{appendix:evaluation}
\paragraph{PINN Evaluation Setting.} 
The trained PINN is evaluated in full-horizon autoregressive rollout mode: starting from the first observed state, the model integrates the learned ODE forward using only the recorded treatment actions, with no further correction from data. Predictions are inverse-standardised to original clinical units for interpretability, and qualitative fit is assessed by comparing the rollout trajectory against sparse clinical observations across all six physiological variables. The result is shown in Figure~\ref{fig:pinn_fit}.

\paragraph{RL Policy Evaluation Setting.} 
Each trained policy is evaluated over $n=100$ independent episodes
without gradient updates. To ensure a fair comparison across all four variants
($\mathsf{CPO}$-$\mathsf{O}$, $\mathsf{SAC}$-$\mathsf{O}$,
$\mathsf{CPO}$-$\mathsf{E}$, $\mathsf{SAC}$-$\mathsf{E}$),
all episodes within a given $K$ share an identical set of initial-state snapshots generated once and reused across models, so that differences in performance are not confounded by initial-condition variability. Robustness is assessed under stochastic conditions by injecting independent Gaussian noise into both the policy observation and the PINN transition, with noise standard deviation $0.02$ in normalised space and 100 independent noise realisations per policy.
To study the effect of decision frequency, each policy variant is
evaluated across all seven interaction budgets
$K \in \{5, 8, 10, 12, 15, 17, 20\}$, with a separate trained model
loaded for each $K$.
We report two metrics computed over the post-stabilisation window
$t \geq 20$\,h: mean per-step SOFA score, safety rate defined as the
fraction of episodes in which lactate remains below $8.5$\,mmol/L
throughout the window.
Full evaluation settings are summarized in Table~\ref{tab:eval_settings}.
\begin{table}[h]
\centering
\caption{RL policy evaluation settings.}
\label{tab:eval_settings}
\small
\begin{tabular}{lr}
\toprule
\textbf{Setting}& \textbf{Value} \\
\midrule
Test episodes              & 100            \\
Observation noise std      & $0.02$  \\
PINN transition noise std  & $0.02$  \\
Noise runs                 & 100            \\
Post-stabilisation window  & $t \geq 20$\,h \\
Lactate safety threshold   & 8.5\,mmol/L    \\
Evaluation horizon $T$     & 96\,h          \\
Interaction budgets $K$    & $\{5,8,10,12,15,17,20\}$ \\
\bottomrule
\end{tabular}
\end{table}

\subsection{Results of Training Processes}
\label{appendix:training_results}

\begin{figure}[htbp]
\centering
\includegraphics[width=\linewidth]{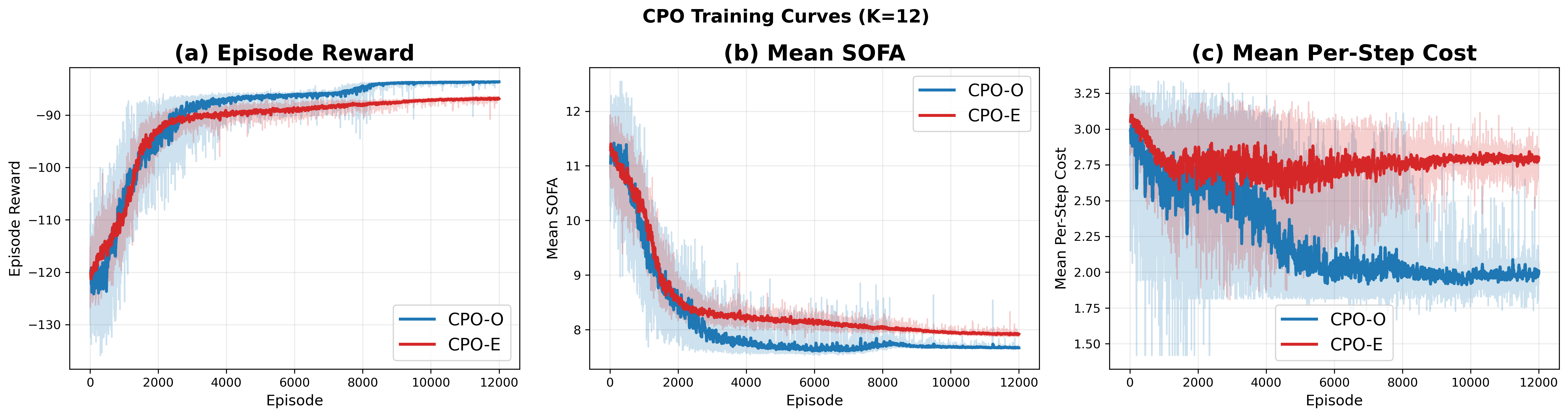}
\caption{Training process of CPO}
\label{fig:CPO_training}
\end{figure}
Both variants of $\mathsf{CPO}$ exhibit stable training dynamics, with consistent reward improvement and reduction in SOFA over the training process. 
$\mathsf{CPO}$-$\mathsf{O}$ achieves a higher final reward level and a lower final SOFA score than $\mathsf{CPO}$-$\mathsf{E}$, indicating stronger asymptotic performance under adaptive interaction timing. 
The mean per-step cost also decreases more substantially for $\mathsf{CPO}$-$\mathsf{O}$, whereas $\mathsf{CPO}$-$\mathsf{E}$ remains at a higher cost level after the initial transient phase. 
Overall, the training curves suggest that adaptive interaction timing improves both treatment effectiveness and safety-related cost control for $\mathsf{CPO}$ at $K=12$.

Training with $\mathsf{SAC}$ shows less stable and more heterogeneous behavior across the two timing schemes. 
$\mathsf{SAC}$-$\mathsf{E}$ improves earlier in training, with a rapid increase in reward and a sharp reduction in SOFA, and it maintains a lower SOFA level throughout most of the later training phase. 
In contrast, $\mathsf{SAC}$-$\mathsf{O}$ remains relatively flat during the early stage, then improves later and eventually reaches a slightly higher reward level, but its final SOFA remains higher than that of $\mathsf{SAC}$-$\mathsf{E}$. 
These results indicate that adaptive timing can improve reward for $\mathsf{SAC}$, but does not necessarily translate into better clinical severity control. 
Compared with $\mathsf{SAC}$, $\mathsf{CPO}$-$\mathsf{O}$ provides a more favorable balance between reward improvement, SOFA reduction, and safety-related cost control.

\begin{figure}[htbp]
\centering
\includegraphics[width=\linewidth]{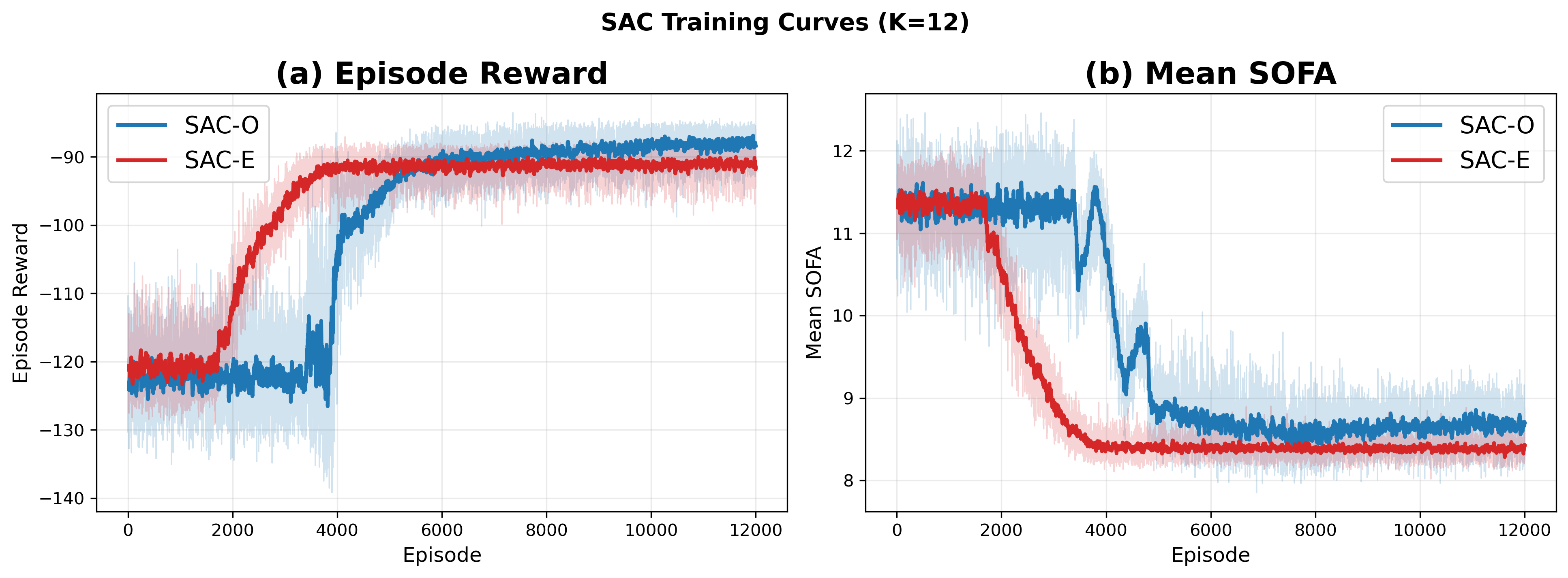}
\caption{Training process of SAC}
\label{fig:SAC_training}
\end{figure}

\begin{figure}[htbp]
\centering
\includegraphics[width=\linewidth]{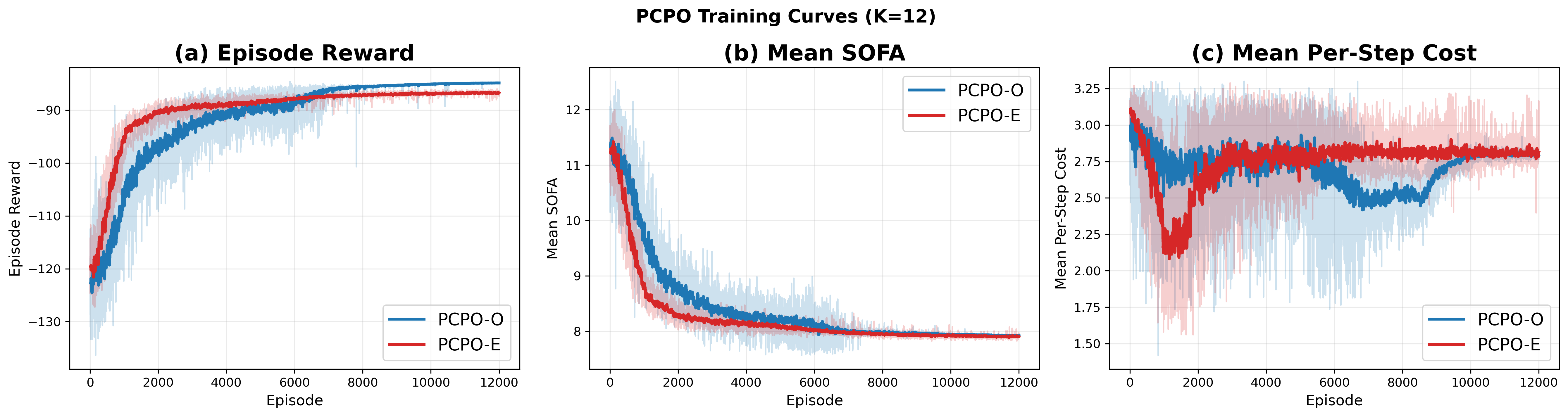}
\caption{Training process of PCPO}
\label{fig:PCPO_training}
\end{figure}

Both variants of $\mathsf{PCPO}$ show stable and well-behaved training trajectories. 
$\mathsf{PCPO}$-$\mathsf{E}$ achieves faster early convergence, with more rapid reward improvement and quicker reduction in SOFA, while maintaining relatively smoother updates in the early stage. 
By comparison, $\mathsf{PCPO}$-$\mathsf{O}$ improves more gradually and exhibits stronger early-stage variability, but continues to improve over a longer horizon and eventually attains a slightly higher reward level with comparable final SOFA. 
The mean per-step cost of both variants remains within a similar range after the initial transient phase, although $\mathsf{PCPO}$-$\mathsf{O}$ displays a more noticeable mid-training reduction before returning close to the level of $\mathsf{PCPO}$-$\mathsf{E}$. 
Overall, $\mathsf{PCPO}$-$\mathsf{E}$ favors faster stabilization, whereas $\mathsf{PCPO}$-$\mathsf{O}$ achieves comparable clinical severity control with slightly better asymptotic reward performance.

The two variants of $\mathsf{TRPO}$ exhibit distinct optimization behaviors. 
$\mathsf{TRPO}$-$\mathsf{E}$ converges faster in the early stage, with more rapid reward improvement and a quicker decrease in SOFA, indicating smoother and more stable optimization under the equidistant interaction schedule. 
In contrast, $\mathsf{TRPO}$-$\mathsf{O}$ learns more gradually and exhibits larger variability during the early phase, but continues to improve throughout training. 
However, unlike the previous comparison, $\mathsf{TRPO}$-$\mathsf{E}$ maintains better final performance in this setting, achieving both a higher reward level and a lower SOFA score than $\mathsf{TRPO}$-$\mathsf{O}$. 
These results suggest that, for $\mathsf{TRPO}$ at $K=12$, the equidistant schedule provides faster and more favorable convergence, while the adaptive-time variant remains less stable and does not fully close the performance gap within the training horizon.

\begin{figure}[htbp]
\centering
\includegraphics[width=\linewidth]{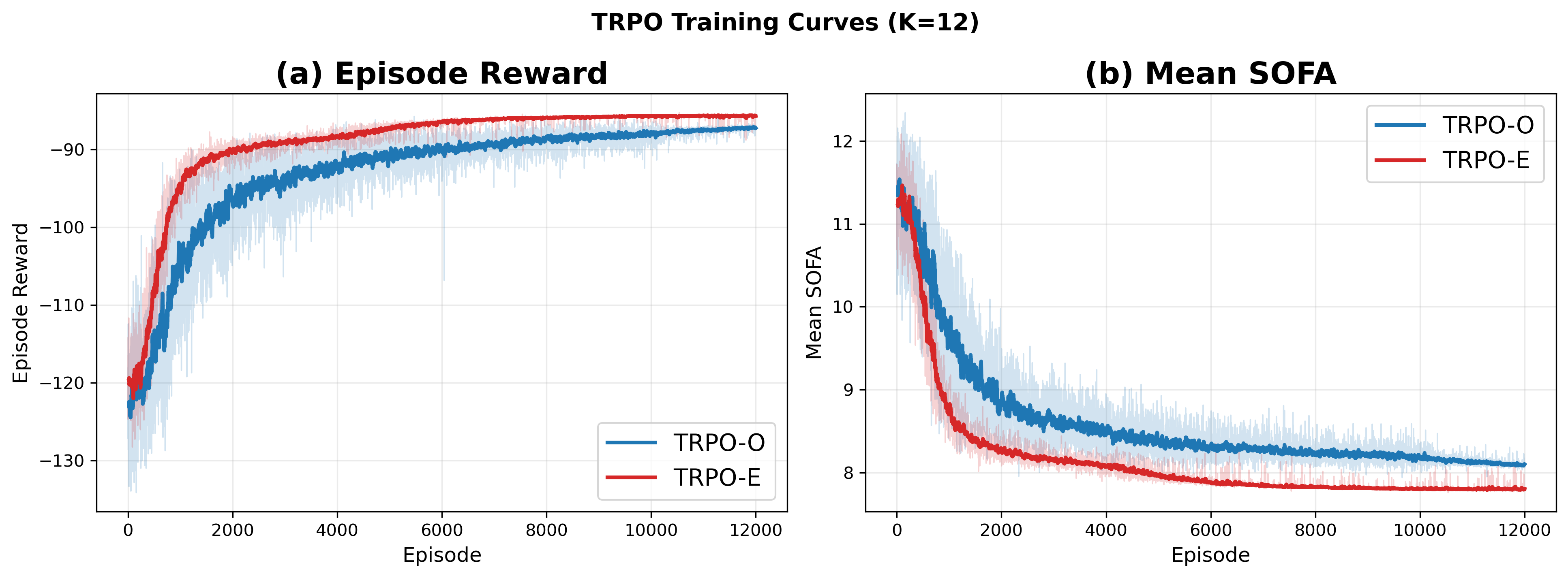}
\caption{Training process of TRPO}
\label{fig:TRPO_training}
\end{figure}

\subsection{State and Action Trajectories}
\label{appendix:state_action_trajectories}

\begin{figure}[htbp]
\centering
\includegraphics[width=\linewidth]{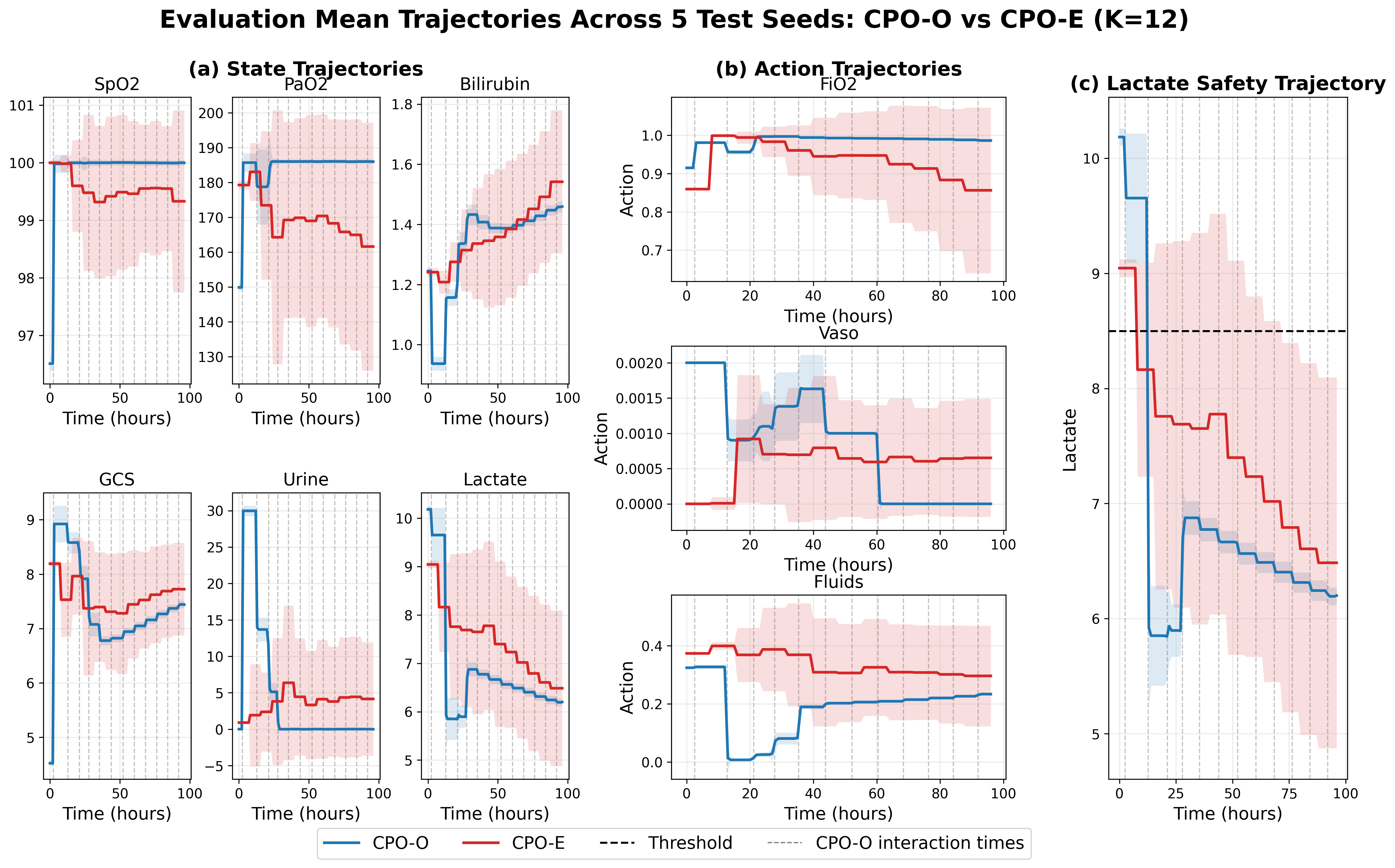}
\caption{Aggregated evaluation trajectories of $\mathsf{CPO}$ with adaptive ($\mathsf{CPO}$-$\mathsf{O}$) and equidistant ($\mathsf{CPO}$-$\mathsf{E}$) interaction timing at $K=12$, averaged across five evaluation seeds. 
Panel (a) shows mean state trajectories, panel (b) shows mean treatment actions, and panel (c) highlights lactate safety relative to the threshold. 
Gray vertical dashed lines indicate representative $\mathsf{CPO}$-$\mathsf{O}$ interaction times, and shaded regions indicate variability across evaluation rollouts.}
\label{fig:j4_cpo_eval_traj}
\end{figure}

Figure~\ref{fig:j4_cpo_eval_traj} summarizes the trajectory-level behavior of the trained $\mathsf{CPO}$ policies during evaluation at $K=12$. 
To account for stochasticity in the evaluation environment, we aggregate trajectories over five evaluation seeds and all evaluation rollouts using a common hourly time grid. 
The gray vertical dashed lines indicate the representative interaction times of $\mathsf{CPO}$-$\mathsf{O}$, obtained by averaging the adaptive interaction times across evaluation rollouts. 
The main difference between the two variants appears in lactate control. 
Although both policies reduce lactate below the safety threshold, $\mathsf{CPO}$-$\mathsf{O}$ decreases lactate more rapidly after the initial phase and maintains a lower lactate level than $\mathsf{CPO}$-$\mathsf{E}$ over most of the horizon. 
By contrast, $\mathsf{CPO}$-$\mathsf{E}$ also improves lactate safety, but its lactate trajectory remains closer to the threshold for a longer period and decreases more gradually.

The action trajectories indicate that the two policies achieve lactate control through different treatment allocation patterns. 
$\mathsf{CPO}$-$\mathsf{O}$ makes more concentrated adjustments around its adaptive interaction times, with an early reduction in fluid administration followed by gradual reallocation later in the horizon. 
$\mathsf{CPO}$-$\mathsf{E}$ follows a more regular equidistant schedule and maintains relatively higher fluid administration over much of the trajectory. 
Both variants keep vasopressor administration close to zero, suggesting that the observed differences in lactate control are mainly associated with adaptive timing and changes in oxygen and fluid-related interventions rather than sustained vasopressor escalation. 
Overall, these trajectory-level results provide qualitative evidence that adaptive interaction timing may improve safety-relevant physiological control compared with an equidistant interaction schedule in this evaluation setting.

\subsection{Comprehensive Evaluations}
\label{appendix:comprehensive_PCPO_TRPO}

\begin{figure}[htbp]
\centering
\includegraphics[width=\linewidth]{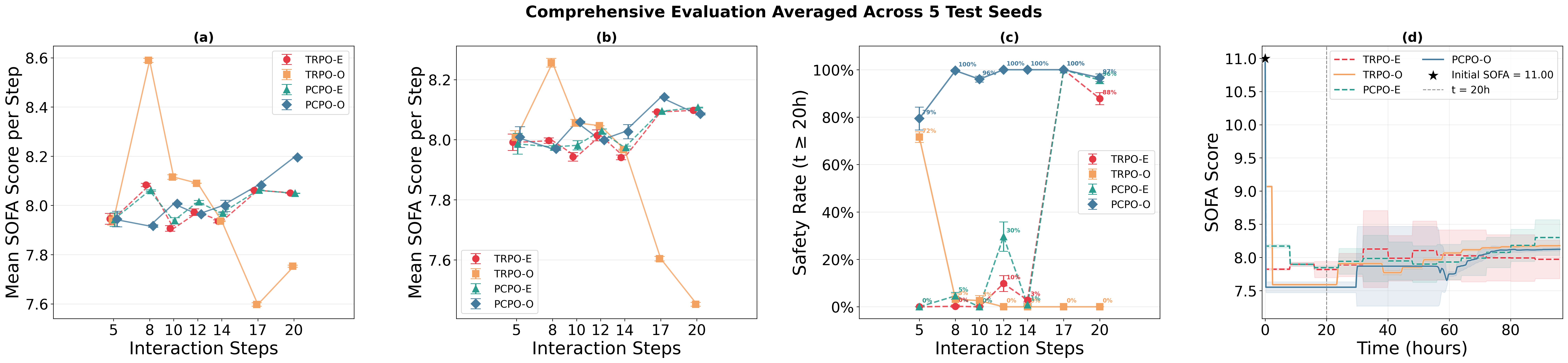}
\caption{Comprehensive evaluation of $\mathsf{TRPO}$ and $\mathsf{PCPO}$ across five test seeds. 
Panels (a) and (b) report mean SOFA scores per step under two evaluation settings, panel (c) shows the safety rate after $t\geq 20$ hours across different interaction steps, and panel (d) shows representative mean SOFA trajectories over time. 
Error bars and shaded regions indicate variability across test seeds.}
\label{fig:j5_trpo_pcpo_eval}
\end{figure}

Figure~\ref{fig:j5_trpo_pcpo_eval} compares the evaluation performance of $\mathsf{TRPO}$ and $\mathsf{PCPO}$ under adaptive and equidistant interaction schedules. 
Across the two SOFA-based evaluation panels, $\mathsf{PCPO}$-$\mathsf{O}$ and $\mathsf{PCPO}$-$\mathsf{E}$ generally remain in a relatively stable SOFA range, while $\mathsf{TRPO}$-$\mathsf{O}$ exhibits larger fluctuations across interaction budgets. 
In particular, $\mathsf{TRPO}$-$\mathsf{O}$ shows strong sensitivity to the number of interaction steps, with noticeably higher SOFA at some intermediate values but lower SOFA at larger interaction budgets in panel (b). 
This suggests that the performance of $\mathsf{TRPO}$ is more dependent on the interaction schedule and evaluation setting.

The safety-rate comparison in panel (c) further highlights the difference between the two methods. 
$\mathsf{PCPO}$-$\mathsf{O}$ maintains consistently high safety rates across most interaction budgets, reaching close to or equal to $100\%$ in several settings. 
By contrast, $\mathsf{TRPO}$-$\mathsf{O}$ shows a sharp decline in safety as the interaction budget increases, while $\mathsf{TRPO}$-$\mathsf{E}$ improves at larger interaction budgets but remains less stable overall. 
The mean SOFA trajectory in panel (d) provides a complementary temporal view: the compared policies gradually stabilize after the initial phase, but their post-$20$h behavior differs in both level and variability. 
Overall, these results suggest that $\mathsf{PCPO}$, especially under adaptive interaction timing, provides more robust safety performance across interaction budgets than $\mathsf{TRPO}$.

\begin{figure}[htbp]
\centering
\includegraphics[width=\linewidth]{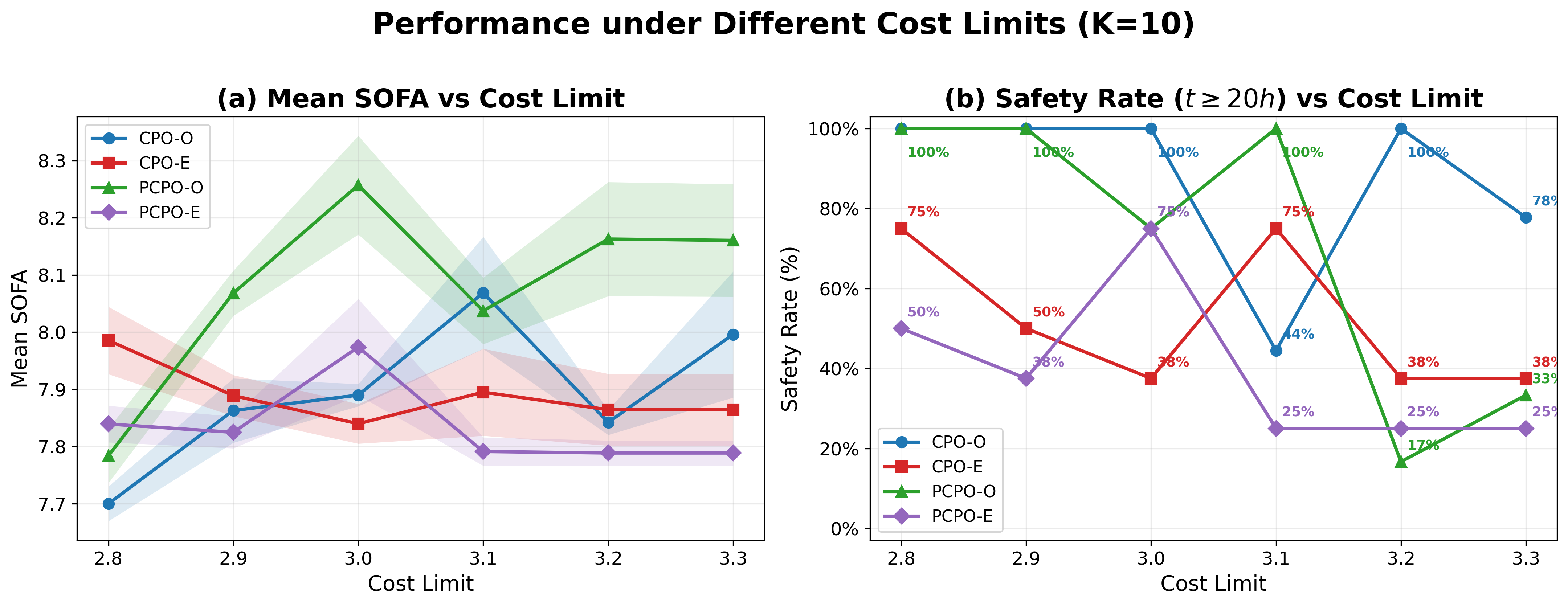}
\caption{Cost-limit sensitivity analysis for $\mathsf{CPO}$ and $\mathsf{PCPO}$ at $K=10$. 
Panel (a) shows mean SOFA as the safety cost limit varies from $2.8$ to $3.3$, while panel (b) shows the post-$20$h lactate safety rate. 
Markers indicate mean values and shaded regions indicate variability across evaluation trajectories.}
\label{fig:j5_cost_limit_sweep}
\end{figure}

Figure~\ref{fig:j5_cost_limit_sweep} examines how the safety cost limit affects the performance of $\mathsf{CPO}$ and $\mathsf{PCPO}$ at $K=10$. This cost-limit sweep is based on a separate set of evaluation runs from those used in Figure~\ref{fig:comprehensive}, and therefore the numerical values should be interpreted as sensitivity results rather than direct seed-matched comparisons.
Panel (a) shows that the mean SOFA response is not monotone in the cost limit, indicating that changing the constraint threshold can alter both policy behavior and physiological outcomes in a nontrivial way. 
Among the four variants, $\mathsf{PCPO}$-$\mathsf{E}$ maintains relatively low and stable SOFA values across the tested cost limits, while $\mathsf{PCPO}$-$\mathsf{O}$ has higher SOFA in several settings. 
$\mathsf{CPO}$-$\mathsf{O}$ achieves low SOFA at some cost limits but is more variable, whereas $\mathsf{CPO}$-$\mathsf{E}$ remains comparatively stable.

Panel (b) shows that the safety rate is highly sensitive to the cost limit and differs substantially across methods. 
$\mathsf{CPO}$-$\mathsf{O}$ achieves high safety rates at several cost limits, including near-perfect safety at some settings, but its safety rate drops at other thresholds. 
$\mathsf{CPO}$-$\mathsf{E}$ shows more moderate safety rates across the sweep. 
For $\mathsf{PCPO}$, the adaptive variant achieves high safety at lower cost limits but becomes less stable as the cost limit changes, whereas the equidistant variant remains lower in most settings. 
Together, these results suggest that the cost limit should be selected jointly with the algorithm and interaction-timing scheme, rather than treated as a fixed implementation detail.

\subsection{Results for Cluster-Level Patient Model} 
\label{appendix:cluster_level_results}
\begin{figure}[htbp]
\centering
\includegraphics[width=\linewidth]{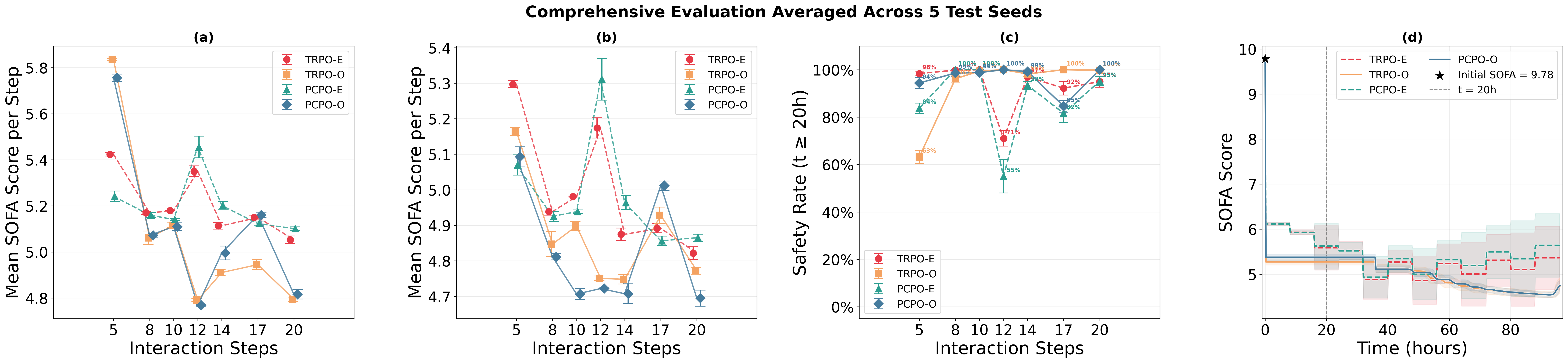}
\caption{Additional evaluation of $\mathsf{PCPO}$ and $\mathsf{TRPO}$ using a cluster-level PINN dynamics model. 
Panels (a) and (b) report the mean SOFA score per step under two evaluation summaries, panel (c) reports the post-20h lactate safety rate across interaction budgets, and panel (d) shows the mean SOFA trajectory over time. 
For this cluster-level evaluation, safety is computed using a stricter test threshold, where a trajectory is considered safe if lactate remains below 3.5 mmol/L after $t \geq 20$h. 
Error bars and shaded regions indicate variability across test seeds.}
\label{fig:j6_cluster_pcpo_trpo}
\end{figure}

Figure~\ref{fig:j6_cluster_pcpo_trpo} provides an additional robustness evaluation using a cluster-level PINN dynamics model rather than a single-patient PINN. 
Patient clusters are formed greedily by selecting centers from high-density regions, where density is defined by the Mean Absolute Error (MAE) of state rollouts between two PINN models. Each center is assigned its 10 nearest neighbors, while centers are enforced to remain a minimum distance apart and only clusters satisfying a predefined distance threshold are retained. Specifically, the $\mathsf{PCPO}$ and $\mathsf{TRPO}$ policies are trained under the dynamics induced by a PINN model learned from patients within each cluster, while the cluster-center patient is used as the initial-state source during training and evaluation. 
This setting tests whether the qualitative behavior of the adaptive-time and fixed-time policies remains meaningful when the learned environment model is constructed from a patient cluster.
Since all methods satisfy the original lactate safety threshold of 8.5 mmol/L under this cluster-level model, we use a stricter test-time lactate threshold of 3.5 mmol/L in panel (c) to obtain a more informative comparison of safety behavior. 
The SOFA results in panels (a), (b), and (d) show that all four variants achieve relatively low SOFA scores under the cluster dynamics. 
Across most interaction budgets, the adaptive-time methods, especially $\mathsf{PCPO}$-$\mathsf{O}$, remain competitive with or better than their fixed-time counterparts in terms of SOFA control. 
At the same time, panel (c) shows that $\mathsf{PCPO}$-$\mathsf{O}$ maintains high post-20h safety rates under the stricter lactate threshold across interaction budgets. 
Overall, these results provide additional evidence that the proposed adaptive constrained policy remains stable under a cluster-level patient dynamics model, while the stricter lactate threshold reveals safety differences that are hidden under the original 8.5 mmol/L evaluation threshold.

\end{document}